\documentclass{cspmA_tight}

\usepackage{soul,rotating,floatpag}  
\rotfloatpagestyle{empty}  
  
\usepackage{amsmath,amsfonts,latexsym,epsfig,subfigure,fancybox,verbatim,mdframed,framed,amsthm,algorithm,isomath,multirow,makeidx,cancel,algorithmic,minitoc,url,varioref,datetime,setspace,array,needspace,cases,mathtools,titlesec,wrapfig,overpic,xspace,listings}   
\usepackage[inline]{enumitem}
\usepackage[theorems,breakable]{tcolorbox}  

\usepackage{newtx}  

\usepackage{tikz,tikz-3dplot}
\usetikzlibrary{chains,arrows.meta,quotes,bending,shapes,positioning,fit,calc,matrix,patterns,decorations}
\usepackage{pgfplots, pgfplotstable}

\usepackage{caption}
\usepackage{graphicx}   

\usepackage[title]{appendix}
  \AtBeginEnvironment{subappendices}{%
    \addcontentsline{toc}{section}{\em Appendix}
\begin{center} \vspace{10pt}   \large{\underline{APPENDIX TO CHAPTER \thechapter}}  
\end{center}
\vspace{-18pt}
}

\newcounter{rowno}
\allowdisplaybreaks

\makeindex

\usepackage[style=alphabetic, giveninits=true,maxbibnames=5]{biblatex}
\DeclareListInputHandler{location}{}

\DeclareFieldFormat{sentencecase}{\MakeSentenceCase{#1}}

\renewbibmacro*{title}{%
  \ifthenelse{\iffieldundef{title}\AND\iffieldundef{subtitle}}
    {}
    {\ifthenelse{\ifentrytype{article}\OR\ifentrytype{inbook}%
      \OR\ifentrytype{incollection}\OR\ifentrytype{inproceedings}%
      \OR\ifentrytype{inreference}}
      {\printtext[title]{%
        \printfield[sentencecase]{title}%
        \setunit{\subtitlepunct}%
        \printfield[sentencecase]{subtitle}}}%
      {\printtext[title]{%
        \printfield[titlecase]{title}%
        \setunit{\subtitlepunct}%
        \printfield[titlecase]{subtitle}}}%
     \newunit}%
  \printfield{titleaddon}}

\defbibheading{bibliography}[\refname]{%
   \chapter*{Bibliography}  
\markboth{Bibliography}{Bibliography}}

\defbibenvironment{bibliography}
{\list
     {\printtext[labelalphawidth]{%
         {\fontseries{sb}\selectfont{\printfield{prefixnumber}}}%
        \textbf{\printfield{labelalpha}}%
        \printfield{extraalpha}}}
     {\setlength{\labelsep}{\biblabelsep}%
      \setlength{\leftmargin}{\labelsep}%
      \setlength{\itemsep}{\bibitemsep}%
      \setlength{\parsep}{\bibparsep}}%
      }
  {\endlist}
  {\item}

\input{cspmA_tight_macro}     
\newcommand{\secn}{\S}

 \makeatletter
  \def\nl#1#2{\begingroup
     \scalebox{0.85}[1]{\textbf{#2}}%
     \def\@currentlabel{\textnormal{\scalebox{0.85}[1]{\textbf{#2}}}}
     \phantomsection\label{#1}\endgroup
}
 \makeatother

  \DeclareMathAlphabet{\mathcal}{OMS}{cmsy}{m}{n}

{\ensuremath{
\begin{empheq}[box=\fbox]{align}
{#1}
\end{empheq}
}}

\counterwithin{algorithm}{chapter}

\newtheorem{theorem}            {Theorem}[chapter]

\newtheorem{lemma}              [theorem]{Lemma}

\theoremstyle{definition}   
\newtheorem{definition}         [theorem]{Definition}

\makeatletter
\def\thm@space@setup{%
  \thm@preskip=6pt plus 5pt minus 3pt 
  \thm@postskip=6pt plus 5pt minus 3pt 
}
\makeatother

 \newcommand{\normal}{\mathcal{N}}

{
\begin{mdframed}
\par\noindent\textbf{#1:}\begin{rmfamily}\noindent}%
{\end{rmfamily}
\end{mdframed}
} 

\newtcbtheorem{summary}{\section*{SUMMARY}
\addcontentsline{toc}{section}{Summary}
}{%
  theorem name,%
  breakable,
  colback=gray!14,%
    enlarge left by=-1.2cm,
        width=1.2\textwidth,  
        colframe=white,coltitle=black,colbacktitle=white,
        before skip=15pt,
        fonttitle=\bfseries
    }{summary}

\newlist{senumerate}{enumerate}{3}
\setlist[senumerate]{label=({\bf \roman*}),wide, labelwidth=!, labelindent=0pt,topsep=0pt}

\newlist{renumerate}{enumerate}{3}
\setlist[renumerate]{label={\bf{\arabic*.}},wide,labelindent=!,labelwidth=!,topsep=0pt,ref={\arabic*}}

\newenvironment{rem}{\needspace{0.5\baselineskip}
 \noindent {\bf{\em \scalebox{0.85}[1.0]{Remarks}}}. 
  \begin{renumerate}} 
  {\end{renumerate}}

\newcommand{\pwe}{Complements and Sources}

\newcommand{\pdf}{p}

\newcommand{\cdf}{F}
\newcommand{\prob}{\mathbb{P}}
\newcommand{\E}{\mathbb{E}}

\newcommand{\bstate}{\bar{\state}}

\newcommand{\obs}{y}

\newcommand{\snoise}{w}
\newcommand{\onoise}{v}
\newcommand{\statem}{A}

\newcommand{\obsm}{C}

\newcommand{\snoisecov}{Q}
\newcommand{\onoisecov}{R}

\newcommand{\state}{x}
\newcommand{\statespace}{\mathcal{X}}
\newcommand{\obspace}{\mathcal{Y}}
\newcommand{\statedim}{X}
\newcommand{\obsdim}{{Y}}

\newcommand{\mc}{r}  

\newcommand{\fun}{\phi}

\newcommand{\oprob}{B}
\newcommand{\tp}{P}

\newcommand{\finaltime}{N}

\newcommand{\belief}{\pi}

\newcommand{\bbelief}{\bar{\pi}}

\newcommand{\Belief}{\Pi(\statedim)}

   \newcommand{\kalmancov}{\Sigma}

\newcommand{\tr}{\operatorname{trace}}
\newcommand{\trace}{\tr}
\newcommand{\Sig}{S}   

\newcommand{\weight}{\omega}
\newcommand{\ole}{\overset{\textnormal{defn}}{=}}

\newcommand{\filterd}{\sigma}

\newcommand{\filter}{T}

\newcommand{\argmin}{\operatornamewithlimits{argmin}}
\newcommand{\argmax}{\operatornamewithlimits{argmax}}

\newcommand{\reals}{\mathbb{R}}

\newcommand{\beq}{\begin{equation}}
\newcommand{\eeq}{\end{equation}}
\newcommand{\nn}{\nonumber}

\renewcommand{\(}		{\left(}
\renewcommand{\)}		{\right)}

\renewcommand{\th}{\theta}

\newcommand{\p}{\prime}

\newcommand{\one}{\mathbf{1}}
\newcommand{\ones}{\mathbf{1}}
\newcommand{\zero}{0}

\renewcommand{\div} {{\operatorname{div}}}

\newcommand{\diag}{\textnormal{diag}}

\newcommand{\pnoise}{\nu}

\newcommand{\horizon}{N}

\newcommand{\cost}{c}

\newcommand{\reward}{r}

\newcommand{\action}{u}

\newcommand{\bact}{\bar{a}}

\newcommand{\actionspace}{\,\mathcal{U}}
\newcommand{\actiondim}{U}

\newcommand{\discount}{\rho}

\def \con {{\beta}}
\def \rcon {{\gamma}}

\def\param{{\alpha}}

 \newcommand{\statpi}{\pi}

\newcommand{\policy}{\mu}

\newcommand{\optpolicy}{\policy^*}

\newcommand{\overlook}{\beta}

\newcommand{\stopset}{\mathcal{S}}

\newcommand{\thtrue}{\theta^o}

\newcommand{\bd}{\succeq_{\mathcal{B}}}

\newcommand{\noise}{n}

\newcommand{\bm}{W}

\newcommand{\cd}{(\cdot)}

\def\ph{\varphi}

\newcommand{\wdt}{\widetilde}

\newcommand{\ad}{&\!\!\!\disp}

\newcommand{\barray}{\begin{array}{ll}}
\newcommand{\earray}{\end{array}}
\newcommand{\disp}{\displaystyle}

\def\cd{(\cdot)}

\newcommand{\numagents}{M}

   \newcommand{\probe}{\alpha}
\newcommand{\response}{\beta}

\newcommand{\utility}{U}
\newcommand{\budget}{I}
 \newcommand{\dataset}{\mathbb{D}}

\newcommand{\tindx}{k}
\newcommand{\Tindxter}{\finaltime}

\newcommand{\dtime}{k}

\usepackage{xpatch}
\makeatletter
\xpretocmd{\@endpart}{%
  \ifx\@abstract\@empty\else
    \bigskip
    \begin{quote}\@abstract\end{quote}
    \global\let\@abstract\@empty
  \fi
}{}{}
\newcommand{\partabstract}[1]{%
  \renewcommand{\@abstract}{\normalsize #1}%
}
\newcommand{\@abstract}{}
\makeatother

\newlist{steplist}{enumerate}{3}
\setlist[steplist]{label={\em{Step \arabic*.}},leftmargin=*,topsep=0pt}

\newlist{steplistn}{enumerate}{3}
\setlist[steplistn]{label={\em{Step \arabic*.}},leftmargin=*,wide}

\renewcommand{\theenumii}{(\alph{enumii})}

\newenvironment{compactenum}[1][]{
  \begin{enumerate}[leftmargin=*,itemsep=0pt,parsep=0pt,topsep=0pt,label=\arabic*.,#1]
    \renewcommand{\theenumi}{\arabic{enumi}}
    
}{
  \end{enumerate}
}

\newenvironment{compactitem}[1][]{
  \begin{itemize}[leftmargin=*,itemsep=0pt,parsep=0pt,topsep=0pt,#1]
}{
  \end{itemize}
}

\newcommand{\eqrefp}[1]{\eqref{#1}~\vpageref{#1}}

\newcommand{\bigtuple}{\Theta}

\newcommand{\BRP}{\operatorname{BRP}}

\newcommand{\datann}{\mathbb{D}}
\newcommand{\datasetaccum}{\datann}

\newcommand{\dpiter}{m}
\newcommand{\dpitertwo}{l}
\newcommand{\dpset}{\mathcal{M}}
\newcommand{\stateset}{\ensuremath{\mathcal{X}}} 
\newcommand{\actionset}{\ensuremath{\mathcal{A}}} 
\newcommand{\obsset}{\mathcal{Y}}
\newcommand{\RIcost}{K}
\newcommand{\hRIcost}{\hat{\RIcost}}
\newcommand{\attfunsymb}{\oprob}

\newcommand{\utilitysymbolagent}[1]{\utilitysymbol_{#1}}
\newcommand{\hutilitysymbolagent}[1]{\hat{\utilitysymbol}_{#1}}
\newcommand{\stoptuple}{\Xi}
\newcommand{\optstoptuple}{\stoptuple_{opt}}

\newcommand{\utilitysymbol}{r}
\newcommand{\hutilitysymbol}{\hat{\utilitysymbol}}

\newcommand{\actselectagent}[1]{p_{#1}(\act|\state)}

\newcommand{\runcostinst}{c}

\newcommand{\sumruncostsymbol}{C}
\newcommand{\sumruncostagent}[1]{\sumruncostsymbol_{#1}}
\newcommand{\numdp}{M}

\newcommand{\actiontwo}{\bar{a}}
\newcommand{\utilityagent}[1]{\utilitysymbol_{#1}(\state,\act)}
\newcommand{\hutilityagent}[1]{\hat{\utilitysymbol}_{#1}(\state,\act)}

\newcommand{\act}{a}
\newcommand{\actdim}{A}
\newcommand{\attspace}{\Delta}

\newcommand{\agent}{m}

\newcommand{\agentset}{\mathcal{M}}

\newcommand{\optstoptime}{\stoptime^\ast}

\newcommand{\stoptime}{\mu}
\newcommand{\funcstop}{\tau}

\newcommand{\datainf}{\mathcal{D}_{M}}
\newcommand{\agenttwo}{l}
\newcommand{\optsearchtuple}{\stoptuple_{opt}}

\newcommand{\SHT}{\operatorname{SHT}}
\newcommand{\numstates}{X}

\newcommand{\contcost}{c}
\newcommand{\stopcost}{s}
\newcommand{\env}{m}
\newcommand{\benv}{l}
\newcommand{\envspace}{\mathcal{M}}

\newcommand{\temperature}{\beta}

\newcommand{\forward}{\mathcal L^*}

\newcommand{\fobspace}{\mathcal{O}}
\newcommand{\fobs}{O}
\newcommand{\foprob}{\Psi}
\newcommand{\msL}{L}
\newcommand{\probedim}{m}
\newcommand{\budgetg}{g}

\newcommand{\sindx}{s}
\newcommand{\precconstraint}{p_*}
\newcommand{\fasttime}{t}

\newcommand{\evalue}{\lambda}
\newcommand{\bound}{\bar{\kalmancov}}
\newcommand{\ebound}{\bar{\lambda}}
\newcommand{\upperbound}{\bar{\response}}
\newcommand{\ARE}{\operatorname{\mathcal{A}}}

\newcommand{\obsresponse}{\bar{\response}}
\newcommand{\Obsresponse}{\boldsymbol{\obsresponse}}
\newcommand{\nresponse}{\obsresponse}

\newcommand{\Anoise}{\boldsymbol{\anoise}}
\newcommand{\Pnoise}{\boldsymbol{\pnoise}}
\newcommand{\anoise}{\epsilon}
\newcommand{\obsdataset}{\mathcal{D}_\text{obs}}
\newcommand{\threshold}{\gamma}
\newcommand{\Probe}{\boldsymbol{\probe}}
\newcommand{\tPhi}{\tilde{\Phi}}
\newcommand{\lambdat}{\lambda^o}
\newcommand{\ut}{u^o}
\newcommand{\ccdf}{\bar{F}_M}
\newcommand{\nprobe}{\bar{\probe}}
\newcommand{\Obsprobe}{\boldsymbol{\nprobe}}
\newcommand{\Response}{\boldsymbol{\response}}
\newcommand{\setresponse}{{\boldsymbol{\response}}}
\newcommand{\potfun}{V}
\newcommand{\np}{P}
\newcommand{\ia}{p}
\newcommand{\ja}{q}
\newcommand{\feasconsi}{\hat{\response}_\dtime^\ia}

 \newcommand{\sunderbar}[1]{\underline{\smash{#1}\vphantom{a}}}

 \newcommand{\uresponse}{\sunderbar{\response}}
 \newcommand{\lind}{\ell}
 \newcommand{\hresponse}{\hat{\response}}
 \newcommand{\welfare}{W}
 \newcommand{\afriat}{\mathcal{A}}
 \newcommand{\margin}{\mathcal{M}}
 \newcommand{\subjectto}{\operatorname{subject\; to}}
 \newcommand{\hutility}{\hat{\utility}}
 
 \newcommand{\presponse}{{\tilde{\response}}}
 \newcommand{\putility}{{\tilde{\utility}}}
 \newcommand{\masking}{\eta}
 \newcommand{\pdataset}{{\tilde{\dataset}}}
 \newcommand{\HMI}{\operatorname{HMI}}
 \newcommand{\VI}{\operatorname{VI}}
 \newcommand{\MCI}{\operatorname{MCI}}
 
  \newcommand{\uaf}{\phi}
  \newcommand{\Tstat}{\mathcal{T}}
  \newcommand{\dtimen}{s}
  \newcommand{\iindex}{i}

\newcommand{\Search}{\operatorname{Search}}
\newcommand{\datainfsearch}{\mathcal{D}_{\numagents}(\Search)}
\newcommand{\aaction}{a}

\newcommand{\csearch}{c}
\newcommand{\searchcostsymbol}{c}
\newcommand{\IRLoutput}{\operatorname{IRL}}
\newcommand{\datafin}{\mathcal{D}_{n}}
\newcommand{\aprob}{G}
\newcommand{\post}{\eta}
\newcommand{\fBelief}{\Pi}

\newcommand{\bB}{\bar{\oprob}}

\newcommand{\truestate}{x^o}

\newcommand{\hb}{z}

\makeatletter
\newcommand{\inlineitem}[1][]{%
\ifnum\enit@type=\tw@
    {\descriptionlabel{#1}}
  \hspace{\labelsep}%
\else
  \ifnum\enit@type=\z@
       \refstepcounter{\@listctr}\fi
    \qquad\@itemlabel\hspace{\labelsep}%
\fi}
\makeatother

\newcommand{\hstate}{\hat{\state}}

\newcommand{\kg}{\Psi}

\newcommand{\anoisecov}{\sigma_\epsilon^2}

\newcommand{\lact}{\underline{\act}}
\newcommand{\actspace}{\mathcal{A}}

\newcommand{\qdp}{\belief_0}

\newcommand{\upar}{\theta}
\newcommand{\unoise}{\epsilon}
\newcommand{\regc}{\psi}

\newcommand{\remar}{\noindent {\textbf{\textit{\scalebox{0.85}[1.0]{Remark}}}}.\ }

\newcommand{\remarn}[1][]{%
  \noindent\textbf{\textit{\scalebox{0.85}[1.0]{Remark}}}%
  \ifx\relax#1\relax.
  \else%
    . \textbf{\textit{\scalebox{0.85}[1.0]{#1}}}: 
  \fi
}

\newcommand{\summar}{\noindent {\textbf{\textit{\scalebox{0.85}[1.0]{Summary}}}}.\ }

\newcommand{\hdformat}[1]{\noindent {\textbf{\textit{\scalebox{0.85}[1.0]{#1}}}}.\ }

  \newcommand{\examt}[1]{%
   \noindent\textbf{\textit{\scalebox{0.85}[1.0]{Example. #1}}}: \kern-0.3em{}%
    }

\newcommand{\examnt}[2]{%
    \noindent \textbf{\textit{\scalebox{0.85}[1.0]{Example #1. #2}}}: \kern-0.3em %
}

\newcommand{\alp}{\alpha}

\newcommand{\kerneln}{\operatorname{K}_\kernelstep}
\newcommand{\kernelstep}{\Delta}

\newcommand{\simplies}{\mathrel{\!\!\implies\!\!}}

\newcommand{\tc}[1]{\mathbf{#1}}

\newcommand{\RHS}{\text{right hand side}\xspace}

\newcommand{\tslow}{n}
\newcommand{\Reward}{R}
\newcommand{\stoptimeirl}{\tau}
\newcommand{\thdim}{N}
\newcommand{\step}{\varepsilon}
\newcommand{\stepa}{\mu}
\newcommand{\kernel}{\operatorname{K}}

\renewcommand{\eth}{\alpha}

\newcommand{\stat}{\pdf}

\tikzset{
    blockff/.style={rectangle, draw, line width=0.2mm, black,  text centered,text width=6em,
                 minimum height=2em},
    line/.style={draw, -latex}}
\newcommand{\obsnoise}{v}
\newcommand{\mreward}{\rho}
\newcommand{\numparticles}{L}
\newcommand{\mcstep}{\eta}
\newcommand{\stepn}{\nu}
\newcommand{\Tr}{\operatorname{Tr}}
\newcommand{\al}{\eth}
\newcommand{\Cons}{B}

\newcommand{\admissible}{\mathcal{D}}
\newcommand{\cond}{\phi}
\newcommand{\randmix}{p}
\newcommand{\logistic}{{\mathcal E}}
\newcommand{\stepmc}{\mcstep}

\begin{document}

\begin{center}
{\LARGE Inverse Reinforcement Learning using Revealed Preferences and Passive Stochastic Optimization}

\vspace{1em}

{\large Vikram Krishnamurthy\\
Cornell University, Ithaca, NY 14853, USA.\\
\texttt{vikramk@cornell.edu}} \\

\vspace{1em}
\today
\end{center}
\vspace{1em}

Cite this monograph  as:
\begin{lstlisting}[basicstyle=\ttfamily, breaklines=true]
@Book{Kri25,
  author    = {Vikram Krishnamurthy},
  title     = {Partially Observed Markov Decision Processes. Filtering, Learning and Controlled Sensing},
  publisher = {Cambridge University Press},
  edition   = {2},
  year      = {2025}
}
\end{lstlisting}

\dominitoc

\frontmatter

\tableofcontents


\chapter*{Summary}


The first two chapters of 
this short monograph  on inverse reinforcement learning (IRL)  draw heavily  from my new book {\em Partially Observed Markov Decision Processes}, Second Edition, published by Cambridge University Press in June 2025. In the first two chapters, we will  view IRL through the lens of microeconomics via revealed preferences. In the third chapter, we study adaptive IRL using passive learning based on Langevin dynamics.

Our motivation for IRL stems from identifying smart controlled sensors from their actions.
  A controlled sensor is a  constrained utility maximizer that  adapts its sensing modes to changes in the environment.
This monograph  discusses {\em inverse} reinforcement learning (IRL)  for controlled  sensing by  addressing two questions:
\begin{compactenum}
\item By observing the signals (actions) of a  sensing system, how to identify if  the system is a constrained utility maximizer?
\item If the sensing system is a constrained utility maximizer,  how to  estimate its   utility function and therefore predict its future actions?
\end{compactenum}

Mathematically speaking, IRL can be viewed   as
 inverse optimization or inverse stochastic control.
  The IRL framework transcends  
classical 
statistical signal processing to
address the deeper issue of {\em how to infer strategy from sensing data}.
IRL facilitates understanding   how adaptive
sensors make decisions, and conversely how to design adaptive sensors that are robust to adversarial attacks. Besides controlled sensing (such as cognitive radar), the data-driven IRL approach
that we will discuss is  used in microeconomics to  understand consumer behavior via revealed
preference, in multimedia social networks to explain user engagement, and in human--robot interaction to understand 
how a robot  learns by observing a human decision maker. 

This monograph comprises three chapters.  Chapter~\ref{chp:afriat} uses classical revealed preference (Afriat's theorem and extensions) to identify constrained utility maximizers based on the actions of an agent, and then reconstruct set-valued estimates of the agent's utility.  We illustrate this procedure by identifying the presence of a cognitive radar and reconstructing its utility function. We also discuss how to construct a statistical  detector for utility maximization behavior when the actions of the agent are measured in noise.

Chapter~\ref{chp:birl} discusses Bayesian IRL. By observing the decisions of a Bayesian agent, how can an analyst determine if the  agent is a Bayesian utility maximizer that is simultaneously optimizing its observation likelihood?  Chapter~\ref{chp:birl} also discusses  inverse stopping-time problems, namely reconstructing the continue and stopping cost of the Bayesian agent that acts over a random horizon. This IRL methodology is then used to identify the presence of a Bayes optimal sequential detector. The chapter also gives a brief description of discrete choice models, inverse Bayesian filtering and inverse stochastic gradient algorithms for adaptive IRL.

Chapter~\ref{chp:langevin} studies the design of  adaptive IRL algorithms that can estimate a time evolving utility function given noisy misspecified gradients. Suppose the adversary is a cognitive sensor that uses (possibly multiple)  stochastic gradient algorithms to optimize adaptively the expected  reward $R(\th) = \E\{r_k(\th)\}$ for choosing its  sensing strategy.
In learning its optimal sensing strategy, the cognitive sensor generates the dataset of randomly chosen points $\th_k$ and noisy gradients  $\nabla_\th\reward_k(\th_k)$:
\begin{equation*}
 {\footnotesize \begin{matrix} \text{Dataset generated by} \\ \text{adversary cognitive sensor}
                   \end{matrix}} \qquad  \{\th_k, \nabla_\th\reward_k(\th_k),k=1,2,\ldots\}
\end{equation*}
 Suppose we observe  the dataset from the above equation in real time.
{\em How can we  design an adaptive IRL algorithm that 
estimates the  expected  reward $R(\th) = \E\{r_k(\th)\}$?}  This chapter draws heavily from the paper~\cite{KY21} that studies adaptive passive Langevin dynamics and their weak convergence.

This monograph is work in progress. 
I plan to class-test  the material  in a PhD level course that I teach  at Cornell.
The IRL methods  in Chapter 1 also apply to identifying group intent, for example, identifying if multiple drones are coordinating their behavior,  and details will be added in future.
In due course, I plan to add a chapter on the finite sample analysis of the passive Langevin dynamics (in contrast to Chapter 3 which contains the asymptotic analysis).
I also plan to add chapters on how IRL can be used as a basis of explainable AI and also in social network analysis for explaining certain sociological phenomenon.

\vspace{0.1cm} 
\hfill Vikram Krishnamurthy,  \\ \mbox{} \hfill Ithaca, New York.  \\
\mbox{}  \hfill vikramk@cornell.edu


\mainmatter


\chapter{Revealed Preferences for Inverse Reinforcement Learning}  
\label{chp:afriat}

\minitoc

\index{inverse reinforcement learning (IRL)|seealso{revealed preferences, Bayesian IRL}}

\noindent
Inverse reinforcement learning (IRL) aims to estimate the utility function of a decision system by analyzing its input--output dataset.
This chapter discusses IRL for controlled sensing.  
Mathematically,  IRL can be viewed   as
inverse optimization. 
To motivate this chapter, we discuss IRL in the context of an adversarial cognitive radar. The reader unfamiliar with the radar application can simply assume that a cognitive radar is a constrained utility maximizer.

Cognitive radars  \cite{Hay06} utilize  the perception--action cycle of cognition to  sense the environment, learn relevant information about the target and the background, and then optimally adapt their sensing modes  to meet the mission's objectives. Cognitive radars   adaptively optimize their waveform, aperture, dwell time and revisit rate. 

This chapter  is motivated by the  next logical step, namely, \textit{inverse cognitive radar}.
  The framework involves an adversarial signal processing IRL  problem comprising    ``us'' and an ``adversary''. Figure \vref{fig:schematica} displays the schematic setup.
``Us'' refers to a drone/UAV  or electromagnetic signal that probes an ``adversary'' cognitive  radar system.
The adversary cognitive radar  estimates our kinematic coordinates using a Bayesian tracker and then adapts its mode (waveform, aperture, revisit time) dynamically using feedback control  based on sensing our kinematic state (e.g., position and velocity of drone).
At each time $\dtime$ our kinematic state can be viewed as a probe vector $\probe_\dtime \in \reals_+^\probedim$ to the radar. We  observe  the  adversary radar's response $\response_\dtime \in \reals^\probedim$. (For now we assume $\probe_\dtime$ and $\response_\dtime$ have the same dimension; in \secn \ref{sec:forges} we will relax this.) Given the time series of probe vectors and responses, $\{\probe_\dtime,\response_\dtime,\dtime=1,\ldots,\horizon\}$, this chapter addresses the following questions:
\begin{compactenum}
\item  How to identify if  the adversary's radar is cognitive, i.e., does there exist a utility function $\utility(\response)$ that the radar  is maximizing to generate its response $\response_\dtime$ to our probe input $\probe_\dtime$?

\item  How to construct an IRL algorithm that estimates  the utility function  of the adversary's cognitive radar? Conversely,
how to design a
cognitive radar that hides its utility from an IRL algorithm?
That is, how can a smart sensor purposely act dumb? 

\item How to identify if multiple radars are coordinating their responses (in a Pareto-optimal manner) and if so, how to estimate the utilities of the individual radars?   
\item  How to construct a statistical detection test for utility maximization when we observe the adversary's radar's actions in noise?
\item  In the statistical detection test, how to  probe the adversary's radar by choosing our state  to minimize the Type-II  error of detecting if the adversary radar is a utility maximizer, subject to a constraint on the Type-I detection error?
\end{compactenum}

    \begin{figure}[h] \centering
            {\resizebox{8cm}{!}{
                \begin{tikzpicture}[node distance = 1cm, auto]
                  \tikzset{
    blocka/.style={rectangle, draw, line width=0.5mm, black, text width=4.5em, text centered,
                 minimum height=1em},
               line/.style={draw, -latex}}
    \node [blocka] (BLOCK1) {Sensor};
    \node [blocka, below of=BLOCK1,right of=BLOCK1,node distance=1.5cm] (BLOCK2) {Optimal Decision \\ Maker};
    \node [blocka, below of=BLOCK1,left of=BLOCK1,node distance=1.5cm] (BLOCK3) {Bayesian Tracker};

    \draw[Latex-] (BLOCK1) -| node[left,pos=0.8]{$\response_\dtime$}  (BLOCK2)  ;
    \draw[-Latex] (BLOCK1.west) -|   node[left,pos=0.8]{$\obs_\fasttime$} (BLOCK3);

    \draw[-Latex](BLOCK3) --  node[above]{$\belief_\fasttime$} (BLOCK2);

    \node[draw=none,fill=none] at (4.5,-1.5) (drone) {\includegraphics[bb=0 0 0 0,scale=0.07]{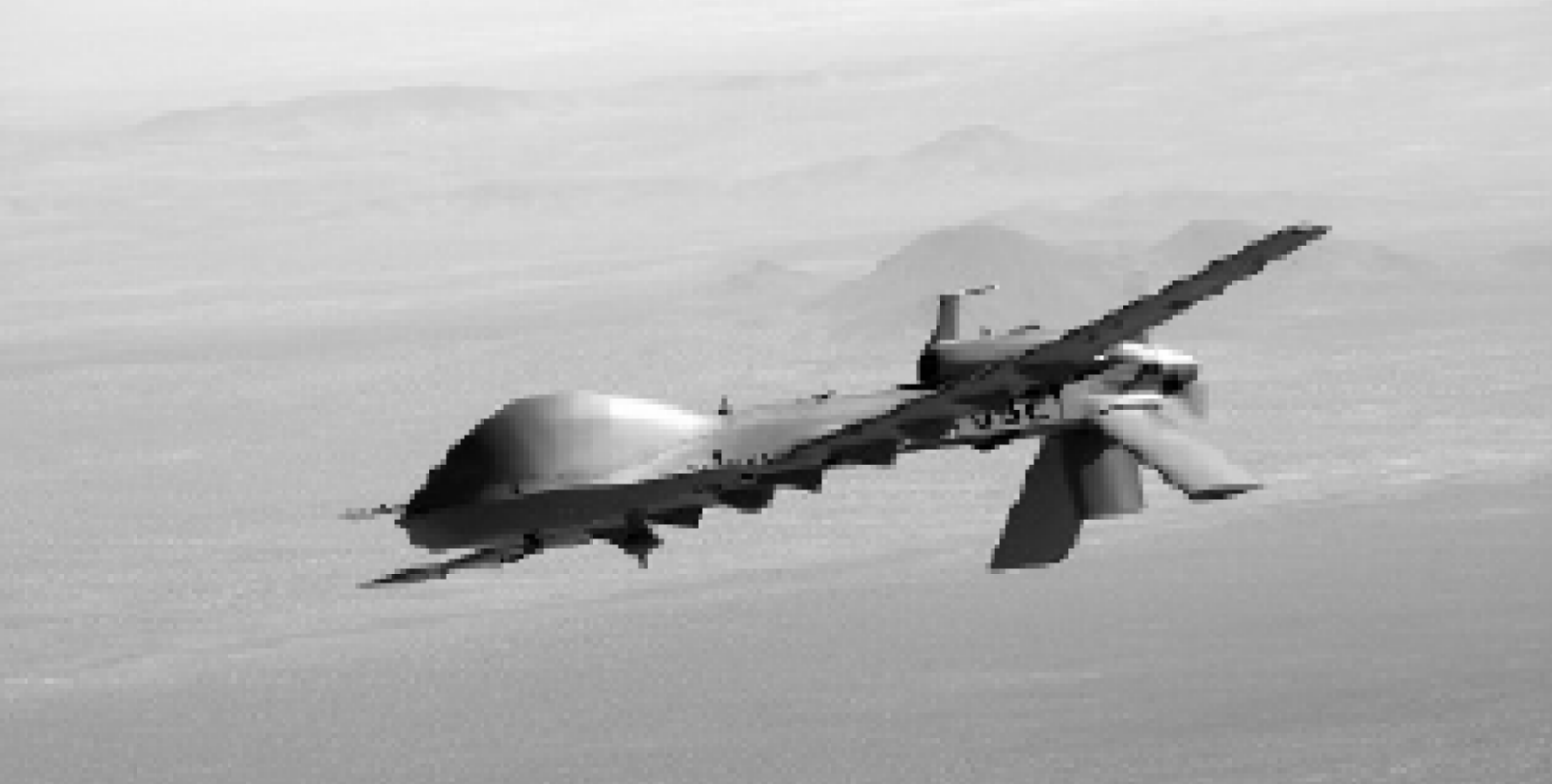}};
     \node[text width=2cm] at (5.5,-1.7) {{Our state $\state_\fasttime$}};

    \draw[-Latex,line width=2pt] (2,0)   -- node[above]{action $\response_\dtime$}(4,0);
    \draw[-Latex,line width=2pt] ([yshift=0.8cm]drone.west)   --   node[below]{probe $\probe_\dtime$} (2.8,-0.7);
    \node[draw] at (5.5,-3.0) {Our side};
    \node[draw] at (0.5,-3.0) {Adversary};
    \draw [dashed] (3.5,1) -- (3.5,-3);
   \end{tikzpicture}} }

\caption{IRL for identifying an adversary cognitive radar. Our state $\state_\fasttime$ (drone or electromagnetic signal), parametrized by $\probe_\dtime$ (purposeful acceleration maneuvers), probes the adversary radar. Here, $\fasttime$ denotes a fast time scale and $\dtime$ denotes a slow time scale.  Based on the noisy observation $\obs_\fasttime$ of our state, the adversary radar  responds with action~$\response_\dtime$. The IRL problem is to determine if the adversary radar response  $\beta_\dtime$  is consistent with constrained utility maximization.}
\label{fig:schematica}
\end{figure}

The central idea of this chapter is to apply \textit{revealed preference} theory  
from microeconomics as a constructive procedure for  IRL.  Revealed preference theory provides a necessary and sufficient condition for a dataset to be consistent with utility maximization, and then     constructs
 set-valued estimates of the utility function.  Estimating a utility
function is an ill-posed problem since argmax is an ordinal function (so any monotone increasing function of
the utility is also a valid utility). Therefore,  a point estimate of the utility, such as a least squares estimate,
is not useful. In contrast, set-valued estimates are important in this setting.
Given the set-valued estimates of the utility,  one can then predict (extrapolate)  future responses of the adversary and also construct confidence intervals around these predictions.

\subsubsection{Context. Revealed Preferences versus Classical IRL}

We will give a detailed discussion and literature review of how revealed preferences  relates to classical IRL in \secn \vref*{sec:pwe_bayesian_irl}. For now, we briefly compare the two approaches.

Consider classical  IRL for a finite-state finite-action MDP.
We assume the cost is only state dependent so that $c(x,u) = c(x)$.  Given the optimal policy and controlled transition matrices, the aim is
to reconstruct the cost vector~$c=[c(1),\ldots,c(\statedim)]^\p$.

Since the costs are not action-dependent,
 Bellman's equation implies that the optimal policy $\policy^*$ satisfies
$[\tp(\policy^*) - \tp(\action)] J_{\policy^*} \leq \zero$ for all $\action \in \actionspace$.
Also for an MDP,  the value function
 $J_{\policy^*} =  \big(I - \discount \tp(\optpolicy) \big)^{-1} \,c$ is a linear function of the cost vector $c$. 
 Since there are a finite number of actions $\actiondim$, in classical IRL, the unknown cost vector $c$ satisfies the  finite set of linear inequalities:
$$  \big(\tp(\policy^*) - \tp(u) \big)\, \big( I - \discount \tp(\policy^*) \big)^{-1} c \leq \zero , \quad u \in \actionspace=\{1,2,\ldots,\actiondim\}  .$$

The revealed preference approach discussed in this chapter generalizes classical IRL for  finite-state finite-action MDPs in three ways. First, it  identifies if the behavior is consistent with constrained utility maximization  and then estimates the set of utility functions that rationalize the dataset. Second, revealed preferences involves active learning in that the observer probes the system, whereas classical IRL is passive.  Third, revealed preference theory operates on $\reals^\probedim$ rather than over a finite set $\{1,\ldots,\actiondim\}$.


\section{Revealed Preferences and Afriat's Theorem} \label{subsec:afriat}
\index{revealed preferences! Afriat's theorem} \index{Afriat's theorem} 

In microeconomics, the principle of revealed preference offers a constructive
test to determine if an agent is a constrained utility maximizer, 
based on its observed   decisions  over time.  
Let $\reals^\probedim_{++}$ denote strictly positive vectors and $\reals_+^\probedim$ denote nonnegative vectors of dimension $\probedim$. We start by defining a constrained utility maximizer.

\begin{definition} \label{defn:utility_maximizer} Consider the time series dataset $\dataset = \{\probe_k,\response_k,k=1,\ldots,\horizon\}$ where probe vector $\probe_\dtime\in \reals_{++}^\probedim$, and  response vector $\response_\dtime \in \reals_+^\probedim$.
  An agent  is  a {\em utility maximizer} if
\begin{equation}
  \response_\dtime\in \argmax_{\probe_\dtime^\p \response \leq 1}\utility(\response) , \quad k=1,\ldots,\horizon
\label{eq:singlemaximization}
\end{equation}
where $\utility(\response)$ is a strictly monotone utility function. 
 \end{definition}

We say that a utility $\utility(\cdot)$  \textit{rationalizes} 
dataset $\dataset$ if Definition~\ref{defn:utility_maximizer} holds.  To avoid trivialities, we assume that $\probedim\geq 2$. (If $\probedim=1$, then any monotone utility rationalizes a utility maximizer.)

\begin{rem}
\item By strictly monotone\footnote{More generally, the results apply to a   nonsatiated utility function. A utility function $\utility$ is  nonsatiated if
for every $\epsilon> 0$ and $\response \in \reals_+^\probedim$, there exists $\bar{\response}$ in $\reals_+^\probedim$, such that $\| \response - \bar{\response}\|
\leq \epsilon$ implies that $\utility(\bar{\response}) >  \utility({\response})$.} utility, we mean that  
 an increase in any element of response $\response$ results in the utility
 $U(\response)$ strictly increasing.

\item The following motivation from economics is helpful:
 $\probe_k$ is the price vector  of $\probedim$ commodities at time $k$, 
and the vector $\response_k$ is the amount of these commodities bought by the agent. 
The constraint  $\probe_\tindx^\p \response \leq 1$
in~\eqref{eq:singlemaximization} models the  budget of an agent with 1 dollar.
Given the agent's dataset $\dataset$, the economist has two objectives. First, identify if the agent is rational, i.e.,  $\dataset$ is consistent with utility maximization. Second, estimate 
 the agent's utility $\utility$ so as to predict the agent's future  behavior. A similar rationale is  used to identify cognitive radars below.

\item  The linear budget constraint  $ \probe_\dtime^\p \response \leq 1$ in~\eqref{eq:singlemaximization} is without loss of generality. As explained after Theorem~\ref{thm:afriat} below, the constraint  can be replaced with $\probe_\tindx^\p \response \leq \budget_\tindx$ for any $\budget_\tindx \in \reals_{++}$.

\item The linear budget constraint $\probe_k^\p \response \leq 1$ implies $ \probe_k$ and $\response_k$ are vectors of the same dimension. \secn \ref{sec:forges}  considers nonlinear budgets that allow $\probe_k$ and $\response_k$ to be of different dimensions.
  
\item Owing  to the monotone assumption on the utility, it is clear that the constraint  $\probe_\tindx^\p \response \leq 1$ is active at $\response_\tindx$.  That is, the agent  spends its entire unit dollar income and its response $\response_k$ satisfies $\probe_\tindx^\p \response_\tindx  = 1$.
The strictly  monotone assumption  on the utility also  rules out  trivial cases, such as a constant utility  which can be  optimized by any response.
\end{rem}

\subsubsection{Nonparametric Test for Utility Maximization -- Afriat's Theorem}

\index{Afriat's theorem|(}
The celebrated
Afriat's theorem \cite{Afr67,Afr87} provides a necessary and sufficient
  condition for a finite dataset $\dataset$ to have originated from a utility maximizer. 
 
\begin{theorem}[Afriat's Theorem] Given  dataset $\mathcal{D}=\{\probe_\tindx,\response_\tindx, \tindx\in 1,2,\dots,\Tindxter\}$,  where  $\probe_\dtime\in \reals_{++}^\probedim$, and   $\response_\dtime \in \reals_+^\probedim$, $\probedim\geq 2$, the following statements are equivalent:
\begin{compactenum}
  \item The agent is a utility maximizer (Definition~\ref{defn:utility_maximizer}) and there exists a monotone  utility function that rationalizes~$\dataset$.
\item There exist scalars $\{\uaf_\dtime,\lambda_\dtime, \dtime=1,\ldots,\horizon\}$ where $\lambda_\dtime>0$ such that  the following   inequalities (called Afriat's inequalities) have a feasible solution:
\begin{equation}
\uaf_\dtimen-\uaf_\tindx-\lambda_\tindx \probe_\tindx^\p (\response_\dtimen-\response_\tindx) \leq 0 \quad \forall \tindx,\dtimen\in\{1,2,\dots,\Tindxter\}.
\label{eqn:AfriatFeasibilityTest}
\end{equation}

\item A family of continuous, concave, monotone, piecewise linear  utility functions that rationalizes $\dataset$  is
\begin{equation}
\hat{\utility}(\response) = \underset{\tindx\in\{1,\ldots, \Tindxter\}}{\operatorname{min}}\{\uaf_\tindx+\lambda_\tindx \probe_\tindx^\p(\response-\response_\tindx)\}
\label{eqn:estutility}
\end{equation}
where $\{\uaf_\dtime,\lambda_\dtime, \dtime=1,\ldots,\horizon\}$  satisfy Afriat's inequalities~\eqref{eqn:AfriatFeasibilityTest}.
\item The dataset $\mathcal{D}$ satisfies the Generalized Axiom of Revealed Preference (GARP): For every ordered subset
\index{generalized axiom of revealed preference (GARP)}
  $   \{i_1,i_2,\ldots,i_\lind\} \subset \{1,2,\ldots,N\}$, if
\begin{equation}
  \probe_{i_1}^\p\, \response_{i_2} \leq \probe_{i_1}^\p \,\response_{i_1}, \;\;
    \probe_{i_2}^\p\, \response_{i_3} \leq \probe_{i_2}^\p \,\response_{i_2}, \ldots,
    \probe_{i_{\lind-1}}^\p \, \response_{i_\lind} \leq \probe_{i_{\lind-1}}^\p\, \response_{i_{\lind-1}},  \label{eq:garp}
  \end{equation}
then it holds that
  $\probe_{i_\lind}^\p \response_{i_1} \geq \probe_{i_\lind}^\p \response_{i_\lind}$.
\end{compactenum}  
\label{thm:afriat}
\end{theorem}

\begin{rem}
  
\item A simplified  proof of Afriat's theorem is provided after the  remarks. However, if the utility is assumed to be concave, increasing and differentiable, then the proof of the necessity of  Afriat's inequalities for utility maximization is easy.
The  Kuhn--Tucker conditions imply
$ \lambda_k > 0$ and $  \nabla U(\response_k) = \lambda_k \probe_k$,
$k=1,\ldots, \horizon$. Concavity implies $U(\response) \leq U(\response_k) + \nabla^\p U(\response_k)\, (\response - \response_k)$. Then defining $\uaf_k = U(\response_k)$ and choosing
$\response=\uaf_s$ yields Afriat's inequalities~\eqref{eqn:AfriatFeasibilityTest}.

\item Afriat’s theorem provides a  nonparametric constructive 
test to identify if an agent is a budget-constrained utility maximizer.
A remarkable feature of Afriat's theorem is that if the dataset can be rationalized by a  utility function, then it can be rationalized
by a continuous, concave, monotonic utility function.  That is, violations of continuity, concavity, or monotonicity cannot be detected with  a finite number of  observations.

\item Since Afriat's theorem is necessary and sufficient for a utility maximizer, the true utility function $U$ satisfies Afriat's inequalities~\eqref{eqn:AfriatFeasibilityTest} and also GARP; this follows from the necessity. 
  For the true utility $\utility$, the scalars $(\uaf_\dtime,\lambda_\dtime)$ in Afriat's inequality~\eqref{eqn:AfriatFeasibilityTest} are
  \begin{equation}
    \label{eq:utility_project}
    \uaf_\dtime^*= \utility(\response_\dtime), \quad  \lambda^*_\dtime = \frac{e_i^\p  \,\nabla U(\response_\dtime)}{\probe_\dtime(i)},
    \quad k = 1,\ldots,\horizon,
  \end{equation}
  for any choice of $i \in \{1,\ldots,\probedim\}$, where $e_i$ is the unit $\probedim$-vector with $i$th element 1.
  
Any utility  $\putility$ that rationalizes dataset $\dataset$ has the following characterization. Define
\begin{equation}
    \label{eq:putility_project}
    \tilde{\uaf}_\dtime= \putility(\response_\dtime), \quad \tilde{\lambda}_\dtime = \frac{e_i^\p  \,\nabla \putility(\response_\dtime)}{\probe_\dtime(i)},
    \quad k = 1,\ldots,\horizon.
  \end{equation}
If $\{\tilde{\uaf}_\dtime,\tilde{\lambda}_\dtime,\dtime=1,\ldots,\horizon\}$ satisfy Afriat's inequalities~\eqref{eqn:AfriatFeasibilityTest}, then $\putility$  rationalizes  $\dataset$, and  $\putility$ coincides with $\hutility$ defined in Statement 3 of Afriat's theorem.

\item The reconstructed piecewise linear  utility (\ref{eqn:estutility}) is not unique and is ordinal by construction. Ordinal means that any positive monotone increasing  transformation of the utility function will also satisfy Afriat's theorem. This is  why the budget constraint $\probe_\dtime^\p \response \leq 1$ is without loss of generality; it can be scaled by an arbitrary positive  constant and  Afriat's theorem still holds.
  The reconstructed  utility (\ref{eqn:estutility}) is the lower envelope of a finite number of hyperplanes.  

We did not impose  nonnegative constraints on $\uaf_\tindx, \tindx = 1,\ldots\Tindxter$ in Afriat's inequalities~\eqref{eqn:AfriatFeasibilityTest}. This is because if a solution exists with unrestricted sign, then we can simply add a large positive  constant to each $\uaf_\tindx$ and the inequalities remain unchanged since they involve $\uaf_\dtimen-\uaf_\tindx$. So if Afriat's inequalities are feasible, then  nonnegative utilities can be constructed using~\eqref{eqn:estutility}.

\item {\em Algorithms.}
Verifying  GARP  (Statement 4 of Afriat's  theorem) on a dataset $\dataset$ comprising $N$ points can be done using Warshall's algorithm with $O(\Tindxter^3)$ computations~\cite{Var82,FST04}. Alternatively, determining if Afriat's inequalities (\ref{eqn:AfriatFeasibilityTest}) are feasible can be done via a linear 
 programming (LP) feasibility test involving $\horizon^2$ linear constraints with $2\horizon$ variables.
 \index{linear programming! Afriat's theorem}
 We cannot directly apply an LP to~\eqref{eqn:AfriatFeasibilityTest} due to the strict inequality constraints $\lambda_\tindx > 0$.
 Since Afriat's inequalities~\eqref{eqn:AfriatFeasibilityTest} are homogeneous in $\uaf_s - \uaf_k$ and $\lambda_k$, i.e., scaling by a positive constant does not affect the inequalities, we can scale
 by $ 1/\min_{k} \lambda_k$. So 
%
it suffices to solve the following LP:   
 \begin{equation}
   \label{eq:fostel_lp}
    \min_{\lambda, \uaf} \,0. \lambda + 0. \uaf  \quad 
     \text{ subject to }   \lambda_\tindx \geq 1,  \tindx \in\{1,2,\dots,\Tindxter\}
                            \text{ and \eqref{eqn:AfriatFeasibilityTest}. }
 \end{equation}

\item \label{rem5:afriat} {\em GARP~\eqref{eq:garp} checks for cyclic consistency}.  At price $ \probe_{i_1}$ the response $\response_{i_1}$ is more expensive than $\response_{i_2}$; yet the agent chose this more expensive response. This implies that  the agent
  must assign a higher utility to $\response_{i_1}$ over  $\response_{i_2}$, meaning
  $U(\response_{i_1}) \geq U(\response_{i_2})$. Then~\eqref{eq:garp}  implies 
  $U(\response_{i_1}) \geq U(\response_{i_2}) \cdots \geq U(\response_{i_\lind})$, and  the condition  $\probe_{i_\lind}^\p \response_{i_1} \geq \probe_{i_\lind}^\p \response_{i_\lind}$   means 
  $U(\response_{i_1}) \geq  U(\response_{i_\lind})$. So GARP checks for cyclic consistency,
  namely  $U(\response_{i_1}) \geq U(\response_{i_2}) \cdots \geq U(\response_{i_\lind})
  \implies U(\response_{i_1}) \geq  U(\response_{i_\lind})$.
If $ U(\response_{i_1}) <  U(\response_{i_\lind})$ then GARP would be violated, and the agent is not a utility maximizer.

\item
In terms of logical quantifiers,   GARP is of the form $p \leq q \implies x \geq y$. This is equivalent to
  $p\leq q \text{ and } x \leq y \implies x=y$. So
  another way of writing GARP is: 
  For every ordered subset
$   \{i_1,i_2,\ldots,i_\lind\} \subset \{1,2,\ldots,N\}$, if
  \begin{equation}
    \label{eq:GARP_equiv1}
    \probe_{i_1}^\p\, \response_{i_2} \leq \probe_{i_1}^\p \,\response_{i_1}, \;
    \probe_{i_2}^\p\, \response_{i_3} \leq \probe_{i_2}^\p \,\response_{i_2}, \ldots,   \probe_{i_{\lind-1}}^\p \, \response_{i_\lind} \leq \probe_{i_{\lind-1}}^\p\, \response_{i_{\lind-1}}\;
 \text{ and }    \probe_{i_\lind}^\p \, \response_{i_1} \leq \probe_{i_\lind}^\p\, \response_{i_\lind}
  \end{equation}
  then  each inequality is an equality.
 Define the square matrix $A$  with elements
  \begin{equation}
    \label{eq:A_rp}
    a_{ij} = \probe_i^\p(\response_j - \response_i) , \quad i,j \in \{1,\ldots,N\}.
  \end{equation}
  Then GARP~\eqref{eq:GARP_equiv1} can be written in clean notation as the acyclic condition
\begin{equation}
  \label{eq:GARP_final}
  a_{i_1 i_2} \leq 0, \, a_{i_2 i_3} \leq 0,\, \ldots \,, a_{i_{\lind-1} i_\lind} \leq 0, \, a_{i_\lind i_1} \leq 0 \implies a_{i_1 i_2} = a_{i_2 i_3} = \cdots = a_{i_\lind i_1} = 0 ,
\end{equation}
i.e., cycles such as $a_{i_1 i_2} < 0, \, a_{i_2 i_3} < 0,\, \ldots \,, a_{i_{\lind-1} i_\lind} < 0, \, a_{i_\lind i_1} < 0$ violate GARP.

An $\horizon\times \horizon$  matrix $A$ that satisfies~\eqref{eq:GARP_final} with diagonal elements $a_{ii} = 0$ is called  \textit{cyclically consistent}.
Together with Remark \ref{rem5:afriat}, 
it follows that if an agent is a utility maximizer, then
\begin{equation}
  \label{eq:aij_garp}
a_{ij} \leq 0 \iff
U(\response_i) \geq U (\response_j).
\end{equation}


\item {\em Cyclic monotonicity}. We start with  the question:  If a map  $g \colon\reals^\probedim\to \reals^\probedim$ is evaluated at  $g(x_1),\ldots,g(x_\lind)$, how to determine if they are subgradients  of a convex function $f\colon \reals^\probedim\to \reals$?
  By definition, the subgradients of a convex function satisfy $(x_{i+1} - x_{i})^\p g(x_i) \leq f(x_{i+1})-f(x_i)$.
    
  Theorem 24.8, p.\ 238, of \cite{Roc70} in convex analysis  states:
 The map $g$ is a subgradient  of a convex function $f$ iff $g$ is  \textit{cyclically monotone}, i.e., 
  $ (x_2 - x_1)^\p g(x_1) + (x_3-x_2)^\p g(x_2) + \cdots + (x_1-x_\lind)^\p g(x_\lind) \leq 0$ for any set of pairs $(x_i,g(x_i))$, $i=1,\ldots,\lind$ where $\lind$ is arbitrary.

{\em Afriat's inequalities~\eqref{eqn:AfriatFeasibilityTest} satisfy this cyclic monotonicity}:
Since they satisfy  $\fun_{i_k} - \fun_{i_j} \leq \lambda_{i_j} a_{i_j i_k}$ and because
$(\fun_{i_2} - \fun_{i_1}) + \cdots + (\fun_{i_{\ell}}-\fun_{i_{\ell-1}}) + (\fun_{i_1} - \fun_{i_\ell})=0$ for any $\ell$, it follows that
$$    \lambda_{i_1} a_{i_1 i_2} + \lambda_{i_2} a_{i_2 i_3} + \cdots + \lambda_{i_{\ell-1}} a_{i_{\ell-1},i_\ell} + \lambda_{i_\ell} a_{i_\ell,i_1} \geq 0 . $$ 
Hence, Statement 3 of Afriat's theorem (concave utility) follows from the  convex analysis result.

\item {\em  Predicting future response}. Given dataset $\dataset=\{\probe_\dtime,\response_\dtime, \dtime=1,\ldots,\horizon\}$, suppose the analyst has used Afriat's theorem to verify that the agent is a  utility maximizer. How can the analyst exploit this? Specifically, given a new probe signal $\probe_{\horizon+1}$, how can the analyst predict (extrapolate) the agent's response $\response_{\horizon+1}$? The solution is to construct the set of predicted responses that are consistent with GARP.  The predicted  response $\response_{\horizon+1}$ to the new  probe $\probe_{\horizon+1}$ is the set
  \begin{multline}
    \label{eq:extrapolate}
    \response_{\horizon+1} \in   
    \big\{\response \in \reals_+^\probedim\colon \{\probe_1,\response_1,\ldots, \probe_\horizon, \response_\horizon\} \cup \{\probe_{\horizon+1},\response\} \text{ satisfies GARP,} \quad
\probe_{\horizon+1}^\p \response = 1 \big\}.
\end{multline}

\item {\em Selecting a representative utility in IRL}. Afriat's theorem yields a \textit{set} of utility functions that rationalize a dataset.
However, analysts often prefer a single representative utility.
  One can choose the utility estimate corresponding to any single feasible point $\uaf_{1:\horizon},\lambda_{1:\horizon}$ in the convex polytope that satisfies Afriat's inequalities~\eqref{eqn:AfriatFeasibilityTest}.
Several choices are available: Chebyshev center, analytic center, centroid, and max-margin of the convex polytope. Once 
$\uaf_{1:\horizon},\lambda_{1:\horizon}$ are selected, then~\eqref{eqn:estutility} is used to construct the corresponding estimated utility $\hat{\utility}$. \index{revealed preferences! max-margin IRL}
\end{rem}

\summar
 Afriat's theorem can be interpreted as set-valued system identification of an \emph{argmax} nonlinear system, where~\eqref{eqn:estutility} generates a set of utility functions that rationalize the finite dataset.  The revealed
preference approach is {\em data centric} -- given a dataset,  we wish to determine if it is consistent with utility maximization.
This contrasts with  {\em model-centric}  signal processing   where
one postulates an objective (typically convex) and then designs  optimization algorithms. 
At its core, Afriat's theorem equates  cyclic consistency (GARP) with cyclic monotonicity.

\subsubsection{Proof of Afriat's Theorem} 

Afriat's Theorem~\ref{thm:afriat} asserts the equivalence of four statements, which we label
$\tc{1}, \tc{2}, \tc{3}$ and $\tc{4}$.
The proof below shows that $\tc{3}\simplies  \tc{1} \simplies \tc{4} \simplies \tc{2} \simplies \tc{3}$. Showing
$\tc{4} \simplies \tc{2}$ is  nontrivial. 

\begin{senumerate}
\item $ \tc{3} \simplies \tc{1}$ trivially.
  
  \item  $\tc{1} \simplies \tc{4}$:
Suppose 
$a_{i_1 i_2} \leq 0, a_{i_2 i_3} \leq 0,\ldots a_{i_\lind i_1} \leq 0$. For a utility maximizer,  from~\eqref{eq:aij_garp}, it follows that  $U(\response_{i_1}) \geq U(\response_{i_2}) \geq \cdots \geq U(\response_{i_\lind}) \geq U(\response_{i_1})$. This means $U(\response_{i_1}) = U(\response_{i_2}) =\cdots = U(\response_{i_\lind}) = U(\response_{i_1})$ and so $a_{i_1 i_2} = a_{i_2 i_3} = \cdots = a_{i_\lind i_1} = 0$. This is  GARP~\eqref{eq:GARP_final}.

\item  
  $\tc{2} \simplies \tc{3}$: Start with~\eqref{eqn:AfriatFeasibilityTest} and consider the
constructed utility~\eqref{eqn:estutility}.
Then for any  $\response \in \reals^\probedim$,
\begin{equation}
  \label{eq:afriat_3}
\hat{U}(\response)  \stackrel{\text{(a)}}{\leq} \uaf_j + \lambda_j\, \probe_j(\response - \response_j) \stackrel{\text{(b)}}{\leq} 
 \uaf_j \stackrel{\text{(c)}}{=}  \hat{U}(\response_j).
\end{equation}
Here,  (a) follows from~\eqref{eqn:estutility} since the right side has a  $\min$;
 (b) holds since $\probe_j^\p(\response - \response_j) \leq 0$ because $\response$ is feasible and the budget constraint is binding at $\response = \response_j$; (c)
 follows by substituting~\eqref{eqn:AfriatFeasibilityTest}.  

From (i), (ii), (iii) we have: $\tc{2}$ (Afriat's inequalities) $\simplies  \tc{4}$ (GARP).


\item  $\tc{4} \text{ (GARP)} \simplies \tc{2} \text{ (Afriat's inequalities)}$: 
This is the main novelty of Afriat's theorem.

\begin{lemma} \label{lem:afriat_cyclic}
  If $A$ is a cyclically consistent matrix,  i.e., \eqref{eq:GARP_final} holds and diagonal elements $a_{ii}=0$, $i=1,\ldots,\probedim$, then Afriat's inequalities are  feasible.
\end{lemma}  
We prove  the lemma for the simplified case when $a_{ij} \neq 0$, $i \neq j$. 
 Start with the following claim: There is at least one  row $i$, such that all elements $a_{ij} > 0$.  Otherwise GARP is contradicted.
 We  now show by induction  that GARP implies that Afriat's inequalities have a   solution.  For $n=1$, choose $\lambda_1=1$ and $\uaf_1$ arbitrarily.
 Now for the induction step. First renumber the rows and columns of $A$ so that
 $a_{nj} > 0$ for $j=1,\ldots,n-1$. This is always possible due to the above claim.
 Suppose there exist $\uaf_1,\ldots,\uaf_{n-1}$ and $\lambda_1,\ldots,\lambda_{n-1}> 0$ so that Afriat's inequalities hold:
 $$ \uaf_j \leq \uaf_i + \lambda_i\, a_{ij} , \quad i \neq j, i,j \in \{1,\ldots,n-1\}.$$
 Then choose
 $ \uaf_n \leq \min_{i \in \{1,\ldots,n-1\}} \{ \uaf_i + \lambda_i\, a_{in} \}, $
 and also choose $\lambda_n > 0$ so that
 $$ \uaf_j \leq \uaf_n + \lambda_n \, a_{nj}, \quad j = 1,\ldots,n-1.$$
 Since all off-diagonal elements of the $n$th row are strictly positive, $\lambda_n$ can be chosen sufficiently large  so that these $n-1$ inequalities hold. This completes the induction.
 
\remar The proof of Lemma~\ref{lem:afriat_cyclic} is  more complex  if $a_{nj} = 0$, since increasing $\lambda_n$ does not help to fix the inequality for the specific $n$ and $j$; see~\cite{FST04,Var82} for details.
\end{senumerate}
\index{Afriat's theorem|)}

\section{Revealed Preferences with Nonlinear Budget Constraint}
\label{sec:forges}

The purpose of this section is to  extend Afriat's theorem to a nonlinear budget constraint.
As  discussed below, in IRL for cognitive sensor,  the nonlinear  budget constraint (nonlinear in $\response$)  emerges naturally from the covariance of the Kalman filter tracker. \index{revealed preferences! Forges and Minelli's theorem}
Also,  a nonlinear budget constraint facilitates IRL with multiple budget constraints (see Remark \ref{afriat:multiple} below).

\begin{theorem}[Nonlinear budget. Forges and Minelli  \cite{FM09}]
  For $\dtime = 1,\ldots,\horizon$, let $\budgetg_\dtime\colon \reals^\probedim \rightarrow \reals$
denote increasing continuous functions that model a nonlinear budget, and are known to the analyst.
Let $\budget_{\dtime} = \{\response \in \reals^\probedim_+| \budgetg_\dtime(\response) \leq 0 \}$. Consider  the dataset $\dataset = \{\response_\dtime,\budget_\dtime, \dtime=1,\ldots,\horizon\}$.
Assume $\budgetg_\dtime(\response_\dtime) = 0$ for $\dtime =1,\ldots,\horizon$.  Then the
following conditions are equivalent:
  \begin{compactenum}
  \item There exists a monotone continuous utility function $\utility$ that rationalizes dataset $\dataset$.
 That is
    $$ \response_\dtime = \argmax_{\response \in \reals^\probedim_+} \utility(\response), \quad \text{ subject to }\;
    \budgetg_\dtime(\response) \leq 0. $$
         \item There exist scalars $\{\uaf_\dtime,\lambda_\dtime, \dtime=1,\ldots,\horizon\}$ where
$\lambda_\dtime>0$ such that the following inequalities have a feasible solution:
\begin{equation} \uaf_\dtimen - \uaf_k -\lambda_\tindx
\budgetg_\tindx(\response_\sindx) \leq 0 \quad \forall \tindx,\dtimen\in\{1,2,\dots,\horizon\}.\
\label{eq:nonlinearFeasibilityTest}
\end{equation}

\item A family of continuous monotone utility functions that rationalize the
  dataset is
    \begin{equation} \hat{\utility}(\response) = \underset{\tindx\in
\{1,2,\dots,\horizon\}}{\operatorname{min}}\{\uaf_\tindx+\lambda_\tindx\, \budgetg_\tindx(\response)
\} 
\label{eq:nonlinearutility}
\end{equation}
where $\{\uaf_\dtime,\lambda_\dtime,\dtime=1,\ldots,\horizon\}$ satisfy (\ref{eq:nonlinearFeasibilityTest}).
                              
  \item The dataset $\{\response_\dtime,\budget_\dtime, \dtime=1,\ldots,\horizon\}$ satisfies the
Generalized Axiom of Revealed Preference (GARP): For every ordered subset $ \{i_1,i_2,\ldots,i_\lind\}
\subset \{1,2,\ldots,N\}$, if
\begin{equation}
  \label{eq:nonlinear_garp}
  \budgetg_{i_1}(\response_{i_2}) \leq 0, \; \budgetg_{i_2}(\response_{i_3}) \leq 0
    , \ldots, \budgetg_{i_{\lind-1}}(\response_{i_\lind}) \leq 0
  \end{equation}   
  then it holds that $\budgetg_{i_\lind}(\response_{i_1}) \geq 0$.  \index{generalized axiom of revealed preference (GARP)! nonlinear budget}
  \end{compactenum}
  \label{thm:nonlinear_afriat}
\end{theorem}

\begin{rem}
\item
Afriat's theorem (Theorem~\ref{thm:afriat}) assumes the agent has a linear budget constraint $\budgetg_\dtime(\response) = \probe_\dtime^\p\,\response - 1 \leq 0$. 
We obtain  Afriat's theorem  by setting   $\budgetg_\dtime(\response) = \probe_\dtime^\p\,(\response - \response_\dtime)$ since $\probe_\dtime^\p \response_\dtime = 1$.
\item Unlike Afriat's theorem, the reconstructed utility function $ \hat{\utility}(\response)$ is not necessarily  concave.
\item Like Afriat's theorem,   (\ref{eq:nonlinearFeasibilityTest}) comprises  linear inequalities in $ \uaf_\tindx, \lambda_\tindx$. Feasibility can be checked using an LP solver. Also, for prediction of future responses, the formula~\eqrefp{eq:extrapolate} applies with GARP now defined as~\eqref{eq:nonlinear_garp}. Finally, the requirement
  $\budgetg_\dtime(\response_\dtime) = 0$ implies the budget constraint is active, as in Afriat's theorem.
\item
The proof of Theorem~\ref{thm:nonlinear_afriat} is similar  to that of  Afriat's theorem. Use the $\horizon\times \horizon$  matrix $A$ with elements
$a_{ij} = \budgetg_i(\response_j)$  instead of $a_{ij} =  \probe_i^\p(\response_j - \response_i) $ in parts (i), (ii) and (iv) of the proof of Afriat's theorem. For part (iii) of the proof,  in~\eqref{eq:afriat_3} replace  $ \probe_j^\p(\response- \response_j)$ with  $\budgetg_j(\response)$, and use the fact that the budget is binding at $\response = \response_j$, i.e., $\budgetg_j(\response_j) = 0$ as specified in  the theorem.

\item {\em IRL with multiple constraints}: \label{afriat:multiple}
Theorem~\ref{thm:nonlinear_afriat}  allows us to achieve IRL with multiple known budget constraints.
  Suppose a cognitive sensor optimizes a utility $\utility(\response)$ subject to $I$  constraints $\{\response  \in \reals^\probedim_+ \colon h_{k,i}(\response)\leq 0, \,i=1,\ldots,I\}$, where  each $h_{k,i}(\response)$ is  continuous and increasing in $\response$.
This is equivalent to the
single nonlinear constraint
  $\budgetg_\dtime(\response) = \max_i\{h_{k,i}(\response)\}$, with  budget set
  $\{\response \in \reals^\probedim_+ \colon \budgetg_\dtime(\response) \leq 0\}$. Clearly, this satisfies the assumptions of Theorem~\ref{thm:nonlinear_afriat}.  \index{revealed preferences! multiple constraints}
\end{rem}

\section{Identifying Optimal Waveform Adaptation in Cognitive Radar}
\label{sec:kf}
\index{revealed preferences! IRL for cognitive radar}
\index{cognitive radar! IRL for waveform adaptation}

Waveform optimization is  perhaps the most important functionality of a  cognitive radar. 
A cognitive radar optimizes its waveform by adapting  its  ambiguity function. This section discusses how the revealed preferences theory of the previous two sections can be used to  identify  such cognitive behavior of the adversary's radar when it  deploys a Bayesian  filter as a  tracker. For concreteness,  we assume that the adversary's cognitive radar uses a  Kalman filter tracker. Since the probe and response signal evolve on a slow time scale $\dtime$ (described below), we assume that both the radar and us (IRL) have perfect measurements of probe $\probe_\dtime$ and response $\response_\dtime$.

We assume that a cognitive radar adapts its waveform by maximizing  a utility function in the sense  of (\ref{eqn:estutility}) with  a possibly nonlinear  budget constraint.
We  formulate  linear  budget constraints  in terms of the Kalman filter  error covariance (specified by the algebraic Riccati equation) where $\probe_\dtime$ and $\response_\dtime$ are the spectra (eigenvalues)  of the state and covariance noise matrices of the state space model. We call this as  \textit{spectral revealed preferences}.

\subsection{Waveform Adaptation by Cognitive Radar} \label{sec:waveform}

We now  formalize the radar--target interaction model, so as to  construct a spectral IRL algorithm. We abstract the radar--target interaction  model into an algebraic Riccati equation (ARE)  which depends on the  probe and response signals.
Suppose a radar adapts its waveform  while tracking a  target (us) using a Kalman filter.
Our probe input comprises  maneuvers that modulate the spectrum (eigenvalues) of the state noise covariance matrix. The radar responds with an optimized waveform which modulates the spectrum of the observation noise covariance matrix. 


As is widely used in target tracking, we assume 
linear Gaussian dynamics for the target's kinematics and
 linear Gaussian measurements at the radar: 
\beq \label{eq:lineargaussian}
\begin{split}
\state_{\fasttime+1} &= \statem\, \state_\fasttime  + \snoise_\fasttime(\probe_\dtime), \quad \state_0 \sim \belief_0 \\
\obs_\fasttime &= \obsm\, \state_\fasttime + \onoise_\fasttime(\response_\dtime).
\end{split}
\eeq
 Here $\fasttime$ indexes the fast time scale on which the target evolves, while  $\dtime$ indexes the slow time scale on which IRL is performed. Also,
$\state_\fasttime \in \statespace = \reals^\statedim$ is ``our'' state with
initial density $\belief_0 = \normal(\hat{\state}_0,\kalmancov_0)$,
 $\obs_\fasttime \in \obspace = \reals^\obsdim$ denotes the cognitive radar's observations,
 $\snoise_\fasttime\sim \normal(0,\snoisecov(\probe_\dtime))$,
 $\onoise_\fasttime \sim \normal(0,
\onoisecov(\response_\dtime))$
and 
  $\{\snoise_\fasttime\}$,  
  $\{\onoise_\fasttime\}$ are mutually independent  i.i.d.\ processes.
  When $\state_\fasttime$ denotes  the [x,y,z] position and velocity components of the target (so $\state_\fasttime \in \reals^6$) then  with  $T$ denoting the sampling interval, 
 \beq \statem_{6 \times 6} = \diag\Bigl[ \begin{bmatrix}  1 & T \\ 0 & 1
  \end{bmatrix}, \begin{bmatrix}  1 & T \\ 0 & 1
  \end{bmatrix}, \begin{bmatrix}  1 & T \\ 0 & 1
  \end{bmatrix}  \Bigr] . \label{eq:target} \eeq  
 
  
We  have explicitly  indicated the dependence of the state noise covariance $\snoisecov$
on our probe signal $\probe_\dtime$, and 
the  observation noise covariance $\onoisecov$ on the radar's response signal $\response_\dtime$.

Let us first explain $\onoisecov(\response)$.  When the radar controls its ambiguity function, it controls the measurement noise covariance $\onoisecov$. This comes at a cost:
Reducing the  observation noise covariance of a target results in  increased visibility of the radar (and therefore higher threat) or increased covariance of other targets. 
Note that the radar  reconfigures its receiver (matched filter) each time it chooses a waveform; (\ref{eq:lineargaussian}) abstracts all  physical layer aspects of the radar response into the observation noise covariance $ \onoisecov(\response_\dtime)$. The precise structure of how $ \onoisecov(\response_\dtime)$ depends on the selected  waveform of the radar (triangular pulse,  linear FM  chirp, etc.)  is discussed in \cite{KAEM20}.

Next we explain $\snoisecov(\probe)$. We 
 probe  the adversary radar   via purposeful maneuvers by  modulating our  state covariance matrix  $\snoisecov$  in (\ref{eq:lineargaussian}) by $\probe_\dtime$.
 Based on observation sequence $\obs_1,\ldots,\obs_\fasttime$, the tracker in the radar computes the  posterior $\belief_\fasttime = \normal(\hat{\state}_\fasttime,\kalmancov_\fasttime)$ where $\hat{\state}_\fasttime$ is the conditional mean
  state   estimate and $\kalmancov_\fasttime$ is the covariance. These are computed by the  Kalman filter. 

  Finally, we explain the interaction of our probe $\probe$ and radar response $\response$ in the tracker.
Under the assumption that the model parameters in (\ref{eq:lineargaussian}) satisfy $[\statem,\obsm]$ is detectable and $[\statem,\sqrt{\snoisecov}]$ is stabilizable, 
the asymptotic  covariance $\kalmancov_{\fasttime+1|\fasttime}$ as
$\fasttime\rightarrow \infty$ is the unique nonnegative definite solution of  the \textit{algebraic Riccati equation} (ARE):
\begin{equation}
     \ARE(\probe,\response,\kalmancov) \ole 
    - \kalmancov + \statem  \big(\kalmancov -  
\kalmancov \obsm^{\p}  \left[ \obsm \kalmancov \obsm^\p + \onoisecov(\response) \right]^{-1}
\obsm \kalmancov \big)  \statem^\p  +  \snoisecov(\probe) = 0.
\label{eq:arerp}
\end{equation}
$ \ARE(\probe,\response,\kalmancov) $ is a symmetric $\probedim \times \probedim $ matrix.
We denote
the solution of ARE at epoch $\dtime$ as  $\kalmancov_\dtime^*(\probe,\response)$.

\subsection{Testing for Cognitive Radar: Spectral Revealed Preferences} \label{sec:linearbudget}

We now discuss how Afriat's theorem can be used to identify  if a radar is maximizing a  utility function $\utility(\response)$ (unknown to us) based on the predicted covariance of the target.
Our aim is to justify the linear budget constraint $\probe_\dtime^\p \response \leq 1$ in Afriat's theorem.
Specifically,
suppose
\begin{compactenum}
\item Our probe  $\probe_\dtime$  specifies our maneuvers. It is the vector of eigenvalues of the positive definite matrix
  $\snoisecov$. Note $\probe_\dtime$ is the incentive (or price) vector the radar receives for tracking the target.
\item The radar response $\response_\dtime$ is the vector of eigenvalues of the positive definite matrix~$\onoisecov^{-1}$ due to choosing a waveform. $\response_\dtime$ is the amount of resources (consumption) the radar spends.
\end{compactenum}
If the radar is 
cognitive, then it chooses  its  waveform parameter $\response_\dtime$  at each  time epoch $\dtime$  as
\beq \response_\dtime\in \argmax_{\probe_\dtime^\p \response \leq 1}\utility(\response). \label{eq:radaropt}
\eeq
Here $\utility$ is a monotone increasing function of $\response$.
Since there is no natural ordering of eigenvalues, our assumption is that
  $\utility(\response)$ is a symmetric function\footnote{Examples of  symmetric utility functions $\utility(\response)$ include $\tr(\onoisecov^{-1}(\response))$, $\det(\onoisecov^{-1}(\response))$, nuclear norm, etc. The assumption of symmetry is only required when we choose $\response $ to be the vector of eigenvalues since there is no natural ordering of the eigenvalues in terms of the ordering of the elements of the matrix. \label{foot:sym_utility}} of $\response$.

Then Afriat's theorem   (Theorem \ref{thm:afriat}) can be used to detect utility maximization and construct a utility function that rationalizes the response of the radar.
Recall that  the 1 in the right hand side of the budget $\probe_\dtime^\p \response \leq 1 $ can be replaced by any nonnegative constant.

Let us justify the linear budget constraint $\probe_\dtime^\p \response \leq 1$
in (\ref{eq:radaropt}). The probe
$\probe_\dtime$ represents the signal power received by the radar from the target's state components. 
The waveform parameters  determine  the radar's allocation  of resources (power)  $\response_\dtime$ to
these state components. A higher $\response(i)$ reduces  measurement noise covariance,
improving  measurement accuracy at the radar.
%
Thus, $\probe_\dtime^\p \response$ represents  the signal to noise ratio (SNR) and the budget constraint $\probe_\dtime^\p \response \leq 1$ imposes a bound on the SNR.
A cognitive  radar  maximizes a utility  $\utility(\response)$, which  increases with accuracy (inverse of noise power) $\response$.  However, limited resources restrict the radar to allocate just enough power to ensure the precision (inverse  covariance) of all components  does not exceed a  prespecified level   $\bound ^{-1}$ at each epoch $\dtime$.
We can justify the linear budget constraint as follows:

\begin{lemma} \label{lem:linear}The linear budget constraint
  $\probe_\dtime^\p \response \leq 1 $  is equivalent to the solution of  ARE (\ref{eq:arerp}) satisfying ${\kalmancov^*_\dtime}^{-1} (\probe_\dtime,\response) \preceq \bound^{-1}$
  for some symmetric positive definite matrix $\bound^{-1}$.
\end{lemma}

The proof of Lemma \ref{lem:linear} follows straightforwardly using the information Kalman filter formulation \cite{AM79}, and showing that ${\kalmancov^*}^{-1}$ is increasing  in $\response$. Afriat's theorem requires that the constraint
$\probe_\dtime^\p \response \leq 1 $  is active at $\response = \response_\dtime$. This holds in our case
since ${\kalmancov^*}^{-1}$ is increasing  in~$\response$.

\begin{rem}
\item \cite{KAEM20}   discusses Theorem~\ref{thm:nonlinear_afriat}  with a  nonlinear budget constraint for 
 identifying cognitive radars. The nonlinear  budget   emerges naturally from the covariance of the  tracker. 
\item As noted in Remark \vref{afriat:multiple},
  we can generalize the procedure to  multiple constraints.
  \item Finally, in a multitarget setting  comprising $\probedim$ targets, 
$\probe_\dtime(i)$ is the incentive for the radar to track target $i$ and is  the maximum eigenvalue (or trace) of covariance matrix
$\snoisecov(i)$ which depends on the maneuver of target $i$. Also  $\response_\dtime(i)$ is the maximum eigenvalue (or trace) of $\onoisecov^{-1}(i)$ due to choosing a waveform for target $i$.
The linear budget constraint is justified as above.  
\end{rem}

\summar  We applied Afriat's theorem (Theorem \ref{thm:afriat}) with $\probe_\dtime$ as the spectrum of $\snoisecov$ and $\response_\dtime$ as the spectrum of $\onoisecov$,  to test for a cognitive radar, i.e.,  utility maximization (\ref{eq:radaropt}). Afriat's theorem  constructs a set of  utility functions (\ref{eqn:estutility}) that  rationalize the decisions of the radar.

\section{Identifying Optimal Beam Allocation in Cognitive Radar}
\index{revealed preferences! IRL for cognitive radar}
\index{cognitive radar! IRL for beam allocation}

Optimal beam allocation is an important functionality of a   cognitive radar. For example, an active electronically scanned array radar adaptively  switches its beam to track multiple targets.
This section constructs a test to identify if a radar is switching  its beam  between multiple  targets to optimize a utility function. At this  level of abstraction, we view each component $i$  of the probe signal
  $\probe_\dtime(i)$ as the trace of the precision matrix (inverse covariance) of target $i$. 

Suppose a radar adaptively switches its beam between
$\probedim$ targets where these $\probedim$ targets (e.g., drones) are purposefully controlled by us. On the fast time scale indexed by $\fasttime$, each target $i$ has linear Gaussian dynamics  and the adversary  radar obtains linear  measurements:
\beq \label{eq:lineargaussian2}
\begin{split}
\state^i_{\fasttime+1} &= \statem\, \state^i_\fasttime  + \snoise^i_\fasttime, \quad \state_0 \sim \belief_0 \\
\obs^i_\fasttime &= \obsm\, \state^i_\fasttime + \onoise^i_\fasttime, \quad i=1,2,\ldots,\probedim.
\end{split}
\eeq
Here $ \snoise^i_\fasttime\sim \normal(0,\snoisecov_\dtime(i))$,
 $\onoise_\fasttime^i \sim \normal(0,
 \onoisecov_\dtime(i))$.
  We assume that both
$\snoisecov_\dtime(i)$ and $\onoisecov_\dtime(i)$ are known to us and the adversary.
Here  $\dtime$ indexes the slow time scale and $\fasttime$ indexes the fast time scale.
 The adversary's radar tracks our $\probedim$ targets using  Kalman filter trackers.
 The fraction of time the radar allocates to each target $i$ in epoch $\dtime$ is $\response_\dtime(i)$. 
The price the radar pays for each target $i$ at the beginning of epoch $\dtime$ is the trace of the  predicted {\em precision} of target $i$. This is  the trace
of the inverse of the predicted  covariance   at  epoch $\dtime$ using the Kalman predictor
\beq  \probe_\dtime(i) = \trace\big(\kalmancov^{-1}_{\dtime|\dtime-1}(i)\big), \quad i =1 ,\ldots, \probedim .
\label{eq:probe_beam}
\eeq
The predicted covariance $\kalmancov_{\dtime|\dtime-1}(i)$ is a deterministic function of  the maneuver covariance $\snoisecov_\dtime(i)$ of target $i$.  So the probe  signal
  $\probe_\dtime(i)$ is chosen by us is a deterministic function of  the maneuver covariance $\snoisecov_\dtime(i)$ of target $i$.  We abstract the target's covariance matrix by its  trace,  denoted as $\probe_\dtime(i)$.  The observation noise covariance $\onoisecov^i_\dtime$ depends on the adversary's radar response $\response_\dtime(i)$, i.e.,  the fraction of time allocated to target $i$.
We assume that each target $i$ is equipped with a radar detector and can estimate\footnote{If we impose a probabilistic structure on the estimates, then the resulting problem of statistical detection of a utility maximizer (stochastic revealed preferences)  is discussed in \secn \ref{sec:irl_noise}.} the fraction of  time $\response_\dtime(i)$ that  the adversary radar devotes to it.

Given the time series $\{\probe_\dtime, \response_\dtime$, $\dtime = 1,\ldots,\horizon\}$,
our aim is to identify if the adversary's radar is cognitive. We assume that  
a cognitive radar optimizes its  beam allocation as follows:
\begin{equation}
    \response_\dtime =     \argmax_\response \utility(\response)   \quad
  \text{ s.t. }    \response^\p \probe_\dtime \leq \precconstraint,\label{eq:beam}
\end{equation}
where $\utility(\cdot)$ is the adversary radar's utility function (unknown to us)  and $\precconstraint \in \reals_+$ is a prespecified average precision of  all $\probedim$  targets.  

Let us discuss
the budget constraint in~\eqref{eq:beam}: For cheaper targets (lower precision $\probe_\dtime(i)$), the radar is  incentivized to  allocate more time $\response_\dtime(i)$. However, due to resource constraints,
the radar can  achieve at most an  average precision of $\precconstraint$ across all $m$ targets.
The setup aligns with  Afriat's Theorem~\ref{thm:afriat}, enabling 
a test of whether the radar satisfies utility maximization in beam scheduling
and  reconstruction of   the utility function set.
Since the reconstructed utility is ordinal, $\precconstraint$ can be chosen as 1 without loss of generality (and does not need to be known by us).

\remar Afriat's theorem (and generalizations) reconstruct the utility function. A widely used generative model  for a monotone utility   in resource allocation  is the Cobb--Douglas utility  \index{Cobb--Douglas utility}
  $ \utility(\response) = \prod_{i=1}^\probedim \response^{\gamma_i} (i)$, $ \gamma_i > 0$.
  \cite{KAEM20}  uses this utility for beam scheduling in a cognitive radar.

\section{Hiding Cognition. How Can a Smart Sensor Act Dumb?}
\label{sec:dumb}

\index{cognitive radar! hiding utility (metacognition)}  \index{utility masking}

So far, we have discussed revealed preferences for achieving IRL, specifically how to 
identify constrained utility maximization and reconstruct the utility of an agent.  In this section we  switch sides: We are now the cognitive radar,  and the IRL is the adversary. {\em By making small purposeful sacrifices in performance, how can the cognitive radar  hide  its utility  from the adversary's  IRL?}
Utility obfuscation  is  important for a radar because,
if  an adversary can estimate the radar's utility, the adversary can then predict the radar's sensing strategy and degrade the radar's performance using electronic countermeasures (e.g.,  jamming).

Given our utility $\utility$,  
we will design a perturbed response $\presponse_k,k=1,\ldots,\horizon$ that barely satisfies the feasibility equations in  Afriat's theorem, thereby making  it difficult for an adversary IRL to identify our utility maximization behavior. Figure~\ref{fig:spoof} illustrates  the method.

\begin{figure}[h] \centering
  \mbox{ \subfigure[Schematic procedure]{
    {\resizebox{5cm}{!}{ 
    \begin{tikzpicture}[node distance = 4cm,scale=0.9]
      \tikzset{
    block/.style={rectangle, draw, line width=0.5mm, black, text width=4.5em, text centered,
                 minimum height=1em},
               line/.style={draw, -latex}}
                   \tikzset{
    block3/.style={rectangle, draw, line width=0.5mm, black, text width=7.5em, text centered,
                 minimum height=1em},
               line/.style={draw, -latex}}
  \node [block] (BLOCK1) {Cognitive Sensor};
  \node[block3,right of = BLOCK1,,node distance=3cm](opt){Purposeful Sub-optimal Response };
    \node [block, below of=BLOCK1,right of = BLOCK1,node distance=2cm] (BLOCK2) {Adversary\\ IRL};
  
    \path [line] (BLOCK1) -- (opt);
    
    \path [line] (opt.east) --++ (0.5cm,0cm)  |-    node[pos=0.3,left=-0.4cm]{\parbox{2cm}{\centering perturbed \\ response}} node[pos=0.3,right]{$\presponse$}    (BLOCK2);
    
    \path [line] (BLOCK2.west) ->++ (0cm,0cm)   -|    node[pos=0.8,left]{probe} 
  node[pos=0.8,right]{$\alpha$}  (BLOCK1.south);
   
\end{tikzpicture} } }} \hspace{1cm}
\subfigure[Feasibility margin $\margin_{\utility}$ before masking]{ 
{\resizebox{2.8cm}{!}{ \begin{tikzpicture}[scale=0.9]
    \draw[-Latex] (0,0) -- (2.5,0) node[right] {$\uaf$};
    \draw[-Latex] (0,0) -- (0,2.5) node[above] {$\lambda$};
    
    \coordinate (A) at (1.5,2.5);
    \coordinate (B) at (3,2);
    \coordinate (C) at (2.5,0.5);
    \coordinate (D) at (.5,0.5);
     \coordinate (E) at (0.2,1);
    \coordinate (Q) at (2.3,2.25);
    \coordinate (P) at (1.8,1.4);
    
    \draw[line width=1pt]  (A) -- (B) -- (C) -- (D) -- (E) -- cycle;
    
    \filldraw[] (P) circle (1pt) node[below] {$(\uaf^*,\lambda^*)$};
    \draw[-Latex, >=stealth]  (P) -- (Q) node[pos=0.65, left] {$\margin_\utility$};
    \draw[-Latex, >=stealth] (Q) -- (P);
  \end{tikzpicture}}}} \hspace{1cm}
\subfigure[Feasibility margin $\margin_\pdataset$ after masking]{
  {\resizebox{2.8cm}{!}{ \begin{tikzpicture}[scale=0.9]
    \draw[-Latex] (0,0) -- (2.5,0) node[right] {$\uaf$};
    \draw[-Latex] (0,0) -- (0,2.5) node[above] {$\lambda$};
    
    \coordinate (A) at (1.5,2.5);
    \coordinate (B) at (2.5,2);
    \coordinate (C) at (2.2,0.5);
    \coordinate (D) at (.5,0.5);
     \coordinate (E) at (0.2,1);
    \coordinate (Q) at (2.4,1.3);
    \coordinate (P) at (1.8,1.4);
    
    \draw[line width=1pt]  (A) -- (B) -- (C) -- (D) -- cycle;
    
    \filldraw[] (P) circle (1pt) node[xshift=3mm,below left] {$(\tilde{\uaf},\tilde{\lambda})$};
    \draw[-Latex, >=stealth]  (P) -- (Q) node[pos=0.35, above] {\footnotesize{$\margin_\pdataset$}};
    \draw[-Latex, >=stealth] (Q) -- (P);
  \end{tikzpicture}}}}}
\vspace{-0.3cm}
\caption{Hiding utility  from  adversary. 
  The convex polytopes are specified by Afriat's inequalities. We design the perturbed dataset
  $\pdataset$
so that $\margin_\pdataset < \margin_{\utility}$.}  
\label{fig:spoof}
\end{figure}

\subsubsection{Feasibility Margin for Afriat's Test}

To discuss utility masking, we first define the feasibility margin for Afriat's test.
Recall that, given probe signals $\probe_{1:\horizon}$,
a  utility  maximizer  generates its response as
\begin{equation} \label{eq:max_dumb}
  \response_\dtime\in \argmax_{\probe_\dtime^\p \response \leq 1}\utility(\response) , \quad  k = 1,\ldots,\horizon.
\end{equation}
This yields the dataset $\dataset_\utility=\{\probe_{k},\response_{k},k=1,2,\ldots,\horizon\}$ corresponding to utility $\utility$. 

Given the dataset $\dataset_\utility$, the inverse learner  (analyst) uses  Afriat's inequalities~\eqrefp{eqn:AfriatFeasibilityTest}  to identify utility maximization. Afriat's inequalities  comprise  $\horizon^2$ linear constraints in the $2\horizon$ variables $\uaf_{1:\horizon},\lambda_{1:\horizon}$ and can be written abstractly as  
$$ \afriat(\uaf_{1:\horizon},\lambda_{1:\horizon}; \dataset_\utility) \leq \zero. $$
We see that the true utility $\utility$ satisfies
the feasibility margin
\begin{equation}
  \label{eq:true_margin}
 \margin_{\utility} = \min_{\epsilon \geq 0}  \epsilon, \quad  \subjectto
 \afriat(\uaf^*_{1:\horizon},\lambda^*_{1:\horizon}; \dataset_\utility) + \epsilon \ones \geq 0,
\end{equation}
where $\uaf^*_{1:\horizon},\lambda^*_{1:\horizon}$ were defined in~\eqrefp{eq:utility_project}.
The margin $\margin_{\utility}$ specifies how far the true  utility $\utility$ is from failing Afriat's test; see Figure~\ref{fig:spoof}.
It is the minimal nonnegative perturbation $\epsilon$ so that Afriat's inequalities are no longer feasible. In the
context of a radar, a utility  with a large feasible margin is a
high-confidence point estimate of the radar’s strategy and,
hence, at higher risk of getting exposed by the adversary.

\subsubsection{Hiding our Utility from Adversary's IRL}
To hide our utility $\utility$ from the adversary, we now  design the perturbed dataset $\pdataset=\{\probe_k,\presponse_K,k=1,2,\ldots,\horizon\}$ that  has a smaller feasibility margin $\margin_\pdataset$ than $\margin_{\utility}$ of the original dataset $\dataset_\utility$. Let  $\masking \in [0,1]$ denote the utility masking coefficient chosen by us. We will design $\pdataset$ so that
\begin{equation}
 \label{eq:design_mask}
  \begin{split}
\margin_\pdataset &\leq (1 - \masking) \,\margin_{\utility}  \\
\text{ where } & \margin_\pdataset = \min_{\epsilon \geq 0}  \epsilon, \quad  \subjectto
\afriat(\tilde{\uaf}_{1:\horizon},\tilde{\lambda}_{1:\horizon}; \pdataset) + \epsilon \ones \geq 0.
  \end{split}
\end{equation}
Here $\tilde{\uaf},\tilde{\lambda}$ are computed using~\eqref{eq:putility_project}
using utility $\utility$ and dataset $\pdataset$, since  we know our utility~$U$.
 If we choose $\masking = 0$, then no masking is achieved since  the perturbed dataset  
 has the same feasibility margin  as the original dataset.
 If  $\masking = 1$, then we maximally mask  utility $\utility$ since the perturbed dataset has zero margin, and so lies on the edge of the feasible polytope; see Figure~\ref{fig:spoof}.

 But we also want the perturbed dataset $\pdataset$ to be as close as possible to our original dataset~$\dataset$; otherwise we sacrifice performance of our sensor w.r.t.\ utility $\utility$. So our design objective to achieve utility masking  is: Compute perturbed response 
 \begin{equation}
   \label{eq:optimal_perturb}
   \begin{split}
     & \presponse_{1:\horizon} \in
       \argmin_{\zeta_\dtime \in \reals_+^\probedim} \sum_{\dtime=1}^\horizon
       \utility(\response_\dtime) - \utility(\zeta_\dtime), \\
          \subjectto & \quad \probe_\dtime^\p \zeta_\dtime \leq 1, \quad  k = 1,\ldots,\horizon, \\ & \quad \text{ and masking  constraint } \eqref{eq:design_mask}.
   \end{split}
 \end{equation}
Algorithm~\ref{alg:irp} summarizes the procedure for hiding the utility from the adversary IRL.

\begin{algorithm} \caption{Masking Radar's Utility $\utility$  from Afriat's  Test (Adversary's IRL)} \label{alg:irp}
  Step 1. Compute radar's optimal response sequence $\response_{1:\horizon}$ as in~\eqref{eq:max_dumb} and feasibility margin $\margin_{\utility} $ in~\eqref{eq:true_margin}.
\\
Step 2. Choose utility masking coefficient $\masking\in[0,1]$. \\
Step 3. Compute cognition-making responses by solving~\eqref{eq:optimal_perturb}.
\end{algorithm}

\begin{rem}

 \item  The optimization problem~\eqref{eq:optimal_perturb} captures   \textit{metacognition} in radars, namely, switching between two modes of cognition: purposefully degrading performance to hide  the utility (privacy) versus achieving optimal sensing performance.
Since the  constraint~\eqref{eq:design_mask} is nonconvex, at best, one can solve~\eqref{eq:optimal_perturb} for a local stationary point. 
 
\item  Low probability of intercept radars hide their presence by reducing their transmitted power. The above approach operates at a higher level of abstraction, focusing on how 
 a cognitive radar can hide its sensing strategy (utility) from an adversary.
This  approach is inspired by privacy-preserving mechanisms in adversarial machine learning \cite{FHM20,PKY19}.
  
\item \cite{PKB23} provides several numerical examples to illustrate the performance of  Algorithm~\ref{alg:irp} in  hiding the radar's utility through waveform adaptation. It shows that  by making small sacrifices in performance, the utility can be masked significantly.

\item The feasibility margin used above is a goodness of fit measure for Afriat's test. Goodness of fit measures for revealed preferences are discussed in \secn \ref{sec:goodness}. 
\end{rem}

\section{Identifying Pareto-Optimal Multiagent Coordination}

\index{revealed preferences! Pareto optimal coordination}
Thus far we have discussed revealed preferences  as a nonparametric IRL procedure for identifying utility maximization of a single agent. This section and the next section describe revealed preference algorithms for \textit{multiagent} systems. By observing the decisions of a multiagent system, how to determine if the individual agents are coordinating their decisions?  If they are coordinating, can we subsequently reconstruct the individual utility functions which induce each agent's output?   For example,
a multifunction radar  coordinates (jointly optimizes) multiple sensing modes such as waveform and beam allocation. Other
examples include identifying if multiple radars or multiple drones are coordinating their behavior.

We will study two formulations to identify coordination behavior amongst multiple agents.
This section discusses identifying Pareto-optimal behavior. The next section discusses  identifying play from the Nash equilibrium of a concave potential game. 
The framework is similar to multiagent inverse reinforcement learning (MAIRL)  \cite{YSE19}  in that we reconstruct each agent's utility function. The important difference is that we first identify coordination among the agents; whereas MAIRL   assumes coordination among the agents. Thus, this problem can be considered at the intersection of inverse game theory \cite{KS15} and MAIRL.

\subsubsection{Pareto-Optimal  Coordination} 
Consider $\np$ agents. Each agent can consume some quantity of $\probedim$ goods. At time $\dtime $, each agent $\ia \in \{1,\dots, \np\}$ has consumption given by the vector $\response_\dtime^\ia \in \reals_+^\probedim$. The aggregate  consumption of the $\np$ agents is  $\response_\dtime = \sum_{\ia=1}^\np \response_\dtime^\ia \in \reals_+^\probedim$, subject to price vector (probe) $\probe_{\dtime} \in \reals_{++}^\probedim$. It is assumed that the preferences of each agent $\ia$ can be represented by a strictly monotone utility function $\utility^\ia(\response),\ \response \in \reals_+^\probedim$.

To define Pareto-optimal coordination, we first define a  welfare function.
Suppose each agent $\ia$ is assigned a weight $\mu_\dtime^\ia \geq 0$ at time $\dtime$. Then
 the welfare function over all $\np$ agents at time $\dtime$ is  
\begin{equation}
  \welfare_{\mu_\dtime}(\response^1,\ldots,\response^\np) = \sum_{\ia=1}^\np \mu_\dtime^\ia\, \utility^\ia(\response^\ia).
\end{equation}

If the $\np$ agents choose their individual  consumption  in accordance with the following definition, then we say that the group of agents \textit{coordinates} its behavior in a Pareto-optimal \index{Pareto optimality}  sense.  

\begin{definition}[Pareto-Optimal Coordination]
\label{coll_rat}
$\np$ agents with utility functions $\{U^\ia: \reals^\probedim \to \reals\}_{\ia=1}^\np$,  and subject to price vector $\probe_\dtime$, achieve
Pareto optimality (group coordination)  if  their consumptions $\{\response_\dtime^\ia\}_{\ia=1}^\np$ satisfy
\begin{equation}
  \label{eq:coll_rat}
  \begin{split}
  & \welfare_{\mu_\dtime}(\response_k^1,\ldots,\response_k^\np) \geq \welfare_{\mu_\dtime}(\response^1,\ldots,\response^\np), \\
     \text{ subject to } &  \response^\ia \in \reals^\probedim_+ , \quad \ia=1,\ldots\np \\  &   \probe_\dtime^\p\sum_{\ia=1}^\np\response^\ia \leq \probe_\dtime^\p\sum_{\ia=1}^\np\response_\dtime^\ia.
  \end{split}
\end{equation}
\end{definition}


The consumption  $\{\response_\dtime^\ia\}_{\ia=1}^\np$ that satisfies the above definition  is Pareto-optimal\footnote{In multiobjective optimization,
consumption $\{\response_\dtime^\ia\}_{\ia=1}^\np$ is Pareto optimal (efficient) 
if there does not exist another feasible  consumption  $\{\eta_\dtime^\ia\}_{\ia=1}^\np$
such that $\utility^\ia(\eta^\ia) \geq \utility^\ia(\response^\ia)$ for all $\ia \in \{1,\ldots,\np\}$ with strict inequality for at least one $\ia$. That is, it is not possible to improve the utility of some agent  $\ia$ without simultaneously decreasing the utility of another agent $q$.
%
%
Note~\eqref{eq:coll_rat} is a consequence of Pareto optimality.}: All the weights are positive, and so any consumption, say $\{\eta_\dtime^\ia\}_{\ia=1}^\np$, with increased welfare in the feasible set would contradict~\eqref{eq:coll_rat}.
The Pareto weights $\{\mu_\dtime^\ia\}_{\ia=1}^\np$ correspond to the bargaining power, or relative importance, of each individual in the formation of the Pareto-optimal solution. 

\subsubsection{IRL for Identifying Pareto-Optimal Coordination} 
We now switch perspectives to an  analyst who observes the multiagent consumption and aims to detect coordination, in the sense of Definition~\ref{coll_rat}. 
At each time $\dtime \in \{1,\dots,\horizon\}$, suppose the analyst observes probe signals $\probe_\dtime \in \reals^\probedim_+$, aggregate consumption $\response_\dtime 
= \sum_{\ia=1}^\np\response_\dtime^\ia \in \reals^\probedim$, and  assignable quantities $\uresponse_\dtime^\ia \leq \response_\dtime^\ia \; \forall \ia \in \{1,\dots,\np\}$. Notice that the assignable quantities  $\uresponse_\dtime^\ia$ observed by the analyst are lower bounds to the unobserved individual consumptions 
$\response_\dtime^\ia $. 
%
To summarize, the dataset observed by the analyst is
 \begin{equation}
 \label{eq:multiagent_dataset}
    \dataset = \{\probe_\dtime, \response_\dtime, \{\uresponse_\dtime^\ia\}_{\ia=1}^\np,\; \dtime = 1,\dots,\horizon \}.
 \end{equation}
 We emphasize that the true individual consumption vectors $\response_\dtime^\ia$ of the agents are not known to  the analyst.
Our aim is  to provide necessary and sufficient conditions for  the  dataset $\dataset$ to be consistent with Pareto-optimal coordination, in the sense of Definition~\ref{coll_rat}.

 Since individual responses $\response_\dtime^\ia$ are not observed by the analyst, we define the  variables
 $\feasconsi \in \reals_+^\probedim, \ \ia=1,\dots,\np$ called \textit{feasible personalized quantities}. They  satisfy
 \begin{equation}
   \label{eq:feasconsi}
   \feasconsi \geq \uresponse_\dtime^\ia, \quad \ia \in \{1,\ldots,\np\}, \qquad \sum_{\ia=1}^\np \feasconsi = \response_\dtime,
 \end{equation}
and  will play the role of the unobserved individual  responses $\response^\ia_\dtime$ in the main result below.

The following main result  states the equivalence between a consistency of dataset $\dataset$ with coordination and existence of a nonempty feasible region of a set of inequalities.

 \begin{theorem}[Collective Rational Behavior \cite{CDV11}] 
 \label{thm:cherchye1} Given dataset $\dataset$ in~\eqref{eq:multiagent_dataset} comprising  $\np$ agents, the following statements are equivalent:
    \begin{compactenum}
    \item There exist  $\np$ concave and continuous utility functions $U^1,\dots,U^\np$ such
that dataset $\dataset$ is consistent with Pareto-optimal coordination, in the sense of Definition~\ref{coll_rat}.
    \item For each agent $\ia=1,\ldots,\np$, there exist scalars $\{\uaf^\ia_\dtime,\lambda^\ia_\dtime,
\dtime=1,\ldots,\horizon\}$ where $\lambda^\ia_\dtime>0$ such that the following inequalities have a
feasible solution:
    \begin{equation}
    \label{eq:af_ineq} \uaf_\dtimen^\ia - \uaf_\dtime^\ia - \lambda_\dtime^\ia \probe_\dtime^\p
(\hresponse_\dtimen^\ia - \hresponse_\dtime^\ia) \leq 0 \quad \forall
\tindx,\dtimen\in\{1,2,\dots,\Tindxter\}.
    \end{equation}
  \item For each agent $\ia=1,\ldots,\np$, a family of continuous, concave, monotone, piecewise linear  utility functions that rationalizes $\dataset$ is
     \begin{equation}
            \label{eq:util} \hat{U}^\ia(\response) = \min_{\dtime \in \{1,\dots,\horizon\}}
\{\uaf_\dtime^\ia + \lambda_\dtime^\ia \probe_\dtime^\p(\response - \hresponse_\dtime^\ia)\}
              \end{equation} where $\{\uaf^\ia_\dtime,\lambda^\ia_\dtime, \dtime=1,\ldots,\horizon\}$
satisfy the inequalities~\eqref{eq:af_ineq}.
\item For each agent $\ia \in \{1,\ldots,\np\}$, dataset $\mathcal{D}$ satisfies the Generalized Axiom of Revealed Preference (GARP): For every ordered subset $   \{i_1,i_2,\ldots,i_\lind\} \subset \{1,2,\ldots,N\}$, if
$$ \probe_{i_1}^\p\, \hresponse^\ia_{i_2} \leq \probe_{i_1}^\p \,\hresponse^\ia_{i_1}, \;\;
    \probe_{i_2}^\p\, \hresponse^\ia_{i_3} \leq \probe_{i_2}^\p \,\hresponse^\ia_{i_2}, \ldots,
    \probe_{i_{\lind-1}}^\p \, \hresponse^\ia_{i_\lind} \leq \probe_{i_\lind}^\p\, \hresponse^\ia_{i_\lind},
$$  
then it holds that
  $\probe_{i_\lind}^\p \hresponse^\ia_{i_1} \geq \probe_{i_{\lind-1}}^\p \hresponse^\ia_{i_{\lind-1}}$.
\end{compactenum}
\end{theorem}

The proof is  similar to Afriat's theorem.
But the  feasibility test \eqref{eq:af_ineq} cannot be implemented directly since 
$\{\feasconsi\}_{\ia=1}^\np$ are not known to the analyst. The following theorem
yields set-valued estimates of  $\{\feasconsi\}_{\ia=1}^\np$ that are  used in
Theorem~\ref{thm:cherchye1} to reconstruct the utilities.

 \begin{theorem} 
   \label{MILP}
   The existence of 
    $\np$  concave and continuous utility functions $U^1,\dots,U^\np$ such that dataset $\dataset$ in~\eqref{eq:multiagent_dataset} is consistent with coordination
is equivalent to the existence of
 $\feasconsi \in \reals_+^\probedim$, $\eta_t^\ia \in \reals_+$, and $x_{st}^\ia \in \{0,1\}, \ia=1,\dots,\np$ that satisfy the following mixed-integer linear program (MILP) constraints:
 \begin{enumerate}[leftmargin=1cm,label=(\roman*)]
    \item $\sum_{\ia=1}^\np\feasconsi = \response_\dtime$ and $\feasconsi \geq \uresponse_\dtime^\ia$
    \item $\eta_\dtime^\ia = \probe_\dtime^\p\feasconsi$
    \item $\eta_s^\ia - \probe_s^\p\feasconsi < y_s x_{s\dtime}^\ia$
    \item $x_{su}^\ia + x_{u \dtime}^\ia \leq 1 + x_{s\dtime}^\ia$
    \item $\eta_\dtime^\ia - \probe_\dtime^\p\feasconsi \leq y_\dtime(1-x_{s\dtime}^\ia)$
 \end{enumerate}
 where $y_\dtime = \probe_\dtime'\response_\dtime$ is the group consumption cost at time $\dtime$.
 \end{theorem}

 Constraint (i) is simply~\eqref{eq:feasconsi}.
 Constraint (ii) serves as the definition of $\eta_\dtime^\ia$. Constraints (iii), (iv) and (v) are equivalent to GARP as we now explain. Set $x^\ia_{ij}=1$ if $a^\ia_{ij} = \probe_i^\p(\hresponse^\ia_j-\hresponse^\ia_i) \leq 0$; and $x_{ij}=0$ otherwise. We defined matrix $A$ in~\eqrefp{eq:A_rp}, but we use $\hresponse$ instead of~$\response$.

 \begin{proof}
Necessity:  To show  GARP implies that conditions (iii) to (v)  are feasible, suppose $x_{i_1i_2}=x_{i_2 i_3}=\cdots=x_{i_{\lind-1} i_\lind}=1$. Then   constraint (iii) holds and constraint (iv)  implies $x_{i_\lind i_1}=1$.  Constraint (v) is then automatically satisfied.

Sufficiency:  To show that if the constraints are feasible then  GARP holds,
suppose GARP is violated. Then if (iii) and (iv) hold, constraint (v) is violated.
\end{proof}

\subsubsection{Example. Testing for Coordination in Radar Network}
\index{cognitive radar! IRL for Pareto-optimal coordination}

We now discuss how Theorem~\ref{thm:cherchye1}  can be used to  test for coordination in a radar network. Although not discussed here, a similar framework can be used to test for 
coordination in a network of drones where the state process for the $\ia$th drone is $\state_k^\ia$, $\ia \in \{1,\ldots,\np\}$.

 Suppose ``we" are a target being tracked by a network of $\np$ adversary radars, and we can intercept some of the emissions of these  radars. Given these observed  signals, how can we determine if the radars are {coordinating} (in the sense of Pareto-optimal power allocation)?  If they are coordinating, how can we reconstruct the  utility functions of the individual radars?

 The  target's kinematics and $\ia$th 
radar's  measurements on  the fast time scale $\fasttime$ are
\beq \label{eq:lineargaussian_pareto}
\begin{split}
\state_{\fasttime+1} &= \statem\, \state_\fasttime  + \snoise_\fasttime(\probe_\dtime), \quad \state_0 \sim \belief_0 \\
\obs_\fasttime^\ia &= \obsm\, \state_\fasttime + \onoise^\ia_\fasttime(\response^\ia_\dtime) , \quad \ia \in \{1,\ldots,\np\}.
\end{split}
\eeq
Here  $\state_\fasttime \in \statespace = \reals^\statedim$ is ``our'' state with
initial density $\belief_0 \sim \normal(\hat{\state}_0,\kalmancov_0)$,
 $\obs_\fasttime^\ia \in \obspace = \reals^\obsdim$ denotes the observations of radar $\ia$,
 $\snoise_\fasttime\sim \normal(0,\snoisecov(\probe_\dtime))$,
 $\onoise^\ia_\fasttime \sim \normal(0,
\onoisecov(\response^\ia_\dtime))$
and 
  $\{\snoise_\fasttime\}$,  
  $\{\onoise^\ia_\fasttime\}$ are mutually independent  i.i.d.\ processes. Here, $k$ indexes the slow time scale.

  To reduce  the measurement noise covariance $\onoisecov$, the adversary radar network must   increase its power output. We model this as a total power constraint, which limits the power  allocation across the $\np$ radars. The target probes the radar network through purposeful maneuvers that  modulate  the state noise covariance matrix $\snoisecov$ by $\probe_{\dtime}$.

Then, Theorem~\ref{thm:cherchye1} can be used to determine if the adversary  radar network's  response satisfies Pareto-optimal  coordination in  Definition~\vref{coll_rat}.  Specifically, suppose:
\begin{compactenum}
\item Our probe $\probe_{\dtime}$, characterizing our maneuvers, is the vector of eigenvalues of  $Q$.
\item The response $\response_{\dtime}^\ia$ of the $\ia$th adversary radar is the vector of eigenvalues of  $R^{-1}$.
  

\item We (the target) can observe the aggregate output signal power across all $\np$ radars $\response_{\dtime} = \sum_{\ia=1}^\np \response_{\dtime}^\ia$ through a broadbeam omni-directional receiver configuration.
    \item For each radar we observe an associated signal power $\hat{\response}_{\dtime}^\ia$ which is upper bounded by the true signal power $\response_{\dtime}^\ia$, i.e., $\hat{\response}_{\dtime}^\ia \leq \response_{\dtime}^\ia$ . 
    \end{compactenum}
    Thus we obtain the dataset 
$    \dataset = \{\probe_{\dtime}, \response_{\dtime} , \{\hat{\response}_{\dtime}^\ia\}_{\ia=1}^\np, \dtime=1,\dots,T\}$.
 We can then use the MILP in Theorem~\ref{MILP} to test for coordination. If this MILP has a  feasible solution,  we can apply Theorem~\ref{thm:cherchye1}  to reconstruct feasible utility functions for each of the $\np$ cognitive radars.

\section{Revealed Preferences for Concave Potential Game}

\index{revealed preferences! concave potential game}  \index{radar network! testing for coordination}
We now consider a multiagent version of Afriat's theorem to decide if a dataset is consistent with  play from the Nash equilibrium of a concave potential game.
\cite{Ney97} showed that the Nash equilibrium of a  smooth  concave potential game is a unique pure  (nonrandomized) strategy which coincides with the unique maximizer of a strictly concave potential function. So  testing if a dataset is consistent with play from the Nash equilibrium of a strictly concave potential game can be solved using a similar procedure to Afriat's theorem
(so it is not necessary to know game theory to read this section). 

Potential games 
were introduced by Monderer and Shapley~\cite{MS96}. 
They have been used to analyze  scheduling and demand side management schemes in the energy market~\cite{ING10}. Another example of  a potential game is a {\it congestion game}~\cite{BKP07} in which the utility of each player depends on the amount of resources it and other players use.
From an IRL perspective, potential games are 
sufficiently specialized to be refutable, i.e., there exist  datasets that fail Afriat's test.
In contrast, for general games, revealed preference conditions may always be satisfied,  
thereby 
losing specificity in identifying play from an equilibrium.

\subsubsection{Potential Game}
Consider a $\np$-player game where each player $\ia \in \{1,2,\dots, \np\}$ has a  set of actions $\setresponse^\ia \subseteq \reals_+^\probedim$ with action  denoted as $\response^\ia\in\setresponse^\ia$. The utility function of each player $\ia$  is $\utility^\ia \colon \setresponse \rightarrow \reals$ where $\setresponse=\prod_{\ia=1}^\np \setresponse^\ia=\{\response=(\response^1,\response^2,\dots,\response^\np)\in
\reals_+^{\np \times \probedim}\}$.
In an {\it ordinal concave potential game}, there exists a concave  function $\potfun(\response^1,\response^2,\dots,\response^\np)$ such that for any two players $\ia$ and $\ja$, 
\begin{equation}
\utility^\ia(\response^\ia,\response^{-\ia})-\utility^\ia(\response^\ja,\response^{-\ia}) > 0 
\text{ iff } \potfun(\response^\ia,\response^{-\ia})-\potfun(\response^\ja,\response^{-\ia}) > 0  \qquad \forall \response^\ia, \response^\ja\in \setresponse^\ia.
\label{eq:cordpotfun}
\end{equation}
Here  $\response^\ia$ denotes the response of player $\ia$, and $\response^{-\ia}$ the response of the remaining players (this notation is standard in game theory). In words, the incentive of all players to change their strategy is determined by a single potential function $\potfun$.

In an exact potential game (also called potential game), the utilities satisfy
\begin{equation}
  \label{eq:exact_pot}
\utility^\ia(\response^\ia,\response^{-\ia})-\utility^\ia(\response^\ja,\response^{-\ia})  = 
\potfun(\response^\ia,\response^{-\ia})-\potfun(\response^\ja,\response^{-\ia})  \qquad  \forall \response^\ia, \response^\ja\in \setresponse^\ia.
\end{equation}

One can also characterize a potential game directly in terms of the utilities. Suppose the utility $\utility^\ia $ of each player $\ia$ is twice continuously differentiable. Then we have a potential game iff the mixed derivatives satisfy
$$ \frac{\partial^2 \utility^\ia}{\partial \response^\ia_{i} \partial \response^\ja_j}=
\frac{\partial^2 \utility^\ja}{\partial \response^\ja_{j} \partial \response^\ia_i}, \quad
\ia, \ja \in \{1,\ldots,\np\}, \quad i,j \in \{1,\ldots,\probedim\}.
$$

Finally, a potential game is \textit{concave} if the potential $\potfun$ in~\eqref{eq:exact_pot} is concave in each component.  In a smooth
potential game with strictly concave potential, a strategy is a pure Nash equilibrium iff it  maximizes  the potential $\potfun$; moreover, strict concavity guarantees a unique maximizer.  

\subsubsection{Revealed Preferences Test}

Consider a time series of data from $\np$ agents  $\dataset=\{\probe_\tindx,\response_\tindx^1,\dots,\response_\tindx^\np, \tindx = 1,2,\dots,\Tindxter\}$ where  $\probe_\tindx\in\reals_{++}^\probedim$ denotes the external influence (probe), and $\response_\tindx^\ia\in \reals_{+}^\probedim$ denotes the response (action) of agent $\ia$ at time $\tindx$. Is it possible to identify if the dataset originated from agents that play from the Nash equilibrium of 
a potential game?

\begin{definition}[Nash rationality]   \label{eq:NashEquilibrium}  
Consider  dataset
$\mathcal{D}=\{\probe_\tindx,\response_\tindx^1,\response_\tindx^2,\dots,\response_\tindx^\np, \tindx = 1,2,\dots,\Tindxter\}$.
Then $\dataset$ 
is consistent with play from the {\it Nash equilibrium}  if there exist utility functions $\utility^\ia(\response^\ia,\response^{-\ia})$, $\ia\in\{1,2,\dots,\np\}$, such that 
\begin{equation}
\response_\tindx^\ia(\probe_\tindx,\response^{-\ia})\in\argmax_{\probe_\tindx^\p \response^\ia \leq 1}\utility^\ia(\response^\ia,\response^{-\ia}), \quad k=1,\ldots, N.
\label{eqn:NashEquation}
\end{equation}
Here $\utility^\ia(\response,\response^{-\ia})$ is a strictly monotone utility function in individual consumption $\response^\ia$.

\end{definition}
 Just as in utility maximization,  the budget constraint $\probe_\tindx^\p \response^\ia \leq 1$  in~\eqref{eqn:NashEquation} is without loss of generality. It can be replaced with
$\probe_\tindx^\p \response^\ia \leq 
\budget_\tindx^\ia$ for any positive real number $\budget_\tindx^\ia$.

The following theorem provides necessary and sufficient conditions to test if dataset $\dataset$ is  consistent with Nash rationality from a \textit{concave} potential game.


\begin{theorem}[Multiagent Afriat's Theorem \cite{Deb08,Deb09}] \mbox{} \newline  Given  dataset
$\mathcal{D}=\{\probe_\tindx,\response_\tindx^1,\response_\tindx^2,\dots,\response_\tindx^\np, \tindx = 1,2,\dots,\Tindxter\}$, where $\probe_\dtime \in \reals^\probedim_{++}$ and response $\response^\ia_k \in \reals_+^\probedim$, the following statements are equivalent:
\begin{compactenum}
\item $\dataset$ is Nash rationalized (Definition~\ref{eq:NashEquilibrium}) by utility functions that admit a strictly monotone concave ordinal  potential function~\eqref{eq:cordpotfun}.
\item There exist scalars $\{\uaf_\dtime,\lambda^\ia_\dtime, \ia=1,\ldots,\np, \dtime=1,\ldots,\horizon\}$, where $\lambda_\dtime^\ia > 0$ such that the following   inequalities have a feasible solution:
\begin{equation}
\uaf_{\dtimen}-\uaf_\dtime-\sum\limits_{\ia=1}^\np\lambda_\dtime^\ia\probe_\dtime^\p(\response_{\dtimen}^\ia-\response_\dtime^\ia) \leq 0 \quad \forall \tindx,\dtimen\in\{1,2,\dots,\Tindxter\}.\label{eqn:NashRationFesTest}
\end{equation}
\item A family of continuous, concave,  monotone, piecewise linear  potential functions that Nash rationalize
  $\dataset$ are
\begin{equation}
  \hat{\potfun}(\response^1,\response^2,\dots,\response^\np) =
\min_{\dtime \in \{1,\ldots,\horizon\}}
\Big\{\uaf_\dtime+\sum_{\ia=1}^\np\lambda_\dtime^\ia\probe_\dtime^\p(\response^\ia-\response^\ia_\dtime)
\Big\}   \label{eq:est_potential}
\end{equation}
where  $\{\uaf_\dtime,\lambda^\ia_\dtime, \ia=1,\ldots,\np, \dtime=1,\ldots,\horizon\}$ satisfy~\eqref{eqn:NashRationFesTest}. 
\end{compactenum}
\label{thm:NashRationFeasibility}
\end{theorem}

\remar
If there is only one  agent (i.e., $\np=1$), then Theorem~\ref{thm:NashRationFeasibility} becomes Afriat's theorem. But the proof of  Theorem~\ref{thm:NashRationFeasibility} is simpler since  the potential is concave by assumption. 

The above theorem constitutes a nonparametric test for Nash rationality. It involves determining if (\ref{eqn:NashRationFesTest}) has a feasible solution. Computing parameters $\uaf_\dtime$ and $\lambda^\ia_\dtime>0$ in (\ref{eqn:NashRationFesTest}) requires  solving a linear program with $\horizon^2$ linear constraints in $(\np+1)\horizon$ variables. 

\subsubsection{Example. Identifying Malicious Agents in a Sensor Field} 

Consider $\np$ malicious agents, each with utility function  $\utility^\ia(\response^\ia,\response^{-\ia})$, $\ia=1,2,\ldots,\np$. Here $\response^\ia\in\reals_+^\probedim$ is the distance agent $\ia$ is from each of  $\probedim$ sensors, and $\response^{-\ia}=\{\response^\ja\}_{\ja\neq \ia}\in\reals_+^{\probedim\times (\np-1)}$ is the  distance each of the other $\np-1$ malicious agents are from the sensors. An example utility function for malicious agents is given by the sum of two terms.
The first term models the interdependence of the average distance of each agent to the sensors. The second term represents each agent's preference to avoid detection.  Nonmalicious agents have no position preference related to the sensors or each other and therefore select $\response^\ia$ in a uniform random fashion. The price (probe) signal $\probe_\dtime\in\reals_{++}^\probedim$, at observation $\dtime$, represents the cost of moving away from the sensors.  Given the dataset $\dataset$, Theorem~\ref{thm:NashRationFeasibility}  can be used to identify malicious cooperation. 

\section{IRL in Noise. Statistical Detection of Utility Maximization}
\label{sec:irl_noise}

\index{detector! utility maximization}
\index{revealed preferences! statistical detection|(}

This section constructs a statistical detector to achieve IRL by
testing for constrained utility maximization when either the probe or response are observed in noise.

 Afriat's theorem (Theorem~\ref{thm:afriat}) and its generalization to nonlinear budgets (Theorem~\ref{thm:nonlinear_afriat}) assume perfect observation of the probe and response signals.
However, when the response (e.g.,  adversary's radar waveform) is measured in noise by us, or  the probe signal (e.g., our maneuver) is measured in noise by the adversary,  violation of the inequalities in Afriat's theorem could be either due to measurement noise or absence of utility maximization. This section  provides a  statistical detection test for utility maximization and characterize the Type-I and Type-II errors of the detector. We give a  tight  Type-I error bound. This section also sets the stage for
\secn \ref{sec:adapt} where the probe signal will be optimized to minimize the Type-II error of the detector.

Suppose we observe  the response $\response_\dtime$ of the adversary's  radar in  additive noise $\anoise_\dtime$ as
\begin{equation}
	\nresponse_\dtime = \response_\dtime + \anoise_\dtime, \quad \dtime  =1,2,\ldots,\horizon.
	\label{eqn:noisemodel}
      \end{equation}
      Here $\anoise_\dtime$ are $\probedim$-dimensional  random variables that can be correlated but are independent of $\response_\dtime$. 

\remar We can view $\nresponse_\dtime$ as the agent's stated preference while $\response_\dtime$ is its underlying preference. 

\subsubsection{IRL Detector. Algorithm and Performance Analysis}
Given the noisy dataset 
\begin{equation}
	\obsdataset= \left\{\left(\probe_\dtime,\obsresponse_\dtime \right), \dtime = 1,\dots,\horizon \right\}
	\label{eqn:noisydataset}
\end{equation}
from the adversary radar, how can we detect if it is cognitive?
Let
\begin{compactitem} \item
  $H_0$ denote the null hypothesis that the 
  radar is cognitive (utility maximizer~\eqref{eq:utilitymaximization})
\item $H_1$ denote the alternative hypothesis, i.e., the radar is not cognitive.
\end{compactitem}
Given the noisy dataset $\obsdataset$,there are two possible sources of error:
\begin{align}
	\text{\bf Type-I errors:}   &\text{\hspace{1.6mm}Reject $H_0$ when $H_0$ is valid (false positive).} \nonumber\\
	\text{\bf Type-II errors:}  &\text{\hspace{1.6mm}Accept $H_0$ when $H_0$ is invalid (false negative).}
	\label{eqn:hypothesis}
\end{align}
Given the noisy dataset $\obsdataset$, we  propose the following statistical detector to determine if the adversary radar is a utility maximizer:
\begin{equation}
\int\limits_{\Tstat^*(\Obsresponse)}^{\infty}f_M(\psi)\, d\psi \overset{H_0}{\underset{H_1}{\gtrless}} \threshold. 
	\label{eqn:Statistical_Test}
\end{equation}
In the statistical test (\ref{eqn:Statistical_Test}): \\ (i) $\threshold$ is the significance level of the test and is chosen by the analyst.
It is the probability of rejecting $H_0$ given that $H_0$ is true. 
The notation $\overset{H_0}{\underset{H_1}{\gtrless}}$ means that the detector decides $H_0$ if the integral is larger than $\threshold$, and decides $H_1$ otherwise.
\\(ii) The test statistic $\Tstat^*(\Obsresponse)$, with ${\Obsresponse}=\left[\obsresponse_1,\obsresponse_2,\dots,\obsresponse_\horizon\right]$ is the solution of the following constrained optimization problem in the variables $\{\Tstat, \lambda_k,\uaf_k,k=1,\ldots, \horizon\}$:
\begin{equation}
\begin{array}{rl}
\Tstat^*(\Obsresponse) = \min & \Tstat \\
\mbox{s.t.} & \uaf_{\sindx}-\uaf_{\tindx}- \lambda_\tindx \probe_\tindx^\prime (\obsresponse_\sindx -\obsresponse_\tindx)-\lambda_\tindx \Tstat \leq 0 \quad  \\
& \lambda_\dtime > 0, \quad \Tstat \geq 0 \quad\text{for}\quad \tindx,\sindx\in \{1,2,\dots,\horizon\}.\label{eqn:AE}
\end{array}
\end{equation}
(iii) $f_M$ is the pdf of the random variable $M$ where
\begin{equation}
	M \ole \underset{\underset{\tindx \ne \sindx}{\tindx,\sindx}}{\max}\left[\probe_\tindx^\prime(\anoise_\tindx - \anoise_\sindx)\right]. 
	\label{eqn:def:M}
      \end{equation}

      The intuition behind   (\ref{eqn:Statistical_Test}), (\ref{eqn:AE}) is clear: if $\Tstat = 0$, then  (\ref{eqn:AE}) is equivalent to  Afriat's inequalities. Owing to the presence of noise,
      typically  $\Tstat=0$ is not feasible; so we seek the minimum perturbation $\Tstat^*(\Obsresponse)$ so that~\eqref{eqn:AE} holds.

The constrained optimization problem (\ref{eqn:AE}) is nonconvex due to the presence of the  bilinear term  $\lambda_\tindx \Tstat$ in the first constraint. However, since the
objective function depends only on the scalar $\Tstat$, a one-dimensional line search algorithm can be used. For any
fixed value of $\Tstat$, (\ref{eqn:AE})  becomes a set of linear inequalities, and
so feasibility is straightforwardly determined.

The implementation of  detector (\ref{eqn:Statistical_Test}), which achieves IRL given noisy responses from the adversary, is described in Algorithm~\ref{alg:detect}.
Step 1  evaluates the empirical cdf $\hat{F}_M$ by simulation and uses
 the probe signals $\probe_1,\ldots\probe_\horizon$.
Step 2 implements the detector on the noisy responses $\Obsresponse$ with significance level~$\threshold$.

\begin{algorithm}
  Input: Noisy dataset  $\obsdataset$, significance level $\gamma$, joint distribution of $\Anoise= [\anoise_1,\ldots,\anoise_\horizon]$.
\begin{compactenum}
\item  {\em Offline Step}.
For iterations $l=1,\ldots L$:
\begin{compactenum} \item 
  Simulate  noise sequence $\Anoise^{(l)}= [\anoise_1,\ldots,\anoise_\horizon]^{(l)}$.
\item Compute $M^{(l)}$  using (\ref{eqn:def:M}).
\end{compactenum}
 Compute the empirical distribution $\hat{F}_M(\cdot)$  of $M$ from these $L$ samples.
   \item Solve  (\ref{eqn:AE}) for test statistic $\Tstat^*$. Finally implement detector (\ref{eqn:Statistical_Test}) as
 \beq 1- \hat{F}_M\big(\Tstat^*(\Obsresponse)\big) \overset{H_0}{\underset{H_1}{\gtrless}} \threshold  .\label{eq:implement} \eeq
\end{compactenum}
\caption{Detecting Utility Maximization given Noisy Response} \label{alg:detect}
\end{algorithm}

The following is our  main result for characterizing the IRL detector  (\ref{eqn:Statistical_Test}). It states that the probability of Type-I error (false positive, i.e.,  incorrectly rejecting the true null hypothesis) of the detector is bounded by $\threshold$ and the optimal solution $ \Tstat^*(\Obsresponse)$ is a  tight false positive bound.

      \begin{theorem} \label{thm:type1} Consider the noisy dataset 	$\obsdataset= \left\{\left(\probe_\dtime,\obsresponse_\dtime \right) \colon \dtime = 1,\dots,\horizon \right\} $ where
$	\nresponse_\dtime = \response_\dtime + \anoise_\dtime$
  and the statistical detector (\ref{eqn:Statistical_Test}). 
     Let   ${\Obsresponse}=\left[\obsresponse_1,\obsresponse_2,\dots,\obsresponse_\horizon\right]$.
        \begin{compactenum} \item  Suppose (\ref{eqn:AE}) has a feasible solution. Then
          $H_0$ (noise free dataset $\dataset$ satisfies utility maximization)  is equivalent to the event   $ \{\Tstat^*(\Obsresponse) \leq M \}$.
          \item 
        The probability of Type-I error (false positive, also called false alarm) is
\beq P_{\Tstat^*({\Obsresponse})}(\text{declare } H_1|H_0) \leq \threshold . \label{eq:pfa} \eeq
\item   The optimizer $ \Tstat^*({\Obsresponse})$ in   (\ref{eqn:Statistical_Test}) yields a tight Type-I error bound, in that for any other  $\tPhi \in [\Tstat^*, M]$,
      \beq  P_{\tPhi({\Obsresponse})}(\text{declare } H_1|H_0) \geq  P_{\Tstat^*({\Obsresponse})}(\text{declare } H_1|H_0) . \label{eq:type1}
\eeq
      \end{compactenum}
    \end{theorem}

\noindent {\em Proof. } 1. Suppose $H_0$ holds. By Theorem \ref{thm:afriat}, $H_0$ is equivalent to  Afriat's equations (\ref{eqn:AfriatFeasibilityTest}) having a feasible solution.
Let $(\lambdat_\tindx, \ut_\tindx)$ denote a feasible solution to  (\ref{eqn:AfriatFeasibilityTest}).
Then substituting $\response_\dtime = \nresponse_\dtime - \anoise_\dtime$, 
we see  that $(\lambdat_\tindx, \ut_\tindx,\Tstat = M)$  is   a feasible solution  to the noisy inequalities (\ref{eqn:AE}), where $M$ is defined in~\eqref{eqn:def:M}. Since  $(\lambdat_\tindx, \ut_\tindx,\Tstat = M)$ is feasible, clearly 
  the minimizing solution  of   (\ref{eqn:AE}) satisfies $\Tstat^*(\Obsresponse) \leq M$. 
  Therefore,   $$\text{    (\ref{eqn:AE})  feasible and  $H_0$  } \implies  \Tstat^*(\Obsresponse) \leq M. $$
 Next, let $(\bar{\lambda}_\tindx,\bar{\uaf}_\tindx)$ denote a feasible solution to the noisy inequalities  (\ref{eqn:AE}). Then  $\Tstat^*(\Obsresponse) \leq M$ implies that  (\ref{eqn:AfriatFeasibilityTest}) has a feasible solution,  i.e.,
  $$ \text{    (\ref{eqn:AE}) feasible  and }   \Tstat^*(\Obsresponse) \leq M \implies H_0 . $$ 
  Therefore if  (\ref{eqn:AE}) is feasible,  $H_0$ is equivalent to  $ \Tstat^*(\Obsresponse) \leq M$.
\\ \noindent
2.  Let $\ccdf$ denote the complementary cdf of $M$.
From the statistical test  (\ref{eqn:Statistical_Test}),  declaring   $H_1$
is equivalent to  the event $\{\ccdf(\Tstat^*(\Obsresponse)) \leq \threshold\}$. So 
using  Part 1 we have
 $$P(H_1|H_0) = P\big(\ccdf(\Tstat^*(\Obsresponse)) \leq \threshold | \Tstat^*(\Obsresponse) \leq M \text{ and \eqref{eqn:AE} is feasible}\big) .$$
Now if $ \Tstat^*(\Obsresponse) =M$, then since $\ccdf(M)$ is uniform\footnote{For any random variable $X$, $\bar{\cdf}(X)$ is  uniformly distributed in $[0,1]$, where $\bar{\cdf}$ is the complementary cdf of $X$.} in $ [0,1]$ clearly
$P(H_1|H_0) = \threshold$. Since  $\ccdf$ is decreasing in its argument,
it follows that for $ \Tstat^*(\Obsresponse) \leq M$ we have
$$  P(H_1|H_0) =P(\ccdf(\Tstat^*(\Obsresponse)) \leq \threshold | \Tstat^*(\Obsresponse) \leq M) \leq \threshold. $$ 
3. Suppose $\tPhi(\Obsresponse) > \Tstat^*(\Obsresponse)$. Then  (\ref{eq:type1}) holds since
$$P(\ccdf(\tPhi(\Obsresponse)) \leq \threshold | \tPhi(\Obsresponse) \leq M) \geq P(\ccdf(\Tstat^*(\Obsresponse)) \leq \threshold | \Tstat^*(\Obsresponse) \leq M). $$

\remar
We have discussed detecting the presence of a cognitive radar when the response of the radar is observed in noise.
A similar approach to above applies
for detecting the presence of a cognitive radar
that measures our probe signal in additive noise
as
\begin{equation}
	\nprobe_\dtime = \probe_\dtime + \pnoise_\dtime,
	\label{eqn:probenoisemodel2}
\end{equation}
where the noise $\{\pnoise_\dtime\}$ are $\probedim$-dimensional i.i.d.\ random variables; see \cite{KAEM20}.

\index{revealed preferences! statistical detection|)}

  \section{Active  IRL. Optimal Probing  to Minimize  Type-II Error Probability} \label{sec:adapt}

\index{revealed preferences! active IRL|(}

This section discusses how to adaptively interrogate an agent (adversary radar)  to detect if it is a utility maximizer (cognitive), based on noisy measurements of the agent's response. 
Theorem~\ref{thm:type1}  guarantees that if we observe the agent's response in noise, then the probability of Type-I errors (\textit{false positive}, i.e., deciding that the agent is not a utility maximizer   when it is) is less than $\threshold$ when employing  the decision test (\ref{eqn:Statistical_Test}).
 Our aim is to enhance the statistical test (\ref{eqn:Statistical_Test})  by adaptively selecting  probe vectors $\Probe=[\probe_1,\probe_2,\dots,\probe_\horizon]$ to reduce the probability of  Type-II errors
 (\textit{false negative}, i.e., deciding that the agent is a utility maximizer  when it is not).

 The framework is shown in Figure \ref{fig:optimize} and constitutes  \textit{active} IRL  since we purposefully probe the agent (adversary radar) to elicit a response.
 Specifically,
given batches of noisy measurements of the adversary's radar  response $\Obsresponse^l=[\obsresponse_{1}^l,\ldots,\obsresponse_{\horizon}^l] \in \reals_+^{\probedim\times \horizon}$, indexed by batch $l=1,2,\ldots$ (where each $\nresponse_k$ is generated by \eqrefp{eqn:noisemodel}), the aim is to adaptively design batches of our probe signals   $\Probe^l = [\probe_{1}^l,\ldots,\probe_{\horizon}^l] \in \reals_{++}^{\probedim\times \horizon}$, $l=1,2,\ldots,$  so as to   minimize  the probability of  Type-II errors of the statistical detector:
\begin{align}
&\argmin_{\Probe\in\reals^{\probedim\times \horizon}_+}J(\Probe) = \underbrace{\prob\Big(\!\!\!\!\!\!\int\limits_{\Tstat^*\big(\Response(\Probe)+\Anoise\big)}^{+\infty} \!\!\!\!\!\! f_{M}(\psi) d\psi > \threshold \big| \{\Probe,\Response(\Probe)\}\in \mathcal{A}\Big)}_\text{Probability of Type-II error}. 
\label{eqn: SPSA Objective}
\end{align}

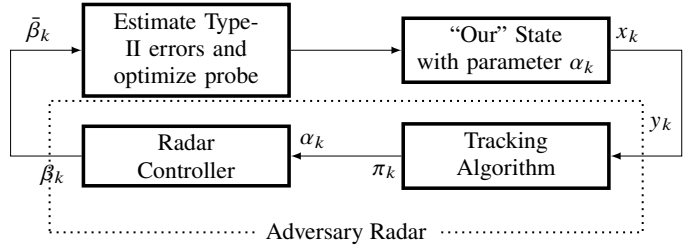
\begin{figure}[h]  \centering
\resizebox{9cm}{!}{ 
  \begin{tikzpicture}[node distance = 4.5cm, auto]
    \tikzset{
    block/.style={rectangle, draw, line width=0.5mm, black, text width=8em, text centered,
                 minimum height=2em},
               line/.style={draw, -latex}}

  \node [block] (BLOCK1) {Estimate Type-II errors and optimize probe};
  \node[block,right of = BLOCK1](opt){``Our'' State \\ with parameter $\probe_\dtime$ };
    \node [block, below of=BLOCK1,node distance=1.5cm] (BLOCK2) {Radar\\ Controller};
    \node[block, right of=BLOCK2] (grad) {Tracking\\ Algorithm};

    \path [line] (BLOCK1) -- (opt);
    
    \path [line] (opt.east) --++ (1cm,0cm) node[pos=0.25,above]{$\state_k$}  |-   node[pos=0.35,left]{$\obs_k$}    (grad);
    
     \path [line] (BLOCK2.west) --++ (-1cm,0cm)  node[pos=0.4,below]{$\response_\dtime$}  |-     node[pos=0.7,above]{$\obsresponse_\dtime$}     (BLOCK1.west);
     \path [line] (grad) --   node[pos=0.15,below]{$\belief_k$}    node[pos=0.8,above]{$\probe_\dtime$}  (BLOCK2);
     \node[draw, thick, dotted, inner sep=3ex, yshift=-1ex,
     fit=(BLOCK2) (grad)] (box) {};
     \node[fill=white, inner xsep=1ex] at (box.south) {Adversary Radar};
   \end{tikzpicture}
 }
 \caption{Optimizing the probe waveform to detect cognition in adversary's radar by minimizing the Type-II errors subject to constraints in Type-I errors.
   }
 \label{fig:optimize}
\end{figure}

Here $\prob(\cdot | \cdot)$ denotes the conditional probability that the statistical test (\ref{eqn:Statistical_Test}) accepts $H_0$, defined in (\ref{eqn:hypothesis}), given that $H_0$ is false. $\Anoise=[\anoise_1,\anoise_2,\dots,\anoise_\horizon]$ is the noise matrix where
the random vectors $\anoise_\dtime$ are defined in
(\ref{eqn:noisemodel}).  $\threshold$ is the significance level of  (\ref{eqn:Statistical_Test}).
 The set $\mathcal{A}$ contains all elements $\{\Probe,\Response(\Probe)\}$, with $\Response(\Probe)=[\response_1,\response_2,\dots,\response_\horizon]$, such that  $\{\Probe,\Response\}$ does not satisfy GARP~\eqref{eq:garp}, i.e., fails Afriat's theorem for utility maximization.

Since the pdf $f_M$ defined in (\ref{eqn:def:M}) is not known explicitly, (\ref{eqn: SPSA Objective}) is a stochastic optimization problem. To estimate a local minimizer  of the Type-II error probability $J(\Probe)$ w.r.t.\  $\Probe$, several types of stochastic optimization algorithms can be used. Algorithm \vref{alg:spsa}  employs  the SPSA algorithm. Recall that 
a useful property of  SPSA  is that estimating the gradient $\nabla_{\Probe} J_\iindex(\Probe^\iindex)$ in (\ref{eqn:SPSA})  requires only two measurements of the cost function (\ref{eqn:SPSACost}) corrupted by noise per iteration, i.e., the number of evaluations is independent of the dimension $\probedim\times \horizon$ of the matrix  $\Probe$. For decreasing step size $\mu_\iindex = 1/\iindex$,  SPSA  converges w.p.1. to a local stationary point of $J(\Probe)$.
\index{revealed preferences! active IRL|)}  
\index{SPSA algorithm}
\begin{algorithm}
\begin{steplist}
\item[\bf Step 1.]  Choose initial probe $\Probe^0=[\probe_1,\probe_2,\dots,\probe_\horizon]\in\reals^{m\times \horizon}_{++}$ (strictly positive).
  \item[\bf Step 2.]  For iterations $\iindex=0,1,\dots$
\begin{compactenum}
\item Estimate empirical Type-II error probability  $J(\Probe_\iindex)$ in (\ref{eqn: SPSA Objective}) in $L$ independent trials: 
\begin{equation}
\hat{J}_\iindex(\Probe^\iindex) =\frac{1}{L}\sum\limits_{l=1}^{L} I{\biggl( \hat{F}_M(\Tstat^*(\Obsresponse^l)) \leq 1-\threshold\biggr)}.
\label{eqn:SPSACost}
\end{equation}
In each trial $l$,   $\Obsresponse^l\in\reals^{m\times \horizon}_+$ is the noisy measurement matrix of the radar response  to our probe matrix~$\Probe^\iindex\in\reals^{m\times \horizon}_{++}$.
$I(\cdot)$ denotes the indicator function.  $\hat{F}_M(\Tstat^*(\Obsresponse^l))$ is the empirical cdf of $M$ computed  as in (\ref{eq:implement}).
 $\Tstat^*(\Obsresponse^l)$ is obtained from (\ref{eqn:AE})  using  noisy observation sequence $\Obsresponse^l$ where $\Anoise^l$ is  a realization of $\Anoise$, and  dataset $\{\Probe^\iindex,\Response(\Probe^\iindex)\}\in\mathcal{A}$, described below (\ref{eqn: SPSA Objective}).

\item  Compute the gradient estimate $\hat{\nabla}_{\Probe}$ with gradient step size $\omega > 0$:
\begin{align}
\hat{\nabla}_{\Probe}\hat{J}_\iindex(\Probe^\iindex) &=  \frac{\hat{J}_\iindex(\Probe^\iindex+\Delta_\iindex\,\omega)-\hat{J}_\iindex(\Probe^\iindex-\Delta_\iindex\,\omega)}{2\omega\Delta_\iindex} \label{eqn:SPSA}\\
\Delta_\iindex(j) &= \begin{cases}
   -1 & \text{with probability 0.5} \\
   +1       & \text{with probability 0.5.} 
  \end{cases} \nonumber
\end{align}
\item  Update the probe matrix  $\Probe^\iindex$ using  stochastic gradient algorithm with step size $\mu_\iindex>0$:
\begin{equation*}
\Probe^{\iindex+1} = \Probe^\iindex-\mu_\iindex\,\hat{\nabla}_{\Probe}\hat{J}_\iindex(\Probe^\iindex).
\end{equation*}
\end{compactenum}
\end{steplist}
\caption{Minimizing Type-II Error Probability for Detecting Utility Maximization}
\label{alg:spsa}
\end{algorithm}

\section{Robustness Margin for IRL.  Revealed Preference Violations} \label{sec:goodness}

\index{revealed preferences! goodness of fit}

This section discusses three  measures of rationality when GARP is violated in Afriat's theorem; see  \cite{DM16,DR23}.  In economics, these measures are important since a human decision maker may make a few mistaken choices, causing the entire dataset to fail GARP. In automated  sensing systems such as cognitive radar, these measures can be exploited to obfuscate  the utility function from an adversary, as discussed in  \secn \ref{sec:dumb}. The measures discussed below assess the robustness margin of the dataset, namely, how close  the dataset is to satisfying  utility maximization.

 From Afriat's Theorem \vref{thm:afriat},  a
 dataset $\mathcal{D}=\{\probe_\tindx,\response_\tindx, \tindx\in 1,2,\dots,\Tindxter\}$ satisfies GARP if the following holds: For every ordered subset $   \{i_1,i_2,\ldots,i_\lind\} \subset \{1,2,\ldots,N\}$, if
$$ \probe_{i_1}^\p\, \response_{i_2} \leq \probe_{i_1}^\p \,\response_{i_1}, \;\;
    \probe_{i_2}^\p\, \response_{i_3} \leq \probe_{i_2}^\p \,\response_{i_2}, \ldots,
    \probe_{i_{\lind-1}}^\p \, \response_{i_\lind} \leq \probe_{i_{\lind-1}}^\p\, \response_{i_{\lind-1}},
$$  
then it holds that
  $\probe_{i_\lind}^\p \response_{i_1} \geq \probe_{i_\lind}^\p \response_{i_\lind}$.
 GARP is necessary and sufficient for   $\dataset $ to be consistent with rationality (utility maximization). If the dataset $\dataset$ fails GARP, then how close is $\dataset$  to rationality?  We discuss three measures of rationality (goodness of fit) below.

\begin{renumerate}
\item  The Houtman--Maks index  (HMI) gives the largest  fraction of the dataset that is consistent with GARP. It is defined as  \index{revealed preferences! goodness of fit! Houtman--Maks index}
 $$ \operatorname{HMI}(\dataset) = \max_{A \in \{1,\ldots,\horizon\}} \frac{|A|}{\horizon}
 \quad \subjectto \quad  \{\probe_\dtime, \response_\dtime, \dtime \in A\} \text{ satisfies GARP. } $$
 If the entire dataset  $\dataset$ is consistent with GARP, then $\HMI(\dataset) = 1$.
 At the other extreme, if every probe and response pair in the dataset violates GARP, then  $\HMI(\dataset) = 1/\horizon$. Computing the HMI index is a combinatorial optimization problem and is  NP-hard.

\item The Varian index involves a relaxation of GARP called e-GARP. Consider the vector  $e=[e_1,\ldots,e_\horizon]$ with elements $e_\dtime \in [0,1]$. Then
dataset $\dataset$  satisfies e-GARP if the following holds: For every ordered subset $   \{i_1,i_2,\ldots,i_\lind\} \subset \{1,2,\ldots,N\}$, if
$$ \probe_{i_1}^\p\, \response_{i_2} \leq e_{i_1} \,\probe_{i_1}^\p \,\response_{i_1}, \;\;
    \probe_{i_2}^\p\, \response_{i_3} \leq e_{i_2}\, \probe_{i_2}^\p \,\response_{i_2}, \ldots,
    \probe_{i_{\lind-1}}^\p \, \response_{i_\lind} \leq e_{i_{\lind-1}}\,\probe_{i_{\lind-1}}^\p\, \response_{i_{\lind-1}},
$$  
then it holds that \index{revealed preferences! goodness of fit! Varian index}
  $\probe_{i_\lind}^\p \response_{i_1} \geq e_{i_\lind}\,\probe_{i_\lind}^\p \response_{i_\lind}$.
Note e-GARP becomes GARP when $e=\ones$.
Also, e-GARP becomes less stringent if smaller values of $e_\dtime$ are chosen.

The Varian index maximizes the average value of the vector $e$ that satisfies e-GARP:
$$ \VI(\dataset) = \max_{e \in [0,1]^\horizon} \frac{1}{\horizon} \sum_{k \leq \horizon} e_\dtime \quad \subjectto \dataset \text{ satisfies e-GARP}. $$
Computing $\VI(\dataset)$ requires solving a mixed integer linear program and is NP-hard. The Afriat index is a special case where all elements of $e$ are chosen equal.
\index{revealed preferences! goodness of fit! Afriat index}

\item The minimal cost index~\cite{DM16} determines the largest set of comparisons in the dataset $\dataset$ for which GARP
  holds. Define the set of ordered pairs $ R_{i_{1:\lind}} = \{(i_1,i_2),(i_2,i_3),\ldots,(i_{\lind-1},i_\lind)\} $ where the ordered subset $   \{i_1,i_2,\ldots,i_\lind\} 
  \subset
\{1,2,\ldots,N\} $ satisfies
$$   \probe_{i_1}^\p\, \response_{i_2} \leq \probe_{i_1}^\p \,\response_{i_1}, \;\;
    \probe_{i_2}^\p\, \response_{i_3} \leq \probe_{i_2}^\p \,\response_{i_2}, \ldots,
\probe_{i_{\lind-1}}^\p \, \response_{i_\lind} \leq \probe_{i_{\lind-1}}^\p\, \response_{i_{\lind-1}}.
    $$
    Define $R = \cup_{i_{1:\lind}} R_{i_{1:\lind}} $ denote all such ordered pairs. Let $B$ denote any subset of $R$.  
    We then define the minimal cost index as the minimal information cost we must ignore to satisfy GARP: 
    $$ \MCI(\dataset) = \min_{B \subset R}  \frac{\sum_{(k,\dtimen)\in B}  \probe_k (\response_k - \response_\dtimen)}{  \sum_{t=1}^\horizon \probe_t^\p \response_t}  \text{ such that }
R \backslash B \text{ satisfies GARP. }
$$
Here $R\backslash B$ denotes the difference of sets $R$ and $B$. Computing $\MCI(\dataset) $ is  also NP-hard.  \index{revealed preferences! goodness of fit! minimal cost index}

\end{renumerate}
  
\section{\pwe}  \label{sec:irl_pwe}

Traditional IRL~\cite{NR00,AN04} estimates an unknown deterministic utility (reward) function of an agent by observing the optimal actions of the agent in a Markov decision process (MDP) setting. The key assumption is that the agent is a utility maximizer. We will  discuss the literature on IRL further  in
\secn \ref{sec:pwe_bayesian_irl}, after  studying Bayesian IRL in Chapter~\ref{chp:birl}.

The revealed preference framework  in this chapter identifies if a dataset is consistent with  utility maximization  and then estimates the set of utilities that rationalize the dataset.
Revealed preference dates back to Samuelson~\cite{Sam38}.  Afriat~\cite{Afr67} showed that GARP is necessary and sufficient for the existence of a nonsatiated utility function given a finite time series
of choices and prices. 
Varian~\cite{Var82,Var83,Var12} also provides  tests for homothetic  and additive separability.   \cite{FST04} presents two clean intuitive proofs of Afriat's theorem; see also \cite{Die73,Die12}. \cite{FM09} extends Afriat's theorem to nonlinear budgets. Deeper ideas that underpin revealed preference stem from Rockafellar’s convex analysis book~\cite[Theorem 24.8, p.\ 238]{Roc70}.
Detecting  utility maximization given noisy measurements 
is studied in \cite{FW05,JE09}.

{\em Applications in Engineering}.
The author's interest in IRL and revealed preferences arises from three industrial collaborations. The first was  with a social networking company seeking IRL algorithms to predict user sentiment for YouTube videos \cite{HAK17,HKP20}. The second  was with a telecommunications company aiming to identify popular multimedia videos for adaptive  caching \cite{HNK15}.
The third collaboration, with a radar company, focused on identifying cognitive radars.
\cite{KAEM20,KPG21} use revealed preferences to construct IRL algorithms for identifying cognitive radars. \cite{KH12} introduces  a stochastic gradient algorithm to adapt probe signals to minimize the Type-II error detection probabilities.
The utility masking method in \secn \ref{sec:dumb} is from~\cite{PKB22,PKB23} and can be viewed as inverse IRL (IIRL), namely  RL that purposefully deceives the adversary's IRL.
\cite{HK14} uses revealed references for nonparametric  demand forecasting in energy aware smart grids. 
 \cite{AK17} formulates revealed preferences with change detection in social media applications.

{\em Asymptotics}. 
 Afriat's theorem  assumes a finite dataset comprising $\Tindxter$ points: If the dataset is consistent with utility maximization then it can be rationalized by a concave utility.
\cite{Ren15} analyzes the case $\Tindxter\rightarrow \infty$ and shows: If the dataset is consistent with utility maximization then it  can  be rationalized by a quasiconcave utility. \cite{Col78}
showed that if the utility is assumed to be Lipschitz continuous, then asymptotically one obtains a unique preference ordering.

{\em Partial Orderable Utility}.   Classical revealed preference  assumes the utility  is real-valued and thus orderable. What if the   utility function is  only partially orderable, e.g.,  a covariance matrix?
 Richter \cite{Ric66} developed  results in revealed preferences for partially ordered utilities.
\cite{NOQ17} shows that identifying maximization of partially ordered utilities involves  preferences on  Hausdorff spaces.
 \cite{PK22a}  uses  Richter's framework to  order probe and response vectors via Blackwell dominance and shows that  Theorem~\ref{thm:nonlinear_afriat} is equivalent to Bayesian revealed preferences discussed in Chapter~\ref{chp:birl}.

 {\em Multiagent Revealed Preferences}. 
\cite{CDV11} extends Afriat's theorem to identify collective rationality amongst multiple agents. 
\cite{SKS23} uses this  to identify coordination  in radar networks.
This complements  multiagent inverse reinforcement learning  \cite{YSE19,NKJ10,LAB19,ZYZ21}.
\cite{Deb09} uses Afriat's theorem to 
identify if a dataset from multiple agents is consistent with  play from the Nash equilibrium of a concave potential game. \cite{HKA16} applies these methods to
social networks like  X (formerly Twitter) 
 and also energy markets.




\chapter{Bayesian Inverse Reinforcement Learning}
\label{chp:birl}
\minitoc

\index{Bayesian IRL|(}

\begin{figure}[h]       \centering      
  \begin{tikzpicture}[scale=0.9, transform shape,node distance =2 cm and  3cm, auto]
    \tikzstyle{arrow} = [thick,->,>=stealth]
    \tikzset{
    block/.style={rectangle, draw, line width=0.5mm, black, text width=5em, text centered,
                 minimum height=2em},
               line/.style={draw, -latex}}
   \tikzset{
    block2/.style={rectangle, draw, line width=0.5mm, black, text width=9em, text centered,
                 minimum height=2em},
               line/.style={draw, -latex}}

  \node[block](sensor){Controlled Sensor};
  \node[block,right of=sensor, node distance=3.5cm](filter){Bayesian Decision Maker};
  \node[right of=filter,node distance=2.2cm](nullnode){};

   \draw[-Latex](filter) -- node[above,pos=0.55] {$a_{k}$} (nullnode);
  \draw[-Latex] (filter.south) -- ++(0,-0.3)    -|  node[pos=0.3,above] {$u_k$}  (sensor.south);
  \draw[-Latex] (sensor) -- node[above] {$\obs_k$}  (filter);
  \draw[Latex-] (sensor.west) -- ++(-1,0) node[left] {$\state_k\sim \belief_0$};
  \node[block,below of=sensor, right of=sensor](irl){IRL (Analyst)};
  \draw[-Latex] (filter.east) -- ++(0.5,0)  |-  (irl.east) ;
  \draw[-Latex] (sensor.west) -- ++(-0.5,0)  |-  (irl.west) ;
  
\end{tikzpicture}
\caption{IRL for rationally inattentive Bayesian utility maximizer.} 
\label{fig:irlbum}
\end{figure}
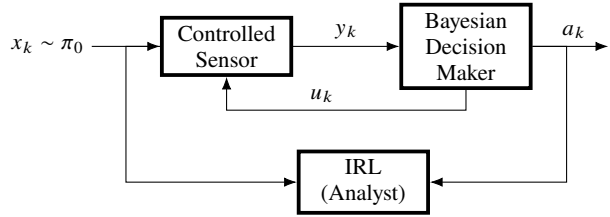

\noindent
The previous chapter studied nonparametric revealed preferences  for inverse reinforcement learning (IRL). This chapter discusses 
Bayesian IRL. The  framework  is illustrated in Figure~\ref{fig:irlbum}
and has two parts: a Bayesian agent (top two boxes in the figure)
and IRL (an analyst). 
\begin{senumerate}
\item The  Bayesian agent (forward learner)  optimizes a utility while performing controlled sensing.
The  agent chooses an action $\act_k$ to maximize its expected utility based on  noisy measurements $\obs_k$ of an underlying state $\state_k$. But the agent is rationally inattentive, so it also controls (optimizes) its observation likelihood to minimize the cost of
 information acquisition. The sensor's  control signal $\action_k$ affects  the  observation  $\obs_k$ and hence the action $a_k$ chosen by the agent. So  the Bayesian agent  jointly performs optimal controlled sensing with utility maximization. 

\item The inverse learner (analyst) conducts IRL to estimate the  Bayesian  agent's utility and controlled sensing costs.
The  analyst  observes the actions $\{a_k\}$ of the  agent.  The analyst  knows the underlying state
$x_k$ but does not know the agent's observation $\obs_k$ or   observation likelihood.
\end{senumerate}

In this chapter, we employ Bayesian revealed preferences to 
 solve the following Bayesian IRL  problem from the analyst's perspective:
\begin{compactenum}
  \item  We formulate necessary and sufficient conditions for the existence of a utility function and information acquisition cost  that rationalize the actions of a Bayesian agent.
  \item We estimate the utility function and information acquisition (controlled sensing)    cost of the agent that is consistent with the observed dataset.
  \end{compactenum}    
 
  Inverse optimization is an ill-posed problem. Similar to classical revealed preference (previous chapter),
  we construct \textit{set-valued}  estimates of the utility and information acquisition  cost.

Some terminology: Controlled sensing is equivalent to 
costly  information acquisition by the Bayesian agent. This is  studied in  economics under the area of ``rational inattention'' \cite{Sim03,Sim10,CDL19}, which is a form of bounded rationality. The key idea is that human attention spans for information acquisition are limited and  can be modeled in information theoretic terms as a Shannon capacity limited communication channel.
This chapter can be viewed as ``inverse social learning''.
 
\hdformat{Outline} We first  discuss  how  to achieve IRL for a one-step Bayesian agent.
Next  we show that the same theory applies to achieve IRL in a Bayesian stopping-time problem over a random horizon.
Then we discuss three examples of IRL in Bayesian stopping-time problems: inverse sequential hypothesis testing,  inverse search, and inverse quickest detection.
Finally, we discuss discrete choice random utility models. 
The Complements and Sources section discusses the IRL literature. The appendix discusses inverse filtering which complements  Bayesian IRL.

\section{IRL for Rationally Inattentive Bayesian Utility Maximizer}
\label{sec:BIRL}  \index{rational inattention! IRL}


This section comprises three parts.
We first describe the  viewpoint of a  utility maximizing rationally inattentive  Bayesian agent. Then we discuss the IRL problem  from the viewpoint of an  analyst that  observes the dataset generated by the  agent and aims to reconstruct the agent's utility and rational inattention cost.
 Finally, we give the main IRL result in Theorem~\ref{thm:BIRL}: A set of linear inequalities that are necessary and sufficient conditions for the actions of a Bayesian agent to be consistent with utility maximization with rational inattention cost. These inequalities yield set-valued estimates of the agent's utility and information acquisition cost.  They constitute the IRL procedure  for identifying a rationally inattentive  Bayesian utility maximizer.

{\em Notation}.
$\statespace= \{1,\ldots,\statedim\}$ denotes the state space,  $\actionset= \{1,\ldots,\actdim\}$ denotes the action space, $\obspace$ denotes the observation space which can either be the finite set $\{1,\ldots,\obsdim\}$  or a  subset of $\reals$.

$\dpset=\{1,2,\ldots,\numdp\}$ is a set of environments over which the agent acts.
We view each environment $\env \in \dpset$ as a distinct experiment setup where the agent is observed by the analyst.
As discussed below, in
each  environment $\env\in \dpset$, the agent has a  distinct reward vector.

\subsubsection*{Viewpoint 1. Utility Maximizing Rationally Inattentive (UMRI) Bayesian Agent}

The purpose of this subsection is to define a UMRI agent and its decision making protocol.

A  UMRI agent uses the following three model parameters in its decision making:
\begin{compactenum}
\item Prior $\belief_0\in \Belief$, where   $\Belief$ is the space of $\statedim$-dimensional probability vectors.
\item A reward vector $\reward_{\env,\act}=[\reward_{m}(1,\act),\ldots, \reward_m(\statedim,\act)]^\p$ in each environment $\env \in \{1,\ldots,M\}$.
\item  Information acquisition (sensing) cost  $\RIcost(\attfunsymb,\belief_0)$.
Here, 
the observation likelihood kernel $\oprob = (\oprob_{xy} = \pdf(\obs|x), x\in \statespace, \obs \in \obspace)$ is known to the agent.
\end{compactenum}
We denote  the optimized observation likelihood in each environment $\env\in \dpset$ as  $\oprob(\env)$ and
$$\oprob_{\obs}(\env) = \diag(\oprob_{1y}(\env),\ldots,\oprob_{\statedim \obs}(\env)),  \, \obs \in \obspace. $$

The following  protocol specifies a UMRI Bayesian agent operating in $\numdp\geq 2 $ environments. 

\noindent {\bf Protocol for UMRI Bayesian Agent}. In each environment $\env \in \dpset=\{1,2,\ldots,\numdp\}$:  \label{pg:umri_protocol}
\begin{steplist}
  \item   \label{step:cs}
 Given its model parameters, the agent optimizes its controlled sensing policy (observation likelihood  kernel $\oprob(\env)$),
by solving the following optimization problem w.r.t.~$\oprob$:
\begin{equation}
  \begin{split}
    \oprob(\dpiter)& \in\argmax_{\attfunsymb \in \attspace} J(\reward_m,\oprob,\belief_0) \ole R(\reward_\env,\oprob,\belief_0) - \RIcost(\attfunsymb,\belief_0)
\\    R(\reward_\env,\oprob,\belief_0) &\ole                      
    \E\big\{\max_{\bact\in\actionset}\E\{ \utilitysymbol_\dpiter(\state,\bact) | \obs\}\big\}  \overset{\text{(a)}}{=} \sum_{\obs \in \obspace} \max_{\bact\in\actionset} \utilitysymbol^\p_{\dpiter,\bact} \, \oprob_\obs\, \belief_0 . 
  \end{split} \label{eq:attentionmaximization}   
\end{equation}
The optimization~\eqref{eq:attentionmaximization} is performed offline over the space
$\attspace$  of observation likelihoods comprising
$\statedim\times \obsdim$ stochastic kernels $\oprob$. Equality (a) is explained in the remark below.
\item The true state $\truestate \in \statespace$ is drawn from prior  pmf $\belief_0$.
But $\truestate$ is not known to the agent.

\item 
The agent draws  observation $\obs$  from the optimized observation  likelihood $\oprob_{\truestate \obs}(\env)$. 
 \item 
The agent then computes the posterior $\belief=[\prob(\truestate=1|\obs),\ldots,\prob(\truestate=\statedim|\obs)]^\p$  as
\begin{equation}\label{eq:compute_posterior}
\belief  =  \filter(\belief_0,\obs,m) \ole \frac{  \oprob_y(\env) \,\belief_0 }{\filterd(\belief_0,\obs,m)},
  \quad \filterd(\belief_0,\obs,m) = \ones^\p \oprob_y(\env)\belief_0.
  \end{equation}
This is Bayes formula for computing  posterior $\belief$ given prior $\belief_0$ and observation~$\obs$.
\item {\em Optimal action}. Finally,  the agent chooses  action $a \in \actionset$ to maximize its expected utility 
\begin{equation}
  \act \in\argmax_{\bact \in \actionset}\E\{ \utilitysymbol_\dpiter(\truestate,\bact) | \obs\}=
\argmax_{\bact \in \actionset}  \reward_{m,\bact}^\p \filter(\belief_0,\obs,\env) =
  \argmax_{\bact \in \actionset}  \reward_{m,\bact}^\p \,\oprob_y(\env) \,\belief_0 . 
\label{eq:utilitymaximization}
\end{equation}
Note that $R(\reward_\env,\oprob,\belief_0)$ in  Step 1  is the expectation (over $\obs$)  of this maximized utility.
\end{steplist}

To summarize,  a UMRI Bayesian agent  is parametrized by the tuple 
\begin{equation}
  \label{eq:umri_tuple}
\big(\dpset,\stateset,\obsset,\actionset,\belief_0,\RIcost,\{\oprob(\dpiter),\utilitysymbol_\dpiter,\dpiter\in\dpset\}\big).
\end{equation}

\begin{rem}

\item \label{comment:arem} Consider the controlled sensing Step 1. Equality (a) in~\eqref{eq:attentionmaximization} holds since   
  $$ R(r_\env,\oprob,\belief_0) =\E_{\obs}\big\{ \max_{\bact} \reward^\p_{\env\bact} \filter(\belief_0,\obs,\env) \big\}
  = \sum_y \max_{\bact} \reward^\p_{\env\bact} \,\filter(\belief_0,\obs,\env)\, \filterd(\belief_0,\obs,\env). $$
\item The information acquisition (rational inattention)  cost $\RIcost(\attfunsymb,\belief_0)$ in~\eqref{eq:attentionmaximization}  is the sensing cost the agent incurs to estimate the underlying state \eqref{eq:compute_posterior}. It is useful to interpret this cost   as
  \begin{equation}
    \label{eq:inf_acq_one}
\RIcost(\attfunsymb,\belief_0) = \sum_\obs \fun\big(\filter(\belief_0,\obs), \belief_0\big)\,\filterd(\belief_0,\obs) 
\end{equation}
where function $\fun$ is chosen by the agent. Often  $\fun$  is chosen as an entropic regularizer (mutual information) in social learning.
In rational inattention  (and controlled sensing),
a higher information acquisition cost is incurred for a more accurate attention strategy  (resulting in a  more accurate state estimate \eqref{eq:compute_posterior} given observation $\obs$).


\item The  optimizations~\eqref{eq:attentionmaximization} and~\eqref{eq:utilitymaximization} define  a UMRI Bayesian agent. Without these optimizations, the above protocol reduces to a classical Bayesian agent that computes the posterior given an observation.  The inner-level optimization~\eqref{eq:utilitymaximization} chooses the  optimal  action $\act$ for a given observation~$\obs$ based on the computed posterior belief. The outer-level optimization~\eqref{eq:attentionmaximization}   chooses  the observations optimally by selecting the optimal attention strategy. Owing to the $\E\{\cdot\}$, the optimal attention strategy is independent of the specific observation and so $\oprob(\dpiter)$ can be precomputed. Referring to  Figure~\vref{fig:irlbum}, the  sensor control signal $u_k$ represents  the optimized observation likelihood $\oprob(\dpiter)$ obtained by solving~\eqref{eq:attentionmaximization} in each environment $\env$.

\item The  state space $\stateset$, observation space $\obsset$, action space $\actionset$, prior $\belief_0$ and  information acquisition cost $\RIcost$ are identical across environments,  but  the reward vector $\reward_{\env,\act}$ depends on the environment $\env$.
  Bayesian revealed preference theory relies on this crucial assumption.
  \end{rem}

\subsubsection*{Viewpoint 2. Inverse Reinforcement learning: Analyst's Model}

We now consider the viewpoint of the analyst.
In  each environment $\dpiter\in \dpset$, the analyst observes the true state $\truestate$ drawn from prior $\belief_0$, and the action $a$ taken by an agent (not necessarily a UMRI Bayesian agent). The analyst observes infinitely many such i.i.d.\ trials. 
By the strong law of large numbers, the empirical frequency of draws of $\truestate$ converges w.p.1. to $\belief_0$. Similarly, the empirical frequency of the agent's actions, given the true state $\truestate=\state$, converges w.p.1 to
the conditional probability that the agent chooses action $\act$, given  $\truestate=\state$ in  each environment $\env$:
\begin{equation}
  \label{eq:suff_stat}
\bB(\env) \ole \{p_\dpiter(\act|\state) ,\state\in\stateset,\act\in\actionset\}, \quad \dpiter\in\dpset.
\end{equation}
These action probabilities serve as a sufficient statistic for the agent from the analyst's viewpoint.

We assume that the  analyst has   the dataset 
\begin{equation} \label{eq:dataset_accum}
\datasetaccum=\{\belief_0 ,p_{\dpiter}(\act|\state) ,\state\in\stateset,\act\in\actionset,\dpiter\in\dpset\}
\end{equation}
generated by an  agent in $\numdp$ environments.
Therefore, in each environment $\env$, the analyst constructs  the \textit{revealed posterior}  using Bayes formula as
\begin{equation}
  \label{eq:revealed_posterior}
\big(\pdf_\env(x|a) , x \in \statespace\big) =
\frac{\bB_a(\env) \,\belief_0}{\ones^\p \,\bB_a(\env) \,\belief_0}, \quad
\text{ where } \bB_a(\env) = \diag\big(p(a|1),\ldots,p(a|\statedim)\big).
\end{equation}

Given the dataset $\datasetaccum$ , the analyst has two IRL objectives:
\begin{compactenum} \item Is $\datasetaccum$ consistent with  a UMRI Bayesian agent? Recall a UMRI agent is defined by~\eqref{eq:umri_tuple} and the five-step protocol described \vpageref{pg:umri_protocol}.
  
\item If  yes,  estimate the agent's utility $\utilitysymbol_\dpiter(\state,\act) $ and information acquisition cost  $\RIcost(\attfunsymb,\belief_0) $.
\end{compactenum}

\remar
The analyst does not know the agent's observation space $\obspace$, observation samples $\obs$ or
observation likelihood $\oprob(\env)$. Such problems arise when the 
 adversary's  sensor is not known.


\subsubsection{Main Result. Bayesian Revealed Preferences}

We  now  state our main Bayesian IRL result from the viewpoint of the analyst. The following theorem  says dataset
$\datasetaccum$ is generated by a UMRI agent {\em if and only if} a set of linear inequalities
(called NIAS and NIAC) has   a feasible solution. 
These inequalities constitute  our {\bf B}ayesian {\bf R}evealed {\bf P}reference (henceforth called BRP) test for rationally inattentive utility maximization.

\begin{theorem}[BRP test for UMRI agent~\cite{CD15}]\label{thm:BIRL} Suppose the analyst has the dataset~$\datasetaccum$, defined in~\eqref{eq:dataset_accum},   from an   agent operating in $\numagents\geq 2 $ environments. Then,\\
{\em 1.} Existence:  An UMRI agent rationalizes dataset $\datasetaccum$ iff there exists a feasible solution
\begin{equation}\label{eq:BRP_ineq}
\BRP\big(\datasetaccum,\{\hutilitysymbol_\dpiter,\hb_\dpiter\}_{\dpiter=1}^{\numdp}\big) \leq \mathbf{0}, \quad
  \hutilitysymbol_\dpiter\in\reals_+^{|\stateset|\times|\actionset|},
\, \hb_\dpiter \in \reals_+.
\end{equation}
Here $\BRP(\cdot)$ is the  set of linear inequalities (in variables $\{\hutilitysymbol_\dpiter,\hb_\dpiter\}_{\dpiter=1}^{\numdp}$)  in Algorithm~\ref{alg:dtest}. \\
{\em 2.} Reconstruction: Given any feasible solution $\{\hutilitysymbol_\dpiter,\hb_\dpiter\}_{\dpiter=1}^\numdp$ to $\BRP(\datasetaccum,\cdot)$:
\begin{senumerate} 
  \item  The set-valued estimate of the  agent's utility (reward) in environment $\dpiter$ is 
    $\hutilitysymbol_\dpiter$.
  \item The set-valued estimate of the agent's  information acquisition $\RIcost$
    is
\begin{equation}
    \hRIcost(\bB,\belief_0)  = \max_{\dpiter\in\dpset} \Big\{\hb_\dpiter + \sum_{\act}\max_{\actiontwo\in\actionset}\sum_{\state} p(\state,\act)\,\hutilitysymbol_\dpiter(\state,\actiontwo) - \sum_{\state,\act}p_{\dpiter}(\state,\act)\,\hutilitysymbol_{\dpiter}(\state,\act)\Big\}. \label{eq:BRP_reconstruct0}
  \end{equation}
  Here, the reconstructed $\hRIcost(\bB,\belief_0) $ is a function of variable 
 $\bB=\{p(\act|\state), \act\in \actionset, \state\in \statespace\}$. Also, $p(\state,\act) = \belief_0(x) p(a|x)$, where $\belief_0$ is the prior.
\end{senumerate}
\end{theorem}
The proof of Theorem  is in Appendix~\vref{proof:BIRL}.

\begin{algorithm}[ht]
\begin{algorithmic}
\REQUIRE Dataset $\datasetaccum=\{\belief_0,p_{\dpiter}(\act|\state),\state\in\stateset,\act\in\actionset,\dpiter\in\dpset\}$ from  Bayesian agent.

\hspace{-0.35cm}\textbf{Find:}  \index{Bayesian IRL! NIAS and NIAC conditions}  $\hb_{\dpiter} \in \reals_+$ and $\hutilityagent{\dpiter} \in \reals_+$ for all $\state\in\stateset,$ $\act\in\actionset,~\dpiter\in\dpset$ that satisfy the NIAS (No-Improving-Action-Switches) and NIAC  (No-Improving-Attention-Cycles) inequalities:
\begin{align} \label{eq:NIAS} 
\underline{\textbf{NIAS}}\colon&~\sum_{\state}p_{\dpiter}(\state|\act)~\big(\hutilitysymbolagent{\dpiter}(\state,\actiontwo) -\hutilityagent{\dpiter}\big)\leq 0,
\quad  \forall\act,\actiontwo\in\actionset,~\dpiter\in\dpset, \\ 
   \underline{\textbf{NIAC}}\colon& \sum_{\act} \max_{\actiontwo}\sum_{\state}p_{\dpitertwo}(\state,\act)\,\hutilitysymbolagent{\dpiter}(\state,\actiontwo)   - \hb_{\dpitertwo}\label{eq:NIAC}\\
    & \hspace{1cm} - \bigg[ \sum_{\act} \sum_\state p_{\dpiter}(\state,\act)\,\hutilityagent{\dpiter} -\hb_{\dpiter} \bigg]\leq 0, \quad \forall\dpitertwo,\dpiter\in\dpset,\nonumber
\end{align}
\hspace{-0.14cm}where $p_{\dpiter}(\state,\act)=\belief_0(\state)\,p_{\dpiter}(\act|\state),~p_{\dpiter}(\state|\act)=\frac{p_{\dpiter}(\state,\act)}{\sum_{\bstate}p_{\dpiter}(\bstate,\act)}$.\\
\hspace{-0.35cm}\textbf{Return:} Set of feasible utility functions $\hutilitysymbolagent{\dpiter}$ and information acquisition costs $\hb_{\dpiter}$ incurred by the agent in environments $\dpiter\in\dpset$.
\end{algorithmic}
\caption{BRP Linear Feasibility Test 
$\BRP\big(\datasetaccum,\{\hutilitysymbol_\dpiter,\hb_\dpiter\}_{\dpiter=1}^{\numdp}\big) \leq \mathbf{0}$}
\label{alg:dtest}  
\end{algorithm}

\begin{rem}
\item
The BRP linear feasibility test in Theorem~\ref{thm:BIRL} comprises two sets of inequalities, namely, the NIAS (No-Improving-Action-Switches) inequalities~\eqref{eq:NIAS} and NIAC (No-Improving-Attention-Cycles) inequalities~\eqref{eq:NIAC}. We will say a lot more about these in subsequent remarks.

  We emphasize the ``iff" in Theorem~\ref{thm:BIRL}. If the inequalities in \eqref{eq:BRP_ineq} are not feasible, then the   dataset~$\datasetaccum$ is  not generated by a UMRI Bayesian agent. 
  If \eqref{eq:BRP_ineq} has a feasible solution, then there exists a reconstructable set of  utility functions and information acquisition costs that rationalize~$\datasetaccum$.
The necessity implies for a  UMRI agent, the true utility  and information acquisition cost satisfy NIAS and NIAC; in this sense NIAS and NIAC yield consistent estimates.

The phrase ``set-valued estimate'' in  the theorem warrants discussion. Clearly, we are not constructing statistical estimates. The estimates are the set of utility functions and information costs that satisfy the necessary and sufficient NIAS and NIAC conditions for Bayesian utility maximization.
Every utility and information cost in the feasible set explains $\datasetaccum$ equally well. 

\item  {\em Identifiability.}
  The BRP feasibility test requires the dataset $\datasetaccum$ to be generated from $\numdp\geq 2$ environments. If $\numdp=1$, then \eqref{eq:BRP_ineq} holds trivially for  any information acquisition cost. 

\item {\em System identification for social learning}.   Theorem~\ref{thm:BIRL}  can be viewed as set-valued system identification of the reward (cost) of a social learning agent. We (analyst) take an  agent in isolation, make it operate in at least two environments,  measure its actions, and then determine if its a UMRI and estimate the agent's utility and information acquisition cost.  The IRL procedure does not require knowing the agents observation or observation likelihood.

  In social learning, the aim is to estimate the state $x$; we knew the agent's rewards and  we computed the public belief $p(x|a)$  by observing the agent's action. In IRL, the aim is to estimate the agent's reward; we know state $x$ and observe the public belief (equivalently, action probabilities $\bB(m)$), and we then construct set-valued estimates of the agent's  reward.

\item
{\em Computational aspects.} Since  the dataset $\datasetaccum$ is obtained from $\numdp$ environments, $\BRP(\datasetaccum)$ is a feasibility test with $\numdp~(\statedim \actdim+1)$ free variables and
$\numdp^2 + \numdp (\actdim^2-\actdim-1)$ linear inequalities.

\item {\em NIAS and NIAC inequalities.} \label{item:nias}  NIAS ensures that for any fixed environment $\dpiter$,  the agent chooses the optimal  action based on the posterior pmf. In comparison, NIAC operates over pairs of environments $(\dpitertwo,\dpiter)$       and ensures that the agent chooses the best attention strategy over all $\numdp$ environments. The BRP  test checks if there exist $\numdp$ utility functions $\hutilitysymbolagent{\dpiter}$ and $\numdp$ positive reals $\hb_{\dpiter}$ that, together with dataset $\datasetaccum$, satisfy the NIAS and NIAC inequalities.

In terms of the expected reward $R$ in~\eqref{eq:attentionmaximization},
we can express the  NIAC inequalities~\eqref{eq:NIAC} as
\begin{equation}  \label{eq:niac_full}\\
  \begin{split}
 &R(\hat{r}_m, \bB(\env),\belief_0) - \hb_{\dpiter} \geq
R(\hat{r}_m, \bB(l),\belief_0) - \hb_{\dpitertwo}   \quad \forall\dpitertwo,\dpiter\in\dpset, \\
  \text{ where } \;  &
                       R(\hat{\reward}_m,\bB(m),\belief_0) \ole \sum_a \sum_\state p_m(a|x)\, \hat{\reward}_m(x,a) \,\belief_0(x),\;
\quad \hb_\env, \hb_{\dpitertwo}, \hat{\reward}_m(x,\cdot) \in \reals_+ 
    \\
    &
    R(\hat{\reward}_m,\bB(l),\belief_0) \ole \sum_a \max_{\bar{a}} \sum_\state p_l(a|x)\, \hat{\reward}_m(x,\bar{a}) \,\belief_0(x) .
    \end{split}
  \end{equation}

\item {\em NIAC. Pairwise versus combinatorial version}.   The NIAC condition~\eqref{eq:NIAC}  is written in terms of pairwise environments $(\dpitertwo,\dpiter)$. This is  equivalent to the NIAC condition in the original work of \cite[Theorem 1]{CD15} which was expressed in
 combinatorial form as:
  \begin{equation}
    \label{eq:niac_cd}
    \sum_{\dpiter \in \bar{\dpset}} \Big[
R(\hat{\reward}_m,\bB(m+1),\belief_0) - R(\hat{\reward}_m,\bB(m),\belief_0) 
    \Big]\leq 0
  \end{equation}
  for all  subsets  $\bar{\dpset} = \{\dpiter_1,\ldots,\dpiter_{\bar{\numdp}}\} \subseteq \dpset$. Here we interpret $\dpiter_j+1 =\dpiter_{j+1}$ if $j <\bar{\numdp}$ and 
  $\dpiter_{\bar{\numdp}+1} = \dpiter_1$. There are
  $2^{\numdp}-\numdp-1$ such subsets $\bar{\dpset}$  with two or more elements; so  we call~\eqref{eq:niac_cd} a combinatorial version.    
  Unlike~\eqref{eq:NIAC}, this combinatorial NIAC condition does not involve the information cost $\hb_{\dpiter}$; see \cite[Appendix~C.2.2]{PK23} for an equivalence proof.

The NIAC condition~\eqref{eq:niac_cd}  ensures that gross utility cannot be increased by reassigning attention strategies along any cycle of environments. Put differently, the  name NIAC (No-Improving-Attention-Cycles) stems from~\eqref{eq:niac_cd}: in environment $m$, if we   use   the attention strategy from different environment  $m+1$, then we are  worse off. So cycling through all environments in  $\bar{\dpset}$ in this way, does not improve performance.

\item 
  {\em Insight: Necessity of NIAC}.  \label{rem:niacn}  Where does  the  NIAC condition~\eqref{eq:niac_cd} come from? To provide  insight, we  now prove  the necessity of NIAC~\eqref{eq:niac_cd} for a UMRI agent.  Consider a UMRI agent~\eqref{eq:umri_tuple} with reward
  $\utilitysymbol_\dpiter$, $\dpiter\in\dpset$.
By optimality (UMRI),   the following inequalities hold for $\bar{\dpset} = \{\dpiter_1,\ldots,\dpiter_{\bar{\numdp}}\} \subseteq \dpset$  (recall $R$ and $K$ below were defined in~\eqref{eq:attentionmaximization}):
\begin{equation*}
  \begin{split}
    R(r_1,B(1),\belief_0) - K(B(1),\belief_0) &\geq R(r_1,B(2),\belief_0) - K(B(2),\belief_0) \\
    R(r_2,B(2), \belief_0) - K(B(2), \belief_0) &\geq R(r_2,B(3), \belief_0) - K(B(3), \belief_0)  \\
     &\vdots \\
    R(r_{\bar{\numdp}},B(\bar{\numdp}),\belief(0) - K(B(\bar{\numdp}),\belief_0) &\geq R(r_{\bar{\numdp}},B(1),\belief_0)- K(B(1),\belief_0).
  \end{split}
\end{equation*}
Summing the above inequalities,  implies  that for a UMRI agent, 
\begin{equation}
  \label{eq:niac_base}
   \sum_{\dpiter \in \bar{\dpset}}   \big[ R(r_m,B(m+1),\belief_0 ) - R(r_m,B(m),\belief_0 )  \big] \leq 0.
 \end{equation}
 The insight is that  NIAC~\eqref{eq:niac_cd} simply amounts to replacing   the observation probabilities $B$ in~\eqref{eq:niac_base} with the observed action probabilities $\bB$ (known to the analyst).
To justify this replacement, Appendix~\ref{proof:BIRL} shows that  for a UMRI agent,
$R(r_m,\bB(m),\belief_0) = R(r_m,B(m),\belief_0)$ and $R(r_m,\bB(m+1),\belief_0) \leq  R(r_m,B(m+1),\belief_0)$. Then clearly~\eqref{eq:niac_base} implies NIAC \eqref{eq:niac_cd}. In summary,   we have proved the necessity of NIAC (combinatorial version)  for  a UMRI agent.

\item \textit{Sufficiency of NIAC}.   We now prove that the combinatorial version of NIAC~\eqref{eq:niac_cd} implies that
  the pairwise version of  NIAC~\eqref{eq:NIAC} holds. Appendix~\ref{proof:BIRL}  in turn  shows that pairwise NIAC~\eqref{eq:NIAC} implies the existence of a UMRI agent, thereby establishing sufficiency of NIAC.

\begin{theorem}  \label{thm:pairwise}  The combinatorial NIAC~\eqref{eq:niac_cd}  implies that the  pairwise  NIAC~\eqref{eq:NIAC} holds. 
\end{theorem}  
\begin{proof}  We start with the following lemma which establishes that 
  there exists no reassignment of the attention strategies
 that increases the total reward.  (This goes in the  reverse direction of  Remark~\ref{rem:niacn} above.) 

 \begin{lemma}[{\cite[p.\ 108]{LY16}}]  \label{lem:assignment}
  With $R(r_\env,\bB(\benv),\belief_0) $ defined in~\eqref{eq:niac_full},
  NIAC~\eqref{eq:niac_cd}  implies that  the following linear assignment problem has an identity map
   solution  $\alp^*_{\env,\env} = 1$:
   \begin{equation}
     \label{eq:assignment_lp}
     \alp^*=  \argmax_{\alp}  \sum_\env \sum_\benv     R(r_\env,\bB(\benv),\belief_0) \,\alp_{\env \benv} \quad \text{ s.t. }
       \sum_\env \alp_{\env \benv} = 1, \; \sum_\benv \alp_{\env, \benv} = 1, \; \alp_{\env,\benv} \geq 0.
   \end{equation}
 \end{lemma}
  
 Next, from~\cite[p.\ 108]{LY16}, it follows via duality   that for the  linear program~\eqref{eq:assignment_lp}, there exist nonnegative variables  $ \hb_\env , \env \in \envspace$ such that
 pairwise  NIAC~\eqref{eq:niac_full} holds.
\end{proof}

\item If   the analyst knows  the agent's utility, then
it  uses NIAC and~\eqref{eq:BRP_reconstruct0}  to reconstruct the information acquisition  (rational inattention) cost
  of the agent. Alternatively, if the information acquisition cost is known, then the utility functions can be reconstructed using  NIAS and NIAC.  If both the utility function and  information acquisition cost are not known to the analyst, then Theorem~\ref{thm:BIRL}  loses specificity since arbitrary choices of utility and information assisting cost  can satisfy NIAS and NIAC. In this case it is
necessary to impose a robust margin for the NIAS and NIAC conditions to introduce specificity into the test. This is discussed  next.
  
\item {\em NIAS and NIAC with $\epsilon$-feasibility margin.}
  A trivial solution that satisfies both NIAS and NIAC  is the constant utility together with an  information cost of all zeros. Such degeneracy is a consequence of the ill-posedness of  inverse optimization problems. In practice, one can ensure only nontrivial solutions satisfy  NIAS and NIAC by introducing a margin constraint:
\begin{equation}\label{eq:niasc_margin}
    \operatorname{NIAS}(\cdot)\leq -\epsilon,~\operatorname{NIAC}(\cdot) \leq -\epsilon, \quad \text{ where } \epsilon>0.
\end{equation}
Introducing margin constraints to  ensure  nondegenerate solutions to feasibility tests is  common in IRL~\cite{RAT06}. The $\epsilon$ restriction of \eqref{eq:niasc_margin}  ensures only nontrivial informative costs pass the NIAS and NIAC feasibility test of Theorem~\ref{thm:BIRL}.

\item {\em Reconstruction of rational inattention cost}.  \eqref{eq:BRP_reconstruct0} specifies the reconstructed  rational inattention cost $\hRIcost(\bB,\belief_0)$  as
a bounded, piecewise linear and convex function of~$\bB$.  Note also that $ \hRIcost(\bB(m),\belief_0)= \hb_\dpiter$. The reconstruction~\eqref{eq:BRP_reconstruct0}  is different to~\cite{CD15}, which is
 \begin{equation}
    \label{eq:CD_RAreconstruct}
  \hRIcost(\bB,\belief_0) =
  \begin{cases} \hRIcost(\bB(m),\belief_0) =\hb_\dpiter &  \text{ if } \bB = \bB(\env) \\
    \infty & \text{otherwise}.
                         \end{cases}                           
                       \end{equation}
     The advantage of~\eqref{eq:BRP_reconstruct0} compared to~\eqref{eq:CD_RAreconstruct}  is  that the rational inattention cost is bounded.

  \end{rem}

  \summar Theorem~\ref{thm:BIRL} and Algorithm~\ref{alg:dtest}  form the basis of Bayesian IRL. The NIAS and NIAC inequalities are  necessary and sufficient conditions for the existence of a
utility maximizing rationally inattentive
(UMRI) agent, i.e., a boundedly rational Bayesian utility maximizer. Furthermore, the NIAS and NIAC inequalities  yield a set of  utility functions and rational inattention costs that rationalize the agent's dataset
$\datasetaccum=\{\belief_0,p_{\dpiter}(\act|\state),\state\in\stateset,\act\in\actionset,\dpiter\in\dpset\}$.

  

  \section{IRL for Bayesian Stopping-Time Problems} \label{sec:irl-bayes-stop}
  \index{Bayesian IRL! stopping-time problem}

This section discusses IRL for Bayesian stopping-time problems. By observing the stopping actions of an agent over a random horizon, how can an analyst decide if these actions are consistent with optimal stopping, and how can the analyst estimate the stop and continue costs?
Our motivation stems from IRL for inverse sequential detection: are the detector's decisions   consistent with Bayes optimality?  Recall that a sequential detection problem is a Bayesian stopping-time POMDP, i.e.,  a stochastic control problem over a random horizon. 

The key idea below is that a Bayesian stopping-time problem  can be mapped to a one-step rationally inattentive Bayesian utility maximizer (discussed in the previous section). So the main IRL result of the previous section applies  to Bayesian stopping-time problems. As a consequence,  the IRL procedure in Algorithm~\vref{alg:dtest} can be used to reconstruct the stopping/continue costs by observing the decisions of the  stopping-time controller.

{\bf Notation}.
We have already defined 
$\dpset$ as the set of environments,  $\stateset= \{1,\ldots,\statedim\}$ as the finite state space, and  $\obspace \subset \reals$ as the observation space in  \secn \ref{sec:BIRL}. The $\statedim$-dimensional probability vector $\belief_0$ is the prior distribution at time 0.
In addition:
\begin{compactitem}
\item  $\actionset=\{1,\ldots,\actdim\}$ is the set of stopping actions. For example, in sequential hypothesis testing discussed in \secn \ref{sec:SHT_back},  there are multiple stopping actions,
one for each hypothesis.
  
\item $\{\contcost_k(x) , k\geq 0, \state \in\statespace\}$ are the continue costs  at time $k$ given state $\state$.
These  costs are not environment dependent.  Define the continue cost vector as
$\contcost_k=[\contcost_k(1),\ldots,\contcost_k(\statedim)]^\p$. We assume $c_k$ is a nonnegative vector with at least one strictly positive element.

\item  $\{\stopcost_\env(x,\act),\state\in\stateset,\act\in\actionset,\agent\in\agentset\}$
  is the cost for taking stop action $\act$ in state $\state$ and environment  $\agent$. Define the stopping cost vector $\stopcost_{\env,\act}=[\stopcost_\env(1,a),\ldots,\stopcost_\env(\statedim,a)]^\p$.
Finally define $\stopcost_\env=\{\stopcost_{\env,\act},\act\in \actionset\}$.

\item As in the previous section, the conditional pdf  $\oprob_{xy} = \pdf(\obs|x)$, $x\in \statespace, \obs \in \obspace$  denotes the observation likelihood, 
  $\oprob_\obs = \diag(\oprob_{1y},\ldots,\oprob_{\statedim \obs})$, $\obs \in \obspace$.

\end{compactitem}

Similar to the   one-step Bayesian utility maximization  discussed in the previous section,  we will describe the framework from  the  viewpoint of an optimal stopping Bayesian agent that makes decisions, and then from the view point of an analyst that observes the decisions of the agent and performs IRL to reconstruct the costs of the Bayesian agent.

\subsubsection{Viewpoint 1. Optimal Bayesian Stopping Agent}

The following protocol specifies a Bayesian stopping agent (not necessarily optimal) that operates in $\env\in \{1,\ldots,\numdp\}$ environments. We require $\numdp\geq 2$ for identifiability.
\begin{steplist}
\item Draw true state $\truestate \sim \belief_0$ at time $k=0$. Here $\truestate$ is not known to the agent.
\item At each  $k>0$, the agent draws (observes) measurement $\obs_k \sim \oprob_{\truestate y }$.
\item The agent updates its belief using Bayes formula as
  $$ \belief_k = \filter(\belief_{k-1},\obs_k) =
  \frac{\oprob_{\obs_k} \, \belief_{k-1}}{\filterd(\belief_{k-1},\obs_k)}, \quad
  \filterd(\belief,\obs) = \ones^\p \oprob_y \,\belief.
  $$
  where $\belief_k=[\belief_k(1),\ldots,\belief_k(\statedim)]^\p$ and $\belief_k(i) = \prob(\truestate=i|\obs_{1:k})$.
\item The agent chooses action $\act_k = \policy_{k,\env}(\belief_k)  \in \actionset \cup \{\text{continue}\}$. \\ If $\act_k \in \actionset$, a stopping  cost $\stopcost_{\env,\act}^\p \belief_k$ is incurred and the protocol stops.\\
If $a_k=\text{continue}$, a continue cost $c_k^\p \belief_k$ is incurred,  set $k=k+1$ and go to Step~2. 
\end{steplist}

The above protocol defines a Bayesian stopping agent. In Step 4, $\policy_{k,\env}(\cdot)$ denotes an arbitrary time-dependent policy in environment $\env$.
Define the stopping time $\tau$ as the random variable 
$$ \tau = \inf\{k \geq 0\colon \policy_{k,\env}(\belief_k) \in  \actionset\}
=  \inf\{k \geq 0\colon \policy_{k,\env}(\belief_k) \neq \text{continue}\}
. $$

 We now  define an {\em optimal} agent. Define the policy sequence $\policy_{\env}= (\policy_{0,\env},\ldots,\policy_{\tau,\env})$.

\begin{definition}[Optimal Bayesian Stopping Agent]\label{def:absoptimality} In  each environment $\agent\in\agentset$, an optimal stopping agent uses the  optimal time-dependent policy $\optstoptime_{\env}(\cdot)$ that satisfies
  \begin{align} \label{eq:opt_stop_time}
\policy^*_{\tau,m}(\belief_\tau)  & = \argmin_{a \in \actionset} \stopcost_{\env,a}^\p \belief_\tau, \quad 
                               J(\stopcost_\env,\policy^*_m,\belief_0) =
\inf_{\policy}    \{  G(\stopcost_m,\policy,\belief_0)  + \RIcost(\policy,\belief_0) \}\\
\text{ where }   &   G(\stopcost_m,\policy,\belief_0) =   \E_\policy \big\{ \min_a \stopcost^\p_{\env,\act} \belief_\tau \mid \belief_0 \big\}, \quad \RIcost(\policy,\belief_0) =\E_\policy \Big\{ \sum_{k=0}^{\tau-1} c_k^\p \, \belief_k \mid \belief_0 \Big\}.
      \nn
  \end{align}
 Here,  $\E_{\stoptime}$, parametrized by policy $\stoptime$, is  w.r.t.\  the joint distribution  $\obs_{1:\funcstop}$. For prior $\belief_0$ and policy~$\policy$,  $ G(\stopcost_m,\policy,\belief_0) $ is the expected stopping cost, and  $\RIcost(\policy,\belief_0)$ is the cumulative continue cost.
  \end{definition}

Definition~\ref{def:absoptimality} provides the standard formulation  for the optimal policy in a sequential stopping problem. Since the continue cost $c_k$ is a positive vector, the stopping time $\funcstop$ is finite w.p.1. The optimal policy  decomposes into two steps: Choose whether to continue or stop; and if the decision is to stop, then  choose a specific stopping action from the set $\actionset$.
The optimal stopping policies $\stoptime^*_{\agent},\agent\in\agentset$ that satisfy (\ref{eq:opt_stop_time}) are  obtained by stochastic dynamic programming. We showed in the POMDP book  that  the set of beliefs for which it is optimal to stop is convex in the belief $\belief$.

To summarize, 
an optimal Bayesian stopping agent is parametrized  by the tuple
\begin{equation}
  \label{eq:stop_agent}
\optsearchtuple =  (\dpset,\stateset,\obsset,\actionset,\belief_0,\contcost_k(x), \oprob,  \{
  \stopcost_\dpiter, \policy^*_\dpiter ,\dpiter\in\dpset\}).
\end{equation}
Note that only the stopping cost vector  $s_\env$ depends on the environment $\env$.

\subsubsection{Viewpoint 2. Inverse Reinforcement Learning. Analyst's Model}

The analyst (inverse learner) observes the stop actions of the Bayesian  agent in $\numagents\geq 2$ environments, each with a different stopping cost vector~$\stopcost_\env$. As in~\eqref{eq:dataset_accum}, by observing  infinitely many i.i.d.\ trials
of the true state $\truestate$ and the agent's stopping action $\act$,
 the analyst has the dataset
\begin{equation} \label{eq:IRL_tuple}
\datasetaccum=\{\belief_0 ,p_{\dpiter}(\act|\state) ,\state\in\stateset,\act\in\actionset,\dpiter\in\dpset\}.
\end{equation}
Here  $\actselectagent{\agent}$ is the  conditional probability that the agent  chooses stop action $\act$ at the stopping time, given the true state~$\truestate=\state$ in environment  $\env \in \dpset$.

The analyst does not know the agent's stopping times,  observations or observation likelihoods.



\subsubsection{Main Result. IRL for Bayesian Stopping-Time Problem}

\begin{theorem}[IRL for  Bayesian optimal stopping~\cite{CD15}]\label{thm:NIAS_NIAC}
  Suppose the analyst has the dataset $\datasetaccum$~\eqref{eq:IRL_tuple}   from a  Bayesian stopping agent  in $\numagents\geq 2 $ environments. Then,
 \newline
{\em 1.} {Existence}:
An  optimal Bayesian stopping agent  $\optsearchtuple$~\eqref{eq:stop_agent} rationalizes dataset $\datasetaccum$ if and only if there exists a feasible solution to  the  set of linear inequalities  (in stopping costs)
\begin{equation}\label{eq:BRP_stop}
\BRP\big(\datasetaccum,\{-\hat{\stopcost}_\env,\hb_\env\}_{\env=1}^\numdp\big)  \leq \mathbf{0}, \quad
 \hat{\stopcost}_\dpiter\in\reals_+^{|\stateset|\times|\actionset|},
\, \hb_\env>0.
\end{equation}
The $\BRP$ feasibility test is defined in Algorithm~\vref{alg:dtest}. \\
\noindent	{\em 2.} {Reconstruction of costs}:
Given any feasible solution $\{\hat{\stopcost}_\dpiter,\hb_\env\}_{\dpiter=1}^\numdp$ to $\BRP(\datasetaccum,\cdot)$
\begin{senumerate} 
  \item  The set-valued estimate of the  agent's stopping cost  in environment $\dpiter$ is 
    $\hat{\stopcost}_\dpiter$.
  \item The set-valued estimate of the agent's expected  cumulative continue cost  $\RIcost(\policy,\belief_0)$ is 
\begin{align}
  \hat{\RIcost}(\policy,\belief_0) & = \max_{\dpiter\in\dpset} \Big\{\hb_\dpiter 
+ \sum_{\state,\act}p_{\dpiter}(\state,\act)\,\hat{\stopcost}_{\dpiter}(\state,\act)
                   -       \sum_{\act}\min_{\actiontwo\in\actionset}\sum_{\state} p_\policy(\state,\act)\,\hat{\stopcost}_\dpiter(\state,\actiontwo)\Big\}.  \label{eq:BRP_reconstruct}
\end{align}
Here, the reconstructed $\hat{\RIcost}(\policy,\belief_0)$ is a function of variable $\policy$ which parametrizes the variables  $\{p_\policy(\act|\state), \act\in \actionset, \state\in \statespace\}$.
Also, $p_\policy(x,a) = \belief_0(x)\, p_\policy(a|x)$, where
$\belief_0$ is the initial belief.
\end{senumerate}
\end{theorem}

  Theorem~\ref{thm:NIAS_NIAC}  is identical to  the one-step Bayesian case (Theorem~\ref{thm:BIRL}), except that the stopping agent is  minimizing a cost
  instead of  maximizing a reward.
The reconstruction formula for  the continue cost in~\eqref{eq:BRP_reconstruct} is identical to that of  the information acquisition  cost in~\eqref{eq:BRP_reconstruct0}, with  the policy $\policy$ now playing the role of $\bB$. Also the discussion regarding~\eqref{eq:CD_RAreconstruct} applies.
  
To prove Theorem~\ref{thm:NIAS_NIAC}, we will show that the stopping-time problem is equivalent to the one-step  Bayesian case  and so  Theorem~\ref{thm:BIRL} applies. The main point is that in a stopping time problem, the policy $\policy$ determines the number of observations gathered until stopping time $\tau(\policy)$, but $\policy$  does not affect the observation probabilities $\pdf(y_k|\state)$ for $k< \tau(\policy)$. So the belief at the stopping time  $\tau(\policy)$ can be computed by a single Bayesian step using the vector of observations $\obs_1,\ldots,\obs_{\tau(\policy)}$. The proof below is simply a formalization of this idea.

\begin{proof} Let $\stopset_\policy$ denote the stopping set when using policy $\policy$. This is 
  the  set of beliefs for which policy $\policy$ applies the stop action.
 Let $\filter$ denote the Bayesian update formula.  Define the set
of all observation sequences that yield a belief in $\stopset_\policy$  
when starting with  prior $\belief_0$,   as
$$ \fobspace_{\policy}= \bigcup\{\obs_{1:\tau(\policy)}\colon  \filter(\belief_0,\obs_{1:\tau(\policy)}) \in \stopset_\policy\}. $$
We denote an arbitrary  element of $\fobspace_{\policy}$ as $\fobs$.
Next define the fictitious observation likelihood
$$ \foprob_\policy(\fobs|\state) = 
\prob\big(\obs_{1:\tau(\policy)}\colon \filter(\belief_0,\obs_{1:\tau(\policy)}) \in \stopset_\policy\mid \belief_0=e_\state\big) , \quad \state \in \statespace.
$$
Clearly,  a one-step Bayesian update using  likelihood $ \foprob_\policy(\fobs|\state)$  yields the same stopping belief  as  the  multistep Bayesian update using observation trajectory $\obs_{1:\tau(\policy)}$.



We now  establish the following equivalence with the one-step  Bayesian case:

1.  Consider the  stopping cost $ G(\stopcost_\env,\policy,\belief_0) =   \E_\policy \big\{ \min_a \stopcost^\p_{\env,\act} \belief_\tau \mid \belief_0  \big\}$.
Let $\filter(\belief_0,\fobs) $ denote the Bayesian belief update and $\filterd(\belief_0,\fobs) = \ones^\p \foprob_{\policy}(\fobs) \belief_0= \pdf_\policy(\fobs|\belief_0)$ denote the normalization term, where $\foprob_{\policy}(\fobs) =
\diag\big(\foprob_{\policy}(\fobs|1) ,\ldots,\foprob_{\policy}(\fobs|\statedim) \big)$.
  Then
 \begin{align}
   \label{eq:gmu} G(\stopcost_\env, \policy,\belief_0) &=\E_{\policy}\big\{
\stopcost^\p_{\env,\policy(\belief_{\tau},\tau)}\,\belief_{\tau}\big\} = \E_{\fobs}
\big\{\min_{\act} \stopcost^\p_{\env,\act} \filter(\belief_0,\fobs) \big\} \\ & \hspace{-1.5cm} =
\int_{\fobspace_{\policy}} \filterd(\belief_0,\fobs) \min_{\act \in \actionset}
\stopcost^\p _{m,a}\, \filter(\belief_0,\fobs)\, d\fobs =
\int_{\fobspace_{\policy}} \min_{\act \in \actionset} \stopcost_{\env,\act}^\p\,
\foprob_{\policy}(\fobs) \, \belief_0 \, d\fobs . \nn
 \end{align}
 This
  is equivalent to the negative utility $-R(r_\env,B,\belief_0) $ of the one-step Bayesian agent~\eqrefp{eq:attentionmaximization} with $\foprob_\policy$ replacing $\oprob$.

2.  Next consider
the cumulative continue cost $\RIcost(\policy,\belief_0)  =\E_\policy \big\{ \sum_{k=0}^{\tau-1} c_k^\p \, \belief_k \mid \belief_0  \big\}$.
Define 
$ C_\tau(\policy,x) = \sum_{k=0}^{\tau-1 } c_k^\p \belief_k$ starting with initial belief $ e_x $.
Then
\begin{equation}
  \label{eq:RIexp} \RIcost(\policy,\belief_0) = \E_\policy\big\{ C_\tau(\policy,x)
|\belief_0\big\} = \int_{\fobspace_{\policy}} \sum_{x\in \statespace} C_\tau(\policy,x)\,
\foprob_{\policy}(\fobs|x) \, \belief_0(x) \, d\fobs .
\end{equation}
This is equivalent to the rational inattention cost  $\RIcost(\oprob,\belief_0)$ of the one-step Bayesian agent~\eqrefp{eq:attentionmaximization}  with $\foprob_\policy$ replacing $\oprob$ in the \RHS.

Having established the equivalence with the one-step Bayesian agent, the NIAS and NIAC 
conditions apply as necessary and sufficient conditions, as per Theorem~\ref{thm:BIRL}.
\end{proof}

\section{Example 1. Inverse Sequential Hypothesis Testing}
\index{Bayesian IRL! inverse sequential hypothesis testing} \index{inverse sequential hypothesis testing}
\label{sec:SHT_back}

Suppose the true state $\truestate\in \{1,2\} $ is drawn from prior $\belief_0$. The agent  observes
$\obs_k \sim p(\obs|\truestate)$. But the agent does not know whether the true pdf is
 $p(\obs|\truestate=1)$ or $p(\obs|\truestate=2)$. So the agent deploys a  sequential hypothesis test (SHT).  By accumulating i.i.d.\ observations  $\{\obs_1,\ldots,\obs_k\}$ from the true pdf $p(\obs|\truestate)$ sequentially over time $k$, the aim of SHT  is to decide  whether $\truestate=1$ or $\truestate=2$ by minimizing a combination of the continue (measurement) cost and misclassification cost.
SHT is a special case of the Bayesian stopping-time problem discussed above.

\begin{definition}[SHT in multiple environments]
 SHT in multiple environments $\agentset$ is a special case of $\optstoptuple$~\eqref{eq:stop_agent} where:
\begin{compactitem}
	\item $\stateset=\{1,2\}$, $\mathcal{Y}\subset \mathbb{R}$ , $\actionset=\stateset$.
	\item  $\runcostinst_{\dtime}(\state)=\runcostinst\in\reals^+,~\forall \state\in\stateset$ is the constant continue cost.
	\item $\{\stoptime_{\agent},\agent\in\agentset\}$ are the SHT stopping strategies over $\numagents$ SHT environments.
	\item The stopping cost $\stopcost_\env(\state,\act)$ in environment $\env$
          is the  misclassification cost 
	\begin{equation*}
	    \stopcost_\env(x,\act)  = \begin{cases}
	    \msL_{\agent,1}, & \text{ if } \state=1,\act=2,\\
	    \msL_{\agent,2}, & \text{ if } \state=2,\act=1,\\
	    0, & \text{ if } \state=\act\in\{1,2\}.
	    \end{cases}
	\end{equation*}
	\end{compactitem} 
\end{definition}
The SHT stopping policies of the agent are stationary and
 computed using Bellman's dynamic programming equation. Indeed,
 the SHT policy has the following threshold structure  parametrized by scalars $\alpha_{\agent},\beta_{\agent}\in(0,1)$:
	\begin{align}\stoptime_{\agent} (\belief) = \begin{cases}
	\text{choose action }2, & \mbox{if } 0\leq \belief(2) \leq \beta_{\agent} \\
	\mbox{continue,} & \mbox{if } \beta_{\agent} < \belief(2) \leq \alpha_{\agent} \\
	\text{choose action } 1, & \mbox{if } \alpha_{\agent}<\belief(2) \leq 1.
	\end{cases}
	\label{eqn:opt_policy}
	\end{align}

\remar
The SHT  is  parametrized by $c,\msL_1,\msL_2$. We set the continue cost as $c=1$ 
without loss of generality since the optimal policy is unaffected. So the expected cumulative continue cost is the expected stopping time:  $\RIcost(\policy,\belief_0)= \E_\policy\big\{\sum_{k=0}^{\tau-1}\one^\p \belief_k| \belief_0\big\} = \E_\policy\{\tau\}$, where we suppress~$\belief_0$. 

\subsubsection{IRL for Inverse SHT}

Suppose the analyst (inverse learner)  observes the actions of  the  stopping agent in $\numagents$   environments. We assume the following about the analyst:
\begin{compactenum}
\item \label{asmp:SHT1} The analyst has the dataset 
\begin{equation}
\datainf(\SHT)=\{\belief_0 , \{p_{\dpiter}(\act|\state) ,\state\in\stateset,\act\in\actionset,\dpiter\in\dpset\},\{\sumruncostagent{\agent},\agent\in\agentset\}\}.
    \label{eqn:dataset_SHT}
\end{equation}
Here $\sumruncostagent{\agent}=\mathbb{E}_{\stoptime_{\agent}}\{\funcstop\}$ is the expected continue cost incurred by the Bayesian agent in the  $\agent$th environment, and is known to the analyst.
\item \label{asmp:SHT2} The stopping strategies $\{\stoptime_{\agent},\agent\in\agentset\}$ are stationary strategies characterized by the threshold structure in (\ref{eqn:opt_policy}).
\end{compactenum}

 Assumption \ref{asmp:SHT1} specifies additional information the analyst has for performing IRL for SHT by observing the agent decisions.
 Assumption \ref{asmp:SHT2} specifies the  partial information that  the analyst has about the stopping strategies chosen by the agent and its observation likelihood. Since the optimal stopping strategy has a threshold structure, the analyst (IRL) only needs to compare the expected cost incurred by  threshold policies to check for optimality.

\begin{theorem}[IRL for inverse SHT] \label{thrm:classic_SHT}
  Consider the analyst with dataset $\datainf(\SHT)$ (\ref{eqn:dataset_SHT}) from a Bayesian agent taking actions in $\numagents$ environments.  Then
Theorem~\ref{thm:NIAS_NIAC} holds. (Note  $C_m = \E_{\stoptime_{\agent}}\{\funcstop\}$  is known to the analyst, and hence is not a free variable.)\\
The set-valued IRL estimates of the SHT misclassification costs are
$$
\hat{\msL}_{1,\agent} = \stopcost_{\agent}(1,2), \; \hat{\msL}_{2,\agent}=\stopcost_{\agent}(2,1), \quad \forall \agent\in\agentset,
$$
	where $\{\stopcost_{\agent}(x,a),\agent\in\agentset\}$ is any feasible solution to the NIAS and NIAC inequalities.
      \end{theorem}


\section{Example 2. Inverse Optimal Search} \index{Bayesian IRL! inverse search} \index{inverse search} \index{optimal search! inverse (IRL)}

This section discusses  how \textit{inverse} optimal search for a nonmoving target can be achieved using Bayesian  IRL.   The aim of inverse  search is to identify if the search actions of the agent are optimal and if so, estimate the search costs.

\subsubsection{Optimal Search for Nonmoving Target}
The classical  search problem has the following model:
$\statespace = \{1,2,\ldots,\statedim\}$ are the search locations. The observations are $\obs_k \in \{F \text{(found)},\bar{F} \text{(not found)}\}$,  and the continue actions $a$ are which location to search. So $a \in \actionset = \statespace$, i.e., the action space and state space are identical.
The overlook probabilities for the search agent are $$\overlook(\aaction)= \prob(\obs= \bar{F}|\text{target is in the cell $\aaction$}), \quad \aaction  \in \actionset. $$

Let $\searchcostsymbol_{\agent}(\aaction)$ denote the cost of searching cell $\aaction$
in environment $\agent.$

Finally let  $\{\policy_{\env},\env\in\envspace\}$ denote  the optimal search strategies of the Bayesian agent over the  $\numagents $ environments, when the agent operates sequentially on a sequence of observations $\obs_1,\obs_2,\ldots$.
as discussed in the POMDP book.

Given the posterior belief $\belief$, 
the optimal stationary search policy of the agent in environment $\env$  is 
$$     \optpolicy_\env(\belief)  = \argmax_{\aaction \in \actionset}  \frac{\belief(\aaction)\, \big(1-\overlook(\aaction)\big)} {\searchcostsymbol_{\agent}(\aaction)}. $$
Since the expected cumulative cost depends only on the search costs (for constant overlook probabilities), we set $c_m(1) = 1$ for each  $\env$ without loss of generality.

\subsubsection{IRL for Inverse Search}
We now discuss inverse search to estimate the search costs $\csearch_\env(\aaction)$, $\aaction \in \actionset$, $\env \in \envspace$.
The framework is the opposite of  Theorem~\ref{thm:NIAS_NIAC} where we considered multiple stopping actions and a single continue action. In optimal search there are multiple continue actions,  namely, which of the 
      $\numstates$ locations to search at each time, and a single stopping action (when the target is found).

Suppose an analyst  (inverse learner) observes the decisions of a Bayesian search agent over $\numagents$ search environments.
We  assume that the analyst knows the dataset
\begin{equation}\label{eqn:dataset_search}
  \datainfsearch =   \big\{\belief_0,\, \{g_{\env}(\act|\state),\,\agent\in\agentset\}\big\}.
\end{equation}
Here, $g_{\env}(\act|\state)$  is the expected number of  times the agent searches location $\aaction$ when the true state of the target is $\truestate=\state$ in environment $\agent$ (recall the analyst knows $\truestate$):
\begin{equation}
    g_\env(\aaction|\state) = \E_{\stoptime_{\agent}}\Big\{\sum_{k=1}^{\funcstop} 
    I\{\stoptime_{\agent}(\belief_{k}) = \aaction\} |\truestate=\state
    \Big\}.
\end{equation}
Assume that there are  $\numagents\geq 2$ environments with distinct search costs.
For environment $\env$, we express the expected cumulative search  cost in  terms of the variable $g(\aaction|\state) $ as
 \begin{align*}
   J(c_\env,\policy,\belief_0)& = \E_{\policy}\Big\{\sum_{k=1}^{\funcstop}
                                 \csearch_\env(\policy(\belief_{k}))\Big\}= \E_{\policy}\Big\{\sum_{\aaction\in\actionset}\csearch_\env(\aaction) \sum_{k=1}^{\funcstop} I\{\policy(\belief_{k})=\aaction\}\Big\}\\
     & \hspace{-1cm} = \sum_{\aaction\in\actionset}\csearch_\env(\aaction)\sum_{\state\in\statespace}\belief_0(\state)\,\E_{\policy}\Big\{\sum_{k=1}^{\funcstop} I\{\policy(\belief_{k})=\aaction\}|\state\Big\} = \sum_{\state\in\statespace,\aaction\in\actionset}\belief_0(\state)\,g(\aaction|\state)\, \csearch_\env(\aaction).
 \end{align*}

  Since there is only one stop action,   NIAS is degenerate. So inverse search is specified by  NIAC.
  Notice that the expected cumulative cost $J$ has the same form as the expected reward $R$ in~\eqref{eq:niac_full} with $p_\env(\act|\state)$ replaced by $g_\env(\act|\state)$, and $\max$ replaced by $\min$.  The search  cost $\csearch_\env(a)$ is independent of state $x$. 
Denoting  $a = \argmin_{\bar{a}} c_m(\bar{a})$, then $J(\hat{\cost}_m,\policy_l,\belief_0)$ in NIAC~\eqref{eq:niac_full}~is
 $$J(\hat{\cost}_m,\policy_l,\belief_0) = \sum_{\act} \min_{\actiontwo}\sum_{\state}g_{\dpitertwo}(\act|\state)\,\hat{\csearch}_\env(\actiontwo)\, \belief_0(x)
 = \sum_{\act} \sum_{\state} g_{\dpitertwo}(\act|\state)\, \hat{\csearch}_\env(\aaction)\, \belief_0(x).
$$
Also, comparing with ~\eqref{eq:attentionmaximization}, we see $\RIcost = 0$; so $\hb_l=\hb_m=0$.
This yields the following result.

\begin{theorem}[IRL for inverse Bayesian search]\label{thrm:Search}
  Consider the analyst  with dataset $\datainfsearch$ (\ref{eqn:dataset_search}) obtained from a search agent acting in multiple environments. 
  Then the search agent is optimal iff  there exists a feasible solution to the following linear (in search costs) inequalities:
\begin{align}
  &\text{Find }\hat{\csearch}_\env(\aaction) \in\reals_+,\;\hat{\searchcostsymbol}_{\agent}(1)=1 \quad\text{s.t.} 
        \nonumber\\
  & 
    \sum_{\state\in\stateset}\sum_{\aaction \in \actionset} \belief_0(\state)\,\big(g_\env(\act|\state)-g_{\agenttwo}(\act|\state)\big)\, \hat{\csearch}_\env(\aaction) \leq 0 \qquad \forall\agent,\agenttwo\in\agentset,~\agent\neq \agenttwo.  \label{eqn:NIAC_dag_def}
  \end{align}
 The set-valued IRL estimate of the agent's search costs  satisfy~\eqref{eqn:NIAC_dag_def}.
\end{theorem}

\begin{rem}
\item The inverse learner only knows  the average number of times the agent searches a particular location. It  does not know the stopping time or the order in which the agent searches the locations.

\item  
  As discussed earlier,  the above set-valued estimates are not statistical estimates. They  satisfy the necessary and sufficient conditions  of an optimal Bayesian agent  (NIAC in this case).

\item Optimal search for a nonmoving target is a multi-armed bandit problem.
 So Theorem~\ref{thrm:Search} can be viewed as IRL for a multi-armed bandit. 
\end{rem}

\section{Example 3. Inverse Quickest Detection}
\index{Bayesian IRL! inverse quickest detection}

This section discusses Bayesian IRL for  inverse quickest detection. 
We studied classical quickest detection in the POMDP book  as   a stopping-time POMDP. In this section,  we 
assume that an underlying discrete-time state  $\state$ jump changes at a random time $\tau^0\in\Gamma=\{1,2,\ldots,\horizon\}$.
The change time $\tau^0$ is not known to  the detector. But the detector knows the
 prior distribution $\qdp$  of the jump time
 on $\Gamma$.

 In quickest detection,  we formulate the
sequential detection policy $\policy$ that  maps the  posterior belief  at each time  to the actions stop (declare change) or continue.
The declared stopping time by the quickest detector is 
 $\tau(\policy)\in\Gamma$. This is the decision maker's estimate of the true jump time $\tau^0$.

 In the $m$th environment, $m \in \{1,\ldots,M\}$,
the quickest detector uses the optimal policy 
 $$ \policy_m^* \in \argmin_\policy  \E_{\qdp,\policy}\big\{ d\, |\tau(\policy) - \tau^0|^+  + f_m\, I(\tau(\policy) < \tau^0)  \big\} $$
 where $f_m\geq 0$ is the false alarm penalty.
 Without loss of generality, we can divide through by $d>0$ and set the delay penalty $d=1$, since the optimal policy is unaffected.

Next, we  consider the analyst's (IRL) point of view.
 By choosing the  change time $\tau^0\in \Gamma$ as the hidden state, we can use the static Bayesian IRL formulation of \secn\ref{sec:BIRL}.
The analyst knows prior  $\qdp$, and observes $\tau^0$ along with the detector's actions (and therefore $\tau$)  in infinite i.i.d.\ trials in each environment~$m$.
 Thus the analyst knows  the action selection  probabilities
 $p_m(\tau|\tau^0)$. 
 Given $\qdp$ and $p_m(\tau|\tau^0)$,
the analyst first estimates the rational inattention cost
$\RIcost(\policy_m,\qdp) =   \E_{\qdp,\policy_m}\big\{  |\tau(\policy) - \tau^0|^+ \big\}$ for each environment $m$. Then to estimate $f_m$, the analyst solves NIAC:
$$ \RIcost(\policy_m,\qdp) + \E_{\qdp,p_m}\big\{ I(\tau < \tau^0)\, f_m\big\} \leq
\RIcost(\policy_n,\qdp) + \E_{\qdp,p_n}\big\{ I(\tau < \tau^0)\, f_m\big\},  \;  m \neq n, m \in \{1,\ldots,M\}$$
where $f_m \geq 0$.
Since there is only a single stop action,   NIAS is degenerate.
The above  result  is  somewhat simplistic since $f_m$ is a constant.
More generally, if the detector uses a false alarm penalty $f_m(\tau^0,\tau(\policy))$ to model a risk averse detector, the above IRL procedure still holds.

\section{Discrete Choice Random Utility Models} \index{random utility model}
\label{sec:disc_choice}

Thus far, we have
 discussed Bayesian IRL for a {\em single} agent, where the analyst reconstructs a \textit{set-valued}  estimate of the  utility.
We close our discussion of IRL by describing a parametric model for utility estimation, in which each agent in a {\em population} draws random actions from a parametrized logistic  model. The analyst then obtains  a \textit{point-valued} estimate of the  utility, by computing  the  maximum likelihood estimate (MLE) of the logistic parameter of the utility.

Consider a population of agents. If agent $k$ chooses  action $a$, then it receives utility 
\begin{equation}
  \label{eq:um_disc}
  U_k(a) = \regc^\p _{k,a}\, \upar  + \unoise_k(a) , \quad a \in \actionset = \{1,\ldots,A\}.
\end{equation}
The attribute vector $\regc_{k,a}$ is  known to the analyst, while $\upar$ is an unknown parameter vector to be estimated. The  analyst assumes $\unoise_k(a)$ is unobserved  Gumbel distributed   noise that is 
i.i.d.\ w.r.t.\ $k$ and $a$,
with
pdf and cdf, respectively 
\begin{equation}
  \label{eq:gumbel}
 \pdf_G(\unoise)  = \exp(-\unoise) \, \exp(-\exp(-\unoise)) ,
\quad \cdf_G(\unoise) = \exp(-\exp(-\unoise)),
\quad \unoise \in \reals.
\end{equation}
This Gumbel assumption yields a closed-form expression for the choice probabilities below.

{\em Aim}. The analyst observes the  fraction of the population that chooses each action  $a \in \actionset$.  The aim of the analyst is to estimate the parameter vector $\upar$ in the utility function~\eqref{eq:um_disc}. For simplicity, we assume that $\upar$ is a shared parameter that does not vary across agents. 

For each action $a^*\in \actionset$,
define the  {\em choice probability} $\prob( a^*) $ as the probability that an agent sampled uniformly from the population chooses  action $a^*$ to maximize its utility.
Then
  \begin{align}
    \prob( a^*) &=  \prob\big(  U_k(a^*) > U_k(a)   \text{ for all } a \in
                  \actionset - \{a^*\} \big)   \label{eq:choice_eq} \\
    &= \prob\big(  \unoise_k(a) <  \unoise_k(a^*) +  (\regc_{k,a^*} - \regc_{k,a})^\p \upar  \text{ for all } a \in
      \actionset - \{a^*\} \big)  \nn \\
    &= \int_\reals \prod_{ a \in
      \actionset - \{a^*\}} \cdf_G\big( \unoise + (\regc_{k,a^*} - \regc_{k,a})^\p \upar \big)\, \pdf_G(\unoise) d\unoise.   \label{eq:choice_eq_last}
  \end{align}
Equation  \eqref{eq:choice_eq} is called the \textit{random utility model} explanation of observed choices.
  A random utility model specifies the probability that the decision maker will choose
  action  $a^*$ over $a$ for any pair of actions  $a,a^*$. The noise term $\unoise$ introduces flexibility in the model by allowing the decision maker to respond differently when the same   pair of actions is repeatedly queried.

  Evaluating~\eqref{eq:choice_eq_last} with  Gumbel noise~\eqref{eq:gumbel} yields the following
 Bradley--Terry model:
%
\begin{theorem}
  Consider a population of independent agents, each  with utility~\eqref{eq:um_disc} and
  Gumbel noise~\eqref{eq:gumbel}.
  Then  the choice probabilities~\eqref{eq:choice_eq_last}  are  given by the logistic  distribution  
  \begin{equation}
    \label{eq:choice_prob}
  \prob(a^*) = \frac{\exp(\regc^\p_{k,a^*}\, \upar) } { \sum_{a\in \actionset} \exp(\regc^\p_{k,a}\,\upar) } , \quad a^* \in \actionset.
\end{equation}
\end{theorem}

The proof is omitted. The main outcome
is that given  empirical measurements of the choice probabilities,
the analyst can  estimate $\upar$ using a logistic regression on the logistic (softmax) model~\eqref{eq:choice_prob}. We studied logistic regression  in the POMDP book. \index{logistic! regression}
The gradient algorithm  can be used to compute  the MLE of $\upar$.
Finally, given  attribute vectors $\regc_{k,a}, a \in \actionset$ of agent $k$ and MLE of $\upar$, the analyst  can also predict the agent's  future choice behavior, e.g., predict consumer preferences for different vehicle types (sedan, SUV, or electric car), based on fuel efficiency, price and brand. The above Bradley--Terry random utility model is  also used in 
reinforcement learning with human feedback to fine-tune large language models including ChatGPT
\cite{ZSW20}.

\begin{rem}
\item   For binary choice models $\actionset=\{1,2\}$, we can directly obtain \eqref{eq:choice_prob} as follows:
Since $ \unoise_k(a^*) $ and $ \unoise_k(a) $ are i.i.d.\ Gumbel, hence
$   \unoise_k(a^*) - \unoise_k(a) $ has a logistic distribution with cdf
$$ \prob\big( \unoise_k(a^*) - \unoise_k(a)  \leq \onoise\big) =  \frac{\exp(\onoise)}{1 + \exp(\onoise)}, \quad \onoise \in \reals.$$
\item  Since $\sum_{a^*} \prob(a^*) = 1$, one component in~\eqref{eq:choice_prob} is redundant (completely determined) given the remaining $A-1$ components. We can  eliminate this  redundancy by pivoting:  Using $\tilde{\regc}_{k,a} = \regc_{k,a} - \regc_{k,1}$ for all $a$ including $a^*$,  we see that~\eqref{eq:choice_prob} remains unchanged.

\item If  $\unoise$ is Gaussian noise, then~\eqref{eq:choice_eq} 
  yields  the Thurstone--Mosteller random utility model.
  
\end{rem}

\section{\pwe} \label{sec:pwe_bayesian_irl}

The formalism used in this chapter for IRL is Bayesian revealed preferences from  microeconomics~\cite{CM15,CD15,CDL19}. Discrete choice models were pioneered by McFadden, see  \cite{Mcf01}
and references therein. \cite{MM15} incorporates  rational inattention in discrete choice.

{\bf IRL for explaining YouTube User Engagement}.
This chapter is based on \cite{HKP20} and \cite{PK23}. In these papers, Bayesian IRL was  applied to
 massive YouTube multimedia datasets to show that the commenting behavior of  user groups  is consistent with rationally inattentive utility maximization; these  yield remarkably accurate predictive performance.
 
{\bf IRL for Detecting Glass Ceiling Effect}. The glass ceiling effect refers to the barrier that keeps certain
groups from rising to influential positions, regardless of their qualifications. In a social network
context, it can be shown that  preferential attachment and homophily leads to the
glass ceiling effect \cite{NAI22}. At a deeper level, revealed preferences can be used to determine the utility functions that result in specific types of preferential attachment and homophily. 
 
{\bf Detecting a Bayesian Detector}.
 In the chapter, we assumed that the analyst knows the
 conditional  state-action  probabilities  $p_m(a|x)$ in each environment $\env$ by observing the agent infinitely many times.
We now briefly discuss the finite sample case.  To be specific, consider the inverse SHT problem. 
The assumption was that  the analyst 
 knows the conditional  probabilities  $p_m(a|x)$ and expected stopping time $\E\{\tau_m\}$ over the environments  $m =1,\ldots, \numagents$.
Suppose instead, the analyst  only has  noisy estimates  $\hat{p}_m(a|x)$ and $\hat{\E}\{\tau_m\}$ due to observing a finite sample size $n$ of the dataset, denoted as  $\dataset_n$. Then violation of the NIAS and NIAC inequalities could be either due to noisy estimates  or absence of Bayes-optimality.

How  to construct  an {\textit{IRL  detector}} to  detect  if the decisions of the SHT  detector  are consistent with Bayes-optimality? (Put simply, how to detect the presence of an optimal detector?)
What is the minimal sample size  so that the Type-I error probabilities
of the IRL detector are within a specified bound? Let $\IRLoutput(\datafin) = \emptyset$ denote the case where NIAS and NIAC are not feasible given $\datafin$.
Then the Type-I error of the detector is
$ P(\IRLoutput(\datafin) =  \emptyset | \text{Bayesian detector is optimal}) $.
\cite{PK23} uses the  Dvoretzky--Kiefer--Wolfowitz inequality  to bound  $|p_m(a|x) - \hat{p}_m(a|x)|$,  Hoeffding's inequality to bound  $| \E\{\tau_m\} - \hat{\E}\{\tau_m\}|$, and then the union bound to compute  sample complexity bounds for the Type-I error of the IRL detector. Thus,  we can solve intriguing
problems such as  detecting the presence of an optimal detector.

{\bf  Interpretable Deep Learning}.
 \cite{PKJ24} shows how Bayesian IRL can be used to interpret  deep convolutional neural networks   as  rationally inattentive 
 Bayesian utility maximizers. It is shown empirically that in many cases,  the deep classifier's response satisfies the necessary and sufficient conditions of NIAS and NIAC.   Such interpretable deep image classification   can be viewed  as system identification of a trained neural network. The information acquisition cost captures the ``learning'' cost incurred during the training. The reconstructed utilities computed using  IRL  provide insight into
 the neural network decision-making process.

{\bf Self-Attention in Large Language Models and Rational Inattention}. Rational inattention, discussed in the chapter, examines the impact of limited cognitive resources on decision making. In large language models (LLMs), self-attention mechanisms enable the model to focus on relevant parts of the input, ignoring less important details. Self-attention \cite{Vas17} enables each token (e.g., word) to assess the relevance of other tokens, optimizing attention to capture context across the entire input sequence. This parallels  rational inattention, where attention is optimized in decision making. Therefore, the Bayesian IRL methods proposed can be used to test if LLMs are rationally inattentive Bayesian decision makers; see \cite{JK25}.

{\bf Inverse Bayesian Contextual Bandits}. The Bayesian IRL framework  can be used to  perform IRL for partially observed regularized contextual bandits.  Optimal Bayesian stopping is  an instance of a  {\em partially observed regularized contextual Bayesian bandit problem}; {\em contextual}~\cite{AG13} since the agent faces multiple ground truths $\state$ (context), {\em partially observed}  since the agent observes  noisy measurements of the underlying context $\state$, {\em Bayesian}~\cite{HKZ22} since the agent minimizes its expected cumulative cost per context averaged over all contexts sampled from  prior  $\belief_0$,
and {\em regularized}~\cite{FBP19} since the agent minimizes the sum of expected stopping cost and a regularization term, namely, the expected continue cost (rational inattention cost).

{\bf Adversary Engagement and Honeypots}. Consider  IRL for stopping-time problems discussed in
\secn \ref{sec:irl-bayes-stop}. The analyst could implement  a honeypot: intentionally  design the environments to  prolong the stopping time (engagement) of the adversary agent, wasting its time while
 extracting valuable insights about its behavior. This  has applications in cyber defense systems.

 {\bf IRL with Partial Sensor Information}.  
 The chapter assumed that the   analyst does not know the agent's observation space $\obspace$, observation samples $\obs$, and observation likelihood $\oprob(\env)$. What if the analyst knows the set of possible  sensors used by the adversary, but not the specific sensor used?  Suppose $\oprob(m) \in \mathcal{B}(m) = \{\oprob^{(1)}(m), \ldots,  \oprob^{(L)}(m)\}$ and the analyst knows the set $\mathcal{B}(m)$ of sensors in each environment~$\env$. How can this additional information be used in IRL? 

 The IRL procedure is modified as follows: For each~$\env$, check if  a stochastic matrix $Q^{(l)}(\env)$ exists with $ \{0,1\}$ elements, such that the action probabilities factorize as $\bB(m) = B^{(l)}(m) \,Q^{(l)}(\env)$ for each~$l$. Then for all such $l \in \{1,\ldots,L\}$ where  this factorization holds,  modify NIAS and NIAC  in Algorithm~\ref{alg:dtest} to use $p^{(l)}_m(x,y) =  \belief_0(x)\, p^{(l)}_m(y|x) $  instead of $p_m(x,a)$.
(We check for a $ \{0,1\}$ matrix $Q^{(l)}(\env)$ since the optimal action is  a nonrandomized function of the observation.)

{\bf Inverse Stochastic Gradient Algorithms and Adaptive IRL}.
\index{contextual bandits! inverse}   \index{multi-armed bandit! contextual! inverse}
The IRL methods discussed thus far  are offline. They  assume that the forward learner has already converged to its optimal strategy, before the IRL algorithm is deployed. However, in adversarial signal
processing applications, it is important to devise \textit{adaptive} IRL algorithms
that operate in the transient
phase while the decision maker (forward RL algorithm)  is still learning to optimize its strategy. \index{inverse stochastic gradient} \index{inverse reinforcement learning (IRL)! adaptive}

Given  real-time estimates from a stochastic gradient algorithm (forward RL algorithm) that optimizes
expected reward $R$, the aim of   adaptive IRL is to construct
  a stochastic
  gradient algorithm to track and learn  $R$ in real time.  
Adaptive  IRL can be viewed as an \textit{inverse stochastic gradient algorithm}.
 To track $R$ in real time, adaptive IRL  
must operate with  a constant step size.  Also, adaptive IRL  must  be passive: The inverse learner cannot dictate where  the forward leaner  should evaluate its gradients.
This requires 
 passive stochastic gradient algorithms that handle  both  misspecified and noisy gradient estimates.
 In the next chapter,  passive Langevin stochastic gradient algorithms are proposed to achieve this adaptive IRL task.

\subsubsection{IRL Literature}

The IRL framework in Part V of the book was presented through the lens of revealed preferences from microeconomics.
%
%
To give additional context, we now give  a brief literature review of IRL.

{\em (a) IRL in fully observed environments:} 
The linear feasibility approach for IRL in stopping-time problems in this chapter generalizes \cite[Theorem 3]{NR00}.  Since the set of policies for an MDP is finite, \cite[Theorem 3]{NR00} comprises a finite set of linear inequalities. In comparison, the set of policies for a partially observed MDP (POMDP) is infinite (continuum). From the feasible set of rewards, \cite{NR00,RAT06} choose the max-margin reward, i.e., the reward that maximizes the regularized sum of differences between the performance of the observed policy and all other policies.  \cite{PK23} computes a regularized max-margin estimate of costs for inverse SHT.
\cite{AN04} achieve IRL by devising iterative algorithms for estimating the agent's reward.

\cite{ZMB08} use  maximum entropy for IRL when the agent's policy is subject to  Shannon mutual information regularization. The optimal policy turns out to be softmax in terms of the Q-function of the MDP. \cite{JSB20} extend \cite{ZMB08} to a more general regularization setup, 
e.g.,  Tsallis entropy~\cite{LKL20} that generalizes Shannon entropy.

\cite{HGW16} append the IRL task with simultaneous learning of model dynamics, specifically, the agent's transition kernel. 
This chapter differs from \cite{HGW16} in that we operate in the nonparametric partially observed setting  where the observation likelihood of the agent is unknown.

\cite{LV12} generalize IRL to continuous space processes and circumvent the problem of finding the optimal policy for candidate reward functions.  \cite{FLL17} uses deep neural networks for IRL to estimate agent rewards. Building on  \cite{Rus94}, \cite{CCS21} study identifiability of MDPs in IRL.

{\em (b) IRL in partially observed environments:} 
 \cite{CK11,MK12} construct IRL in a POMDP setting.
 \cite{MK12} extends Bayesian IRL for MDPs to POMDPs. In analogy to Bayesian IRL, the aim is to compute the posterior distribution of reward functions given an observation dataset.  In \cite{CK11}, the inverse learner first checks if the agent chooses the optimal action given a particular posterior belief, for {\em finitely many beliefs} aggregated from the observed trajectories of belief--action pairs. This is analogous to the NIAS condition in Theorem~\ref{thrm:classic_SHT}, where we check if the agent's terminal action is optimal given its terminal belief. 
 \cite{CK11} develops IRL methods for POMDPs with no assumption on problem structure. This chapter considers a subset of POMDPs, namely, Bayesian stopping-time problems, and the  IRL algorithms {\em do not} require knowledge of the observation likelihood of the decision maker, nor require solving a POMDP. 

{\em (c) Inverse rational control (IRC):} IRC~\cite{KDS20} is a closely related field to IRL in partially observed environments. IRC models suboptimality in decision makers as a misspecified reward function and aims to estimate this reward.  IRC  comprises two subtasks: 
First, the inverse learner constructs a map from a continuous space of reward functions parametrized by $\theta$ to the reward's optimal policy. Second, based on a finite observation dataset $\dataset$, the underlying hyperparameter $\theta$ is estimated as the maximum likelihood estimate $\argmax_{\theta} \prob(\dataset|\theta)$.
\index{Bayesian IRL|)}

{\em (d)  Role of  deep neural networks (DNNs):}  Deep autoencoders are used as a preprocessing step for IRL to map a  dataset to a lower-dimensional feature dataset  \cite{HKP20}. 
In Deep IRL, DNNs are  used as   functional  approximators to learn a utility function~\cite{WOP15,FDS20}.

\newpage

  \begin{subappendices}

  \section{Inverse Optimal Filtering}
\label{sec:inverse_filtering}
  \index{inverse filtering|(} \index{filter! inverse}
It is apt to conclude our discussion of Bayesian IRL with inverse filtering.
Consider an adversarial  signal processing problem involving ``us'' and an ``adversary''. The adversary observes our  state in noise, updates its posterior distribution of our state and then chooses an action based on this posterior. Given knowledge of ``our'' state and the sequence of  the adversary's actions observed in noise, we consider:  How can we estimate the  adversary's posterior distribution?
That is, we wish to estimate the adversary's estimate of us. 

Such problems arise in adversarial signal processing  within counter-autonomous systems. The adversary operates  an autonomous sensing system;
given measurements of its actions, we aim to estimate the adversary’s sensor’s capabilities and predict its
future actions (thus enabling defensive measures). For a detailed exposition of inverse filtering,
  see \cite{KR19}.


   \begin{figure}[h] \centering
          {\resizebox{9cm}{!}{
              \begin{tikzpicture}[node distance = 1cm, auto,every node/.append style={font=\Large}]
                \tikzset{
    block/.style={rectangle, draw, line width=0.5mm, black, text width=6em, text centered,
                 minimum height=2em},
               line/.style={draw, -latex}}
                \tikzset{
    block2/.style={rectangle, draw, line width=0.5mm, black, text width=5em, text centered,
                 minimum height=2em},
               line/.style={draw, -latex}}
    \node [block2] (BLOCK1) {Sensor};
    \node [block2, below of=BLOCK1,right of=BLOCK1,node distance=1.5cm] (BLOCK2) {Decision \\ Maker};
    \node [block, below of=BLOCK1,left of=BLOCK1,node distance=1.5cm,xshift=-1cm] (BLOCK3) {Tracker $\filter(\belief_{k-1},\obs_k)$};

    \draw[Latex-] (BLOCK1) -| node[left,pos=0.8]{$\lact_k$}  (BLOCK2)  ;
    \draw[-Latex] (BLOCK1.west) -|   node[left,pos=0.6]{$\obs_k\sim \oprob_{\state_k,\obs}$} (BLOCK3);

    \draw[-Latex](BLOCK3) --  node[above]{$\belief_k$} (BLOCK2);

    \node[draw=none,fill=none] at (6,-1.5) (drone) {\includegraphics[bb=0 0 0 0,scale=0.07]{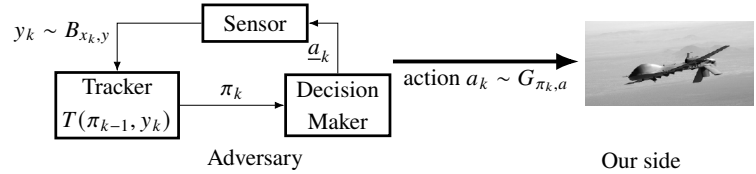}};
      \draw[Latex-,line width=2pt] ([yshift=0.8cm]drone.west)   --   node[below]{action $\act_k\sim \aprob_{\belief_k,\act}$} (2.5,-0.7);
    \node at (7,-2.5) {Our side};
    \node at (0,-2.5) {Adversary};

  \end{tikzpicture}} }
\caption{Schematic of adversarial inference problem. Our side is a drone/UAV or electromagnetic signal that probes the  adversary's multifunction radar system.}
\label{fig:inverse_filter}
\end{figure}

The inverse filtering problem involves two players:  ``us'' and an ``adversary'' as  illustrated in Figure~\ref{fig:inverse_filter}.  With $k=1,2,\ldots$ denoting discrete time, the model has
the following dynamics:
\begin{equation}
\begin{split}
    \state_k &\sim  \tp_{\state_{k-1}, \state} = \pdf(\state | \state_{k-1}), \quad \state_0 \sim \belief_0 \\
    \obs_k  &\sim \oprob_{\state_k \obs} = \pdf(y | x_k)\\
    \belief_k &= \filter(\belief_{k-1}, \obs_k)\\
    \act_k &\sim \aprob_{\belief_k,\act} = \pdf(\act | \belief_k).
  \end{split} \label{eq:model}
\end{equation}
\begin{itemize}
\item $\state_k\in \statespace$ is our Markov state on state space $\statespace$  with  transition kernel $\tp_{\state_{k-1}, \state}$ and prior $\belief_0$.
    \item $\obs_k\in \obspace$ is the adversary's noisy observation of our state $\state_k$ with observation likelihoods $\oprob_{xy}$. Here $\obspace$ denotes the observation space.
    \item $\belief_k = \pdf(x_k| \obs_{1:k})$ is the adversary's belief (posterior)  of our state $\state_k$ where $\obs_{1:k}$ denotes the sequence  $\obs_1,\ldots,\obs_k$. The operator $T$ in (\ref{eq:model}) is the  Bayesian filter  
       \beq  \filter(\belief,\obs) (\state) = \frac{
    \oprob_{\state \obs} \int_\statespace 
    \tp_{\zeta  \state}\, \belief(\zeta) \,d\zeta}
  {\int_\statespace  \oprob_{\state \obs} \int_\statespace 
    \tp_{\zeta  \state}\, \belief(\zeta)\, d\zeta d\state} ,  \qquad \state \in \statespace. 
  \label{eq:belief}
\eeq
Let $\Belief$ denote the  space of all such beliefs. When the state space
$\statespace$ is Euclidean space, then $\Belief$ is a function space comprising the space of density functions; if $\statespace$ is finite, then $\Belief$ is  the unit $
(\statedim-1)$-dimensional simplex of probability vectors.
    \item $\act_k \in \actspace$ denotes our measurement of the adversary's action based on its  belief $\belief_k$ where $\actspace$ denotes the action space. The adversary chooses an action $\lact_k$ as a deterministic function of $\belief_k$ and we observe $\lact_k$ in noise as $\act_k$. We encode this as $\aprob_{\belief_k,\act_k}$,  the conditional probability (or density if $\actspace$ is continuum) of observing  action $\act_k$ given the adversary's belief~$\belief_k$.
    \end{itemize}

    An identical formulation    
    applies to interactive learning:  $\state_k$ is  the  material being taught, 
    $\belief_k$ models the learner's knowledge of the material, and $\act_k$ is the response to an exam administered by the instructor. The instructor aims to  estimate the learner's knowledge  (posterior $\belief_k$), to optimize how to present the material.
    
    \subsubsection{Inverse Optimal Filter for Estimating Belief}
    
Given the model  (\ref{eq:model}), we now derive a filtering recursion for the posterior of the adversary's belief given knowledge  of our state sequence and recorded actions. Define
$$\post_{k}(\belief_k) = \pdf(\belief_k | \act_{1:k},\state_{0:k}).$$
The posterior $\post_k(\cdot)$ is a {\em random measure}, as it represents our conditional distribution of the adversary's belief $\belief_k$ given its actions $a_{1:k}$ and our  state sequence $x_{0:k}$. In the theorem below, we assume that the adversary's observations $\obs_k$ are drawn from a known pdf or pmf $\oprob_{\state_k,\obs}$.
\begin{theorem} \label{thm:post}
The posterior  $\post_k$ satisfies the following filtering recursion:
\beq
  \post_{k+1}(\belief) = \frac{\aprob_{\belief,\act_{k+1}}
    \,  \int_{\fBelief} \oprob_{\state_{k+1}, \obs_{\belief_k,\belief}}\, \post_k(\belief_k) d\belief_k}
  {\int_{\fBelief} \aprob_{\bbelief,\act_{k+1}}
    \,  \int_{\fBelief} \oprob_{\state_{k+1}, \obs_{\belief_k,\bbelief}}\, \post_k(\belief_k) \,d\belief_k \, d\bbelief}
  \label{eq:post}
  \eeq
   initialized by prior  $\post_0 = \pdf(\belief_0)$.
  Here $\obs_{\belief_k,\belief}$ is the observation such that $ \belief = \filter(\belief_k,\obs)$ where $\filter$ is the adversary's  filter (\ref{eq:belief}). The conditional mean estimate
  of the belief is $$\E\{\belief_{k+1}|\act_{1:k+1},\state_{0:k+1}\} = \int_\fBelief \belief\, \post_{k+1}(\belief) \,d\belief .$$
\end{theorem}
\begin{proof}
Start with the unnormalized density:
\begin{multline*}
  \pdf(\belief_{k+1},\obs_{k+1},\act_{1:k+1}, \state_{0:k+1})
  = \pdf(\act_{k+1}|\belief_{k+1})\, \pdf(\obs_{k+1}|\state_{k+1})\, \\
  \times \pdf(\state_{k+1}|\state_k)\, \int_{\fBelief} \pdf(\belief_{k+1}| \obs_{k+1},\belief_k)\,  \, \pdf(\belief_k,\act_{1:k},\state_{0:k})\, d\belief_k .
\end{multline*}
Note that $\pdf(\belief_{k+1}|\obs_{k+1},\belief_k) = \delta(\belief_{k+1} - \filter(\belief_k, \obs_{k+1}))$ where $\delta(\cdot)$ denotes  the Dirac delta function.
Finally,  marginalizing\footnote{It is here that the assumption that $\oprob_{\state,\cdot}$ is a pdf is used. If $\oprob_{\state,\cdot}$ is a pmf, then the probability of multiple observations resulting in the same posterior needs to be accounted for; see (\ref{eq:hmmpost}). \label{foot:multiple}}  the above  by integrating over $\obs_{k+1}$ and then normalizing yields (\ref{eq:post}). The term $\pdf(\state_{k+1}|\state_k)$ cancels out after normalization.
\end{proof}

  We refer to (\ref{eq:post}) as the {\em optimal inverse filter} since it provides the Bayesian posterior of the adversary's  belief, given our state and noisy measurements of the  adversary's actions.
  Note that $\post_{k+1}$   does depend on $\tp$ since  from (\ref{eq:belief}),  $\obs_{\belief_k,\belief_{k+1}}$ depends on $\tp$.


\subsubsection{Example 1. Inverse HMM Filter}  \index{inverse filtering! HMM filter}
\index{HMM filter! inverse filter}
Here  $\statespace = \{1,\ldots,\statedim\}$,
  $\obspace = \{1,\ldots,\obsdim\}$,  $\actspace = \{1,\ldots,\actdim\}$. Our state $\{\state_k, k \geq 0\}$ is a finite-state Markov chain with transition matrix $\tp$. The adversary observes the   HMM measurements  $\{\obs_k,k\geq 0\}$ where $\obs_k \in \obspace$ is drawn from  observation probability $\oprob$.
  Then $\filter(\belief,\obs)$ in (\ref{eq:belief}) is the  HMM filter deployed by the adversary
to compute the posterior of our state.
  $\fBelief$ is the  $(\statedim-1)$-dimensional simplex, i.e., the belief space of $\statedim$-dimensional probability vectors.
Assume that  (\ref{eq:post}) is initialized  with $\post_0(\belief) = \delta(\belief- \belief_0)$, i.e., the prior is a Dirac delta at $\belief_0 \in \fBelief$.

How to estimate the adversary's posterior distribution $\belief_k$ of our state, by observing its actions in noise? We use the inverse HMM filter.
For $k=1,2,\ldots$ construct the  finite sets $\fBelief_k$ of belief states via the following recursion:
$$\fBelief_k = \big\{ \filter(\belief,\obs), \;\obs \in \obspace, \belief \in \fBelief_{k-1} \big\}, \quad  \fBelief_0 = \{\belief_0\}.$$
 Note  $\fBelief_k$ has
$\obsdim^k$ elements.
Using (\ref{eq:post}), the inverse HMM filter for estimating the belief reads: for $\belief \in \fBelief_{k+1}$, the posterior and conditional mean are
\begin{align}
 & \post_{k+1}(\belief) = \frac{\aprob_{\belief,\act_{k+1}}
    \,  \sum_{\bbelief \in \fBelief_k} \sum_{\obs \in\obspace_{\bbelief,\belief}}\oprob_{\state_{k+1}, \obs}\, \post_k(\bbelief) }
  {\sum_{\belief \in \fBelief_{k+1}} \aprob_{\belief,\act_{k+1}}
    \,  \sum_{\bbelief \in \fBelief_k} \sum_{\obs \in\obspace_{\bbelief,\belief}} \oprob_{\state_{k+1}, \obs}\, \post_k(\bbelief) }
\nonumber  \\
 & \hat{\belief}_{k+1} = \E\{\belief_{k+1}|\act_{1:k+1},\state_{0:k+1}\} = \sum_{\belief \in \fBelief_{k+1}} \belief \,\post_{k+1}(\belief).
 \label{eq:hmmpost}
\end{align}
Here $\obspace_{\bbelief,\belief} = \{\obs\colon \filter(\belief,\obs) = \bbelief \}$ and accounts for the  fact that in finite  observation spaces, different observations can yield the same
posterior.
  The inverse HMM filter (\ref{eq:hmmpost})  is a finite-dimensional recursion, but the cardinality  of 
  $\fBelief_k$ grows exponentially with $k$;
\index{inverse filtering|)} see \cite{MRK20}.

 \subsubsection{Example 2. Inverse Kalman Filter}
 \index{inverse filtering! Kalman filter}
 Suppose that ``we'' are a target, and our state process $\{\state_k\}$ evolves  according to linear time invariant Gaussian dynamics. An  adversary radar observes us in Gaussian noise and runs a Kalman filter to compute the conditional mean estimate
 $\hat{\state}_k$ and covariance $\kalmancov_k$ of our state $\state_k$ at each time $k$.
 We call this the forward learner's Kalman filter.
 
 The adversary then chooses an action $\lact_k = \fun(\kalmancov_k)\,\hat{\state}_k$ at each time $k$ for some prespecified function $\fun$.  This choice is motivated by linear quadratic Gaussian control where the action is chosen as a linear function of the estimated state $\hstate_k$ weighted by the covariance matrix.
 
 Suppose we observe the adversary's actions  $\{\lact_k\}$ in i.i.d Gaussian noise as $\{\act_k\}$.
Given these noisy actions $\act_{1:k}$, our aim is to estimate the adversary's
posterior  $\belief_k=\normal(\hat{\state}_k,\kalmancov_k)$ of us,
at each time~$k$. This is achieved by the inverse Kalman filter which computes the posterior of $\belief_k$.

 
 Since the conditional mean update of a  Kalman  filter has linear Gaussian dynamics, our
 inverse Kalman filter operates on the following linear Gaussian state space model with state
 $\hstate_k$:
\beq \label{eq:inversekf}
\begin{split}
  {\hstate}_{k+1} &=   (I - \kg_{k+1} \obsm) \, \statem \hstate_{k} + \kg_{k+1} \onoise_{k+1} + \kg_{k+1} \obsm \state_{k+1},  \text{ where }  \kg_{k+1} = \kalmancov_{k+1|k} \obsm^{\p}  \Sig_{k+1}^{-1}  \\
   \act_k &= \fun(\kalmancov_k)\,\hstate_k + \anoise_k, \quad
   \anoise_k \sim \normal(0,\anoisecov) \text{ i.i.d. }   
\end{split}
\eeq
Here $\hstate_k$ is the unobserved state to us (since we observe the adversary's estimate $\hstate_k$ via noisy actions $a_k$), while our state $\state_k$ is known to us. 
The inverse Kalman filter is the Kalman filter algorithm run on model~\eqref{eq:inversekf}  to estimate $\hstate_k$, i.e., it estimates the adversary's estimate of us.

The above inverse Kalman filter does not interact with the forward learner's Kalman filter.

\section[Proof of Theorem \ref{thm:BIRL}]{Proof of Theorem~\vref{thm:BIRL}}\label{proof:BIRL}

{\bf Necessity of NIAS and NIAC Inequalities}:

1.  NIAS~(\ref{eq:NIAS}): For any environment $\dpiter\in\dpset$, define $\obsset_{\act}\subseteq\obsset$ such that for any observation $\obs\in\obsset_{\act}$, given posterior $p_{\dpiter}(\state|\obs)$, the optimal choice of action is $\act$ as defined in~\eqref{eq:utilitymaximization}. The revealed posterior given action $\act$, denoted  as $p_{\dpiter}(\state|\act)$,  is  the public belief in social learning, and is a  garbled version of the actual posterior  $p_{\dpiter}(\state|\obs)$:
 	\begin{equation}\label{eq:revpos}
 	  p_{\dpiter}(\state|\act) =\sum_{\obs\in\obsset} p_{\dpiter}(\obs|\act)\, p_{\dpiter}(\state|\obs).
 	\end{equation} 
 	Since the optimal action is $a$ for all $\obs\in\obsset_{\act}$, \eqref{eq:utilitymaximization} implies that for all actions $\actiontwo\in \actionset$:
 	\begin{align*}
 	& \hspace{-0.4cm}\quad\quad~~\sum_{\state\in \stateset} p_{\dpiter}(\state|\obs) (\utilitysymbolagent{\dpiter}(\state,\actiontwo)-\utilityagent{\dpiter})\leq 0\\
 	&\hspace{-0.4cm}\implies \sum_{\obs\in \obsset_{\act} } p_{\dpiter}(\obs|\act) \sum_{\state\in \stateset} p_{\dpiter}(\state|\obs) (\utilitysymbolagent{\dpiter}(\state,\actiontwo)-\utilityagent{\dpiter})\leq 0 \\
        & \hspace{-0.4cm}\implies \sum_{\obs\in \obsset} p_{\dpiter}(\obs|\act) \sum_{\state\in \stateset} p_{\dpiter}(\state|\obs) \big(\utilitysymbolagent{\dpiter}(\state,\actiontwo)-\utilityagent{\dpiter}\big)\leq 0
          \quad~~(\text{since }p_{\dpiter}(\obs|\act)=0,~\forall \obs\in\obsset\backslash\obsset_{\act})\\
	&\hspace{-0.4cm}\implies\sum_{\state\in \stateset} p_{\dpiter}(\state|\act) \big(\utilitysymbolagent{\dpiter}(\state,\actiontwo)-\utilityagent{\dpiter}\big)\leq 0~(\text{from }\eqref{eq:revpos}) \implies \text{NIAS~\eqref{eq:NIAS}}.
 	\end{align*}

2. NIAC~(\ref{eq:NIAC}):
The action probabilities
$\bB(\env)$ are a  garbled version of the observation probabilities
$\oprob(\env)$ since
 	\begin{equation}\label{eq:revattfun}
 	    \bB_{\state,\act}(\env)  = \pdf_{\dpiter}(\act|\state) = \sum_{\obs\in\obsset}p_{\dpiter}(\act|\obs)\,\oprob_{\state,\obs}(\env).
          \end{equation}
 Equation         \eqref{eq:revattfun} implies that $B(m)$ Blackwell dominates $\bB(m)$. (We discussed  Blackwell dominance extensively in the POMDP book  which we denoted as  $B(m) \bd \bB(m)$.)
Next from~\eqref{eq:attentionmaximization},
$R(\reward_\env,\oprob,\belief_0)  $   is convex in~$\oprob$. 
Then Blackwell dominance and convexity imply that for any two environments $l$ and $m$, 
 	\begin{equation}
 	    R(r_l, \bB(\env),\belief_0) \leq R(r_l,\attfunsymb(\env),\belief_0).  \label{eq:ineq_BLK}
 	\end{equation}
The above inequality holds with equality if $\dpiter=\dpitertwo$ (due to NIAS \eqref{eq:NIAS}).

Then
consider~\eqref{eq:attentionmaximization} for optimality of attention policy.
  Let $\hb_{\dpiter}=\RIcost(\attfunsymb(\env),\belief_0)>0$  denote the information acquisition cost. The following inequalities hold for any two environments $\dpitertwo\neq\dpiter$:
$$
     R(r_m, \bB(\env),\belief_0) - \hb_{\dpiter}\overset{\eqref{eq:ineq_BLK}}{=} R(r_m, \oprob(\dpiter),\belief_0) - \hb_{\dpiter} 
     \overset{\eqref{eq:attentionmaximization}}{\geq} R(r_m,\oprob(\dpitertwo),\belief_0) - \hb_{\dpitertwo} \geq 
     R(r_m, \bB(l),\belief_0) - \hb_{\dpitertwo} .
     $$
 The last inequality follows from~\eqref{eq:ineq_BLK} and is the NIAC inequality~\eqref{eq:NIAC}; see~\eqref{eq:niac_full}.

 \noindent{\bf Sufficiency of NIAS and NIAC Inequalities:}
 
Let $\{\utilitysymbol_\dpiter,\hb_\dpiter\}_{\dpiter=1}^\numdp$ denote a feasible solution to the NIAS and NIAC inequalities of Theorem~\ref{thm:BIRL}. To prove sufficiency, we construct the following  UMRI tuple as a function of dataset $\datasetaccum$, and then construct a  feasible solution that satisfies the optimality conditions~\eqref{eq:attentionmaximization},  \eqref{eq:utilitymaximization}:
 \begin{equation}
   \begin{split}
     &\bigtuple  = (\dpset,\stateset,\obsset=\actionset,\actionset,\belief_0,\RIcost,\{\bB_{x,a}(\env)=p_\dpiter(\act|\state),\utilitysymbol_\dpiter,\dpiter\in\dpset\}),  \\ 
 &  \text{ where } \;   \RIcost(\bB,\belief_0)  = \max_{\dpiter\in\dpset} \hb_\dpiter + R(r_m, \bB,\belief_0) - R(r_m, \bB(\env),\belief_0).
   \end{split}
   \label{eq:construct_cost}
 \end{equation}
Here, $\RIcost(\cdot,\belief_0)$ is convex since it is a point-wise maximum of monotone convex functions. Further, since NIAC is satisfied, \eqref{eq:construct_cost} implies $\RIcost(\bB(\dpiter),\belief_0)=\hb_{\dpiter}$. It only remains to show that inequalities \eqref{eq:utilitymaximization} and \eqref{eq:attentionmaximization} are satisfied for all environments  in $\dpset$.
 \begin{compactenum}
 \item {\em NIAS implies \eqref{eq:utilitymaximization} holds.} This hold trivially  since the observation and action sets can be chosen to be identical. Recall that the analyst does not know the observation space.

 \item {\em Information Acquisition Cost \eqref{eq:construct_cost} implies \eqref{eq:attentionmaximization} holds.} Fix environment  $\dpitertwo\in\dpset$. Then 
 \begin{align*}
     &\RIcost(\bB,\belief_0)  =  \max_{\dpiter\in\dpset} \big\{ \hb_{\dpiter} + R(r_m,\bB,\belief_0) - R(r_m, \bB(\env),\belief_0) \big\} \\
      & \implies R(r_l,\bB(\dpitertwo),\belief_0)  -\hb_\dpitertwo \geq R(r_l, \bB,\belief_0) - \RIcost(\bB,\belief_0), \quad \forall~ \bB \in \attspace \\
        & \implies \bB(\dpitertwo) \in \argmax_{\bB \in \attspace} R(r_\dpitertwo, \bB,\belief_0) - \RIcost(\bB,\belief_0)=\eqref{eq:attentionmaximization}.
 \end{align*}
 \end{compactenum}

\end{subappendices}


\chapter{Langevin Dynamics for Adaptive Inverse Reinforcement Learning}
\label{chp:langevin}

As discussed in previous chapters, 
IRL aims to  estimate the reward function of optimizing agents by observing their response (estimates or actions).
  This chapter studies IRL  when noisy estimates of the gradient of
a reward  function  generated by
  multiple stochastic gradient agents are observed.
  We  present a generalized Langevin dynamics algorithm to estimate the reward function $\Reward(\th)$; specifically, the resulting Langevin algorithm asymptotically generates samples from the  distribution
  proportional to $\exp(\Reward(\th))$. The proposed adaptive IRL algorithms use  kernel-based passive learning schemes. We also construct multi-kernel passive Langevin algorithms for IRL which are suitable for high dimensional data. The performance of the proposed IRL algorithms are illustrated on examples in adaptive Bayesian learning, logistic regression (high dimensional problem) and constrained Markov decision processes.

Specifically,
the problem we consider is this: Suppose we  observe  estimates of multiple (randomly initialized) stochastic gradient algorithms (reinforcement learners) that  aim to maximize a (possibly non-concave) expected reward. {\em How to  design another stochastic gradient algorithm \textit{(inverse learner)} to estimate the expected reward function?}
Put simply our aim is to design an 
inverse stochastic gradient algorithm. A key issue is that the inverse stochastic gradient algorithm needs to be passive: the inverse learner
observes noisy gradients evaluated at random points chosen by the adversary. Unlike classical (active)
stochastic gradient algorithms, the inverse learner cannot specify where the gradients are evaluated. Thus
we will design  passive learning from mis-specified noisy gradients

\section{RL and IRL Algorithms}
To discuss  the main ideas, we first describe the point of view of multiple  agents performing reinforcement learning (RL).  These agents act sequentially to perform RL by using   stochastic gradient algorithms  to maximize a reward function.
Let $\tslow = 1,2\ldots$ index agents that perform RL sequentially. The sequential protocol is as follows.  The
agents  aim to maximize a possibly non-concave   reward $\Reward(\th) = \E\{\reward_\dtime(\th)\}$ where $\th \in \reals^\thdim$.
Each agent $\tslow$  runs a stochastic  gradient algorithm over the time horizon $ \dtime \in \{ \stoptimeirl_{\tslow} , \stoptimeirl_\tslow+1,\ldots, \stoptimeirl_{\tslow+1}-1\}$:
\begin{equation}
  \label{eq:rl}
 \begin{split}
   \th_{\dtime+1} &= \th_\dtime + \step \,\nabla_\th \reward_\dtime(\th_\dtime), \quad \dtime = \stoptimeirl_{\tslow} , \stoptimeirl_\tslow+1,\ldots, \stoptimeirl_{\tslow+1}-1 \\
 &  \text{ initialized independently by }   \th_{\stoptimeirl_\tslow} \sim \belief(\cdot).
 \end{split}
\end{equation}
Here $\nabla_\th \reward_\dtime(\th_\dtime)$ denotes the sample path gradient evaluated at $\th_\dtime$, and
$\stoptimeirl_\tslow$, $\tslow=1,2\ldots,$ denote stopping times measurable wrt the $\sigma$-algebra
generated by
$\{\th_{\stoptimeirl_\tslow}, \nabla_\th \reward_\dtime(\th_\dtime),\dtime=\stoptimeirl_\tslow, \stoptimeirl_\tslow+1,\ldots\} $. The initial estimate  $\th_{\stoptimeirl_\tslow} $ for agent $n$ is sampled independently
from  probability density function $\belief$ defined on   $\reals^\thdim$.  Finally,  $\step$ is a small positive constant step size.

Next  we consider the point of view of an observer that  performs \textit{inverse reinforcement learning (IRL)} to estimate the reward function $\Reward(\th)$. The observer (inverse learner)  knows initialization density $\belief(\cdot)$ and only has access to    the estimates $\{\th_\dtime\}$ generated by RL algorithm  (\ref{eq:rl}). The observer  reconstructs
  the gradient  $\nabla_\th \reward_\dtime(\th_\dtime)$  as
  $\hat{\nabla}_\th \reward_\dtime(\th_\dtime) = (\th_{\dtime+1} - \th_{\dtime})/\stepa$ for some positive step size $\stepa$.
  
The main idea of this chapter  is to propose and analyze  the following IRL  algorithm (which is a passive Langevin dynamics  algorithm) deployed by the observer: \index{passive Langevin dynamics}
\begin{equation}
  \boxed{  \eth_{\dtime+1} = \eth_\dtime +  \stepa  \bigl[ \kerneln\big(\frac{\th_\dtime-\eth_\dtime}{\kernelstep}\bigr)\,
 \frac{\temperature}{2}\, \nabla_\th \reward_\dtime(\th_\dtime) + \nabla_\eth \belief(\eth_\dtime)\bigr] \,  \belief(\eth_\dtime)+  \sqrt{\stepa}\, \belief(\eth_\dtime)\, \noise_\dtime,
    \quad \dtime =1,2,\ldots},
\label{eq:irl}
\end{equation}
initialized by  $\eth_0 \in \reals^\thdim$.
Here $\stepa$ and $\kernelstep$ are small positive constant step sizes,  $\{\noise_k, \dtime \geq 0\}$ is an i.i.d. sequence
of  standard $\thdim$-variate Gaussian random variables, and $\temperature = \step/\stepa $ is a fixed constant. 
Note that we have  expressed  (\ref{eq:irl}) in terms of $\nabla_\th \reward_k(\th_k)$
(rather than $\hat{\nabla}_\th \reward_\th(\th_k)$) since we have absorbed the ratio of step sizes into  the scale factor $\temperature$.

The key construct in (\ref{eq:irl})  is the kernel function\footnote{Our use  of the  term ``kernel'' stems from non-parametric statistics and passive stochastic approximation. It is not related to reproducing kernels in Hilbert spaces.}
   $\kernel(\cdot)$. This kernel function is chosen by the observer such that $\kernel(\cdot)$ decreases monotonically to zero as any component of the argument increases to infinity,
\beq \label{eq:kernel_properties} \kernel(\th) \geq 0, \quad \kernel(\th) = \kernel(-\th), \quad  \int_{\reals^\thdim} \kernel(\th) d\th = 1.
\eeq

An example is to choose the kernel as a multivariate normal  $\normal(0,\sigma^2I_{\thdim})$ density with $\sigma = \kernelstep $, i.e.,
$$ \kerneln\bigl(\frac{\th}{\kernelstep}\bigr) = (2 \pi)^{-\thdim/2}  \kernelstep^{-N} \exp \bigl (- \frac{\|\th\|^2}{2 \kernelstep^2}\bigr),
$$ which is essentially  like
a Dirac delta centered at 0 as $\kernelstep \rightarrow 0$.
Our main result stated informally is as follows; see Theorem~\ref{thm:weak-conv} in
Sec.\ref{sec:weak} for the formal statement.

To sumarize, the above  adaptive IRL algorithm 
 has  two important differences compared to the previous two chapters.  First, we are interested in {\em adaptive IRL}, namely, learning and tracking the reward  in real time as the dataset is being generated 
by the forward learner. Second, the IRL algorithm needs to be  passive:
in an adversarial setting the inverse learner  cannot specify to the forward learner  where to evaluate the gradients.
The above passive Langevin dynamics algorithm takes these e two requirements into account.

\noindent {\bf  Convergence of Adaptive IRL (Informal).} 
{\it
Based on the estimates $\{\th_\dtime\}$ generated by RL algorithm  \eqref{eq:rl}, the
 IRL algorithm  \eqref{eq:irl} asymptotically generates samples
$\{\eth_\dtime\}$
from the Gibbs measure
\beq \stat(\eth)   \propto  \exp \bigl(  \temperature \Reward(\eth ) \bigr) , \quad \eth \in \reals^\thdim, \quad
\text{ where } \temperature =  \step/\stepa .\label{eq:stationary1} \eeq
}

To explain the above result, let   $\hat{\stat}$ denote the empirical density function constructed from samples $\{\eth_\dtime\}$ generated by IRL algorithm (\ref{eq:irl}). Then
clearly\footnote{Since the IRL algorithm  does not know the step size $\step$ of the RL, it can only estimate $\Reward(\cdot)$ up to a proportionality constant $\temperature$. In classical Langevin dynamics  $\temperature $ denotes an inverse temperature parameter.}   $\log \hat{\stat}(\eth) \propto \Reward(\eth)$. Thus  IRL algorithm~(\ref{eq:irl}) serves as a \textit{non-parametric method} for  stochastically exploring and reconstructing  reward $\Reward$, given the estimates $\{\th_\dtime\}$ of  RL algorithm (\ref{eq:rl}).
Hence  based on  the estimates $\{\th_\dtime\}$ generated by  RL algorithm (\ref{eq:rl}), IRL algorithm~(\ref{eq:irl})  serves as a randomized sampling method for exploring the reward $\Reward(\eth)$ by  simulating random samples from it. Finally, in adaptive Bayesian learning discussed in  Sec.\ref{sec:numerical},  the RL agents maximize  $\log \Reward(\eth)$ using gradient algorithm (\ref{eq:rl});  then IRL algorithm (\ref{eq:irl}) directly yields samples from $\temperature \Reward(\eth)$.

\section{Context and Discussion} \label{sec:context}

The stochastic gradient RL algorithm (\ref{eq:rl}) together with non-parametric
passive Langevin IRL algorithm (\ref{eq:irl}) constitute the adaptive IRL framework. Figure \ref{fig:schematic} displays this framework.

\begin{figure}[h]
\begin{tikzpicture}[node distance =5.25cm and 2cm, auto]
   \node [blockff] (l1) {Stochastic Gradient Learner $\{\th_k\}$};
 \node [blockff,right of=l1] (l2) {Passive IRL Langevin\\ Dynamics~$\{\eth_\dtime\}$};
 \node [blockff,left of=l1] (l0) {Time evolving Utility $\Reward(\cdot)$};
  \node [right of=l2,node distance=2.5cm] (estimatedu)[draw=none]{};
 \draw[->](l0) -- node[above,pos=0.5]{noisy} node[below,pos=0.5]{measurement} (l1);
 \draw[->](l1) -- node[above,pos=0.5]{ $\{\nabla_\th \reward_\dtime(\th_\dtime)\}$} (l2);
 \draw[->](l2) --  node[above,pos=0.5]{ $\hat{\Reward}(\cdot)$} (estimatedu);
 \draw[dashed,->](l2) --  node[anchor=south] {} ++(0,-1.5) -| (l1);
 \node[text width=4cm] at (3,-1.2) {active (Sec.\ref{sec:activeirl})};
\end{tikzpicture}
\caption{{\small Schematic of adaptive IRL framework. Multiple agents (learners)
  compute noisy gradient estimates $ \nabla_\th \reward_\dtime(\th_\dtime)$ of
  a possibly time evolving reward $\Reward(\cdot)$.
  By observing these gradient estimates, the IRL Langevin dynamics algorithm
 (\ref{eq:irl})  generates samples $\eth_k \sim \exp(\temperature \Reward(\eth))$. So reward
  $\Reward(\cdot)$ can be estimated from the log of the empirical distribution of $\{\eth_k\}$.  The IRL algorithm (\ref{eq:irl}) is \textit{passive}:  its estimate $\eth_k$
  plays no role in determining the point $\th_\dtime$ where the learner evaluates gradients $ \nabla_\th \reward_\dtime(\th_\dtime)$.
Sec.\ref{sec:alternate} presents several additional IRL algorithms including a variance reduction algorithm and a non-reversible diffusion.
 Sec.\ref{sec:activeirl} presents an
  active IRL where the IRL requests the learner to provide a gradient at $\eth_k$, but the learner provides the noisy mis-specified gradient estimate
  $\nabla_\th \reward_\dtime(\eth_\dtime+ \obsnoise_\dtime)$ where $\{v_k\}$ is a noise sequence. Finally, Sec.\ref{sec:markov} discusses  the tracking properties of the IRL algorithm when the reward $\Reward(\cdot)$ evolves in time according to an unknown Markov chain.
}}
\label{fig:schematic}
\end{figure}
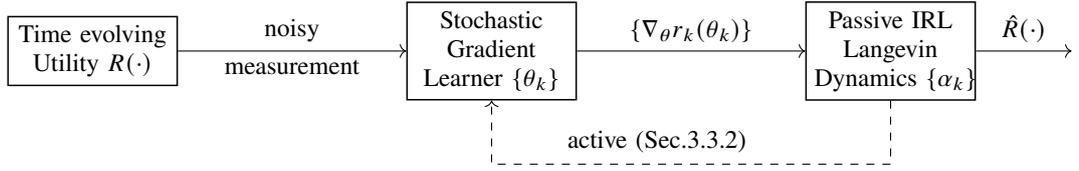

More abstractly, the IRL problem we address is this: given a sequence of noisy sample path  gradients $\{\nabla_\th \reward_\dtime(\th_k)\}$, how to estimate the expected reward $\Reward(\cdot)$? The IRL algorithm (\ref{eq:irl}) builds on classical stochastic gradient algorithms in 3 steps. First, it is {\em passive}: it does not specify where the RL agents compute gradient estimates.  The gradient estimates are evaluated by RL agents at points $\th_k$, whereas the IRL algorithm requires  gradients at $\eth_k$. To incorporate these mis-specified gradients, the passive algorithm uses the kernel $\kernel(\cdot)$.
  Second, a classical passive stochastic gradient algorithm only estimates a local  maximum of $\Reward(\cdot)$; in comparison we are interested in non-parametric reconstruction  (estimation) of the entire reward $\Reward(\cdot)$. Therefore,  we use  a passive {\em Langevin dynamics} based algorithm. Finally, we are interested in {\em tracking} (estimating) time evolving reward functions $\Reward(\cdot)$. Therefore we use a {\em constant step size}, passive Langevin dynamics IRL algorithm; see point (vii) below.

To give additional insight we now discuss the context, useful generalizations of IRL algorithm (\ref{eq:irl}),  and related works in the literature.

\begin{enumerate}[wide,labelwidth=!, labelindent=0pt,label=(\roman*)]
\item {\em Multiple agents}.  The multiple agent RL algorithm  (\ref{eq:rl}) is natural in non-convex stochastic optimization problems. Starting from various randomly chosen  initial conditions  $\th_{\stoptimeirl_\tslow} \sim \belief(\cdot)$, the  agents evaluate the gradients $\nabla_\th\reward_\dtime(\th_\dtime)$
  at various points $\th_\dtime$ to estimate the global maximizer. Since
  the initializations $\{\th_{\stoptimeirl_\tslow}\}$ is a sequence of
  independent random variables,
  the RL agents can also act in parallel
(instead of sequentially).
 Given this sequence of gradients $\{\nabla_\th\reward_\dtime(\th_\dtime)\}$, the aim  is to construct  IRL algorithms to estimate $\Reward(\th)$.
As an example, motivated by
 stochastic control involving information theoretic measures \cite{GRW14} detailed in Sec.\ref{sec:bayesian}, suppose
 multiple RL agents run stochastic gradient algorithms to estimate the minimum
 of a non-convex Kullback Leibler (KL) divergence.
 By observing these gradient estimates, the passive IRL algorithm (\ref{eq:irl})  reconstructs  the KL divergence.\footnote{To give additional context, multi-agent systems for  {\em deterministic} optimization are  studied in \cite{NO09} where the aim is to optimize cooperatively the sum of convex objectives corresponding to multiple agents. Agents  (represented by  nodes in a graph) deploy deterministic sub-gradient algorithms  and exchange information to achieve consensus.  We consider a {\em stochastic}  optimization framework where agents compute noisy gradient estimates. Our passive IRL uses these gradient estimates  to reconstruct (explore) the reward function and  uses a (stochastic gradient)) Langevin dynamics  algorithm.}

\item  {\em Passive IRL}. The IRL algorithm (\ref{eq:irl}) is a Langevin dynamics based gradient algorithm
  with  injected noise $\{\noise_\dtime\}$. It is a {\em passive} learning algorithm since the gradients are not evaluated at $\eth_\dtime$ by the inverse learner; instead the gradients are evaluated at the random points  $\th_\dtime$ chosen by the RL algorithm. This passive framework  is natural in an IRL. The inverse learner passively observes the RL algorithm and aims to estimate its utility.

We emphasize  that the passive Langevin  IRL algorithm \eqref{eq:rl}  estimates
    the utility function $\Reward(\th)$; see (\ref{eq:stationary1}). This is unlike classical passive stochastic gradient algorithms that
    estimate a local stationary point of the utility. To the best of our knowledge, such passive Langevin dynamics algorithms have not been proposed or analyzed - yet such algorithms arise naturally in estimating the utility by observing the estimates from a stochastic gradient algorithm.

The kernel $\kernel(\cdot)$ in (\ref{eq:irl}) effectively weights the usefulness of the gradient $\nabla_\th\reward_\dtime(\th_\dtime)$ compared to the required gradient
$\nabla_\eth\reward_\dtime(\eth_\dtime)$.
If $\th_\dtime$ and $\eth_\dtime$ are far apart, then kernel $\kernel((\th_\dtime-\eth_\dtime)/\kernelstep)$  will be  small. Then only a small proportion of the gradient estimate $\nabla_\th\reward_\dtime(\th_\dtime)$ is added to the IRL iteration. On the other hand, if $
\eth_\dtime = \th_\dtime$,
(\ref{eq:irl}) becomes a standard Langevin dynamics type algorithm. We refer to \cite{Rev77,HN87,NPT89,YY96} for the analysis of  passive stochastic  gradient algorithms.  The key difference compared to these works is that
we are dealing with a passive Langevin dynamics algorithm, i.e., there is an extra injected noise  term involving $\noise_\dtime$.

\item  {\em Intuition behind passive Langevin IRL algorithm (\ref{eq:irl})}. To discuss the intuition behind (\ref{eq:irl}), we first
  discuss the classical Langevin dynamics and also a more general reversible diffusion.
    The classical Langevin dynamics algorithm with fixed step size $\mu>0$ and deterministic reward $\Reward(\th)$    is of the form
    \begin{equation}
      \label{eq:classical_lang}
      \th_{\dtime+1} = \th_{\dtime} + \stepa\,  \nabla \Reward(\th_\dtime)
   + \sqrt{\stepa}\,  \sqrt{\frac{2}{\temperature}}\, \noise_\dtime,
    \quad \dtime =1,2,\ldots
  \end{equation}
 Indeed \eqref{eq:classical_lang}  is the Euler-Maruyama time discretization of the continuous time diffusion
  process
  \begin{equation}
    \label{eq:3}
    d\th(t) = \nabla_\th R(\th) + \sqrt{\frac{2}{\temperature}}\, d\bm(t)
  \end{equation}
    which has stationary measure $\pdf(\th)$ given by  (\ref{eq:stationary1}).
More generally,  assuming $\sigma(\cdot)$ is differentiable, \cite{ST99} studied  reversible diffusions of the form
    \begin{equation}
      \label{eq:tweedie}
     d\th(t) =  \biggl[\frac{\temperature}{2} \sigma(\th)\, \nabla_\th \Reward(\th)\, dt +  \nabla_\th \sigma(\th)\, dt +  d\bm(t) \biggr]\,\sigma(\th),
   \end{equation}
   whose  Euler-Maruyama time discretization  yields
   \begin{equation}
     \label{eq:tweedie_discrete}
     \th_{\dtime+1} = \th_\dtime +  \stepa  \bigl[
 \frac{\temperature}{2}\, \nabla_\th \Reward(\th_\dtime) + \nabla_\th \sigma(\th_\dtime)\bigr] \,  \sigma(\th_\dtime)+  \sqrt{\stepa}\, \sigma(\th_\dtime)\, \noise_\dtime,  \quad \dtime =1,2,\ldots
   \end{equation}
    It is easily verified that  reversible diffusion \eqref{eq:tweedie} has the same   Gibbs stationary measure $\pdf(\th)$ in~(\ref{eq:stationary1}).

    The IRL algorithm \eqref{eq:irl} substantially generalizes
      \eqref{eq:tweedie_discrete} in three ways: First, the gradient is at a mis-specified point $\th_k$ compared to $\eth_k$; hence we use the kernel
      $\kernel$ as discussed in point (ii) above.
      Second, unlike \eqref{eq:tweedie_discrete} which uses  $\nabla_\th \Reward(\th)$, IRL
      algorithm (\ref{eq:irl}) only has the (noisy) gradient estimate $\nabla_\th \reward_\dtime(\th)$. Finally, we choose $\sigma(\th)$ as $\belief(\th)$, namely the initialization density specified in (\ref{eq:rl}), to ensure that the stationary measure is as specified in \eqref{eq:stationary1} as explained at the end of
      Sec.\ref{sec:informal}.

The intuition behind the weak convergence of the passive Langevin IRL algorithm (\ref{eq:irl}) is explained in Sec.\ref{sec:informal}.
It is shown there  via stochastic averaging arguments  as the kernel converges to a Dirac-delta, the IRL algorithm (\ref{eq:irl})  converges to the reversible diffusion process \eqref{eq:tweedie} with
stationary measure given by (\ref{eq:stationary1}).

  \item {\em  IRL for Markov Decision Process.}
  Several types of RL based policy gradient algorithms in the Markov decision process
(MDP) literature \cite{BT96,SB98} fit our framework.
As a motivation, we now briefly discuss IRL for an infinite horizon average cost\footnote{As mentioned in Sec.\ref{sec:mdp}, our   IRL algorithms also apply to the simpler discounted cost MDP case.} MDP; details are  discussed in
Sec.\ref{sec:mdp}. Let  $\{\state_n\}$ denote a finite state Markov chain with
controlled transition probabilities $\tp_{ij}(\action) = \prob[\state_{n+1} = j | \state_n = i, \action_n = \action]$ where action $\action_n$ is chosen from policy $\policy_\th$ parametrized by $\th$ as  $\action_n = \policy_\th(\state_n)$.
Solving an average cost MDP (assuming it is unichain \cite{Put94}) involves computing  the optimal parameter $\th^*= \sup\{\th: \Reward(\th)\}$ where  the cumulative reward is
\begin{equation}
   \label{Jo}
\Reward(\th) = \lim_{\finaltime\to\infty} \inf \frac{1}{\finaltime} \E_{\th} \Big[\sum_{n=1}^\finaltime
\mreward(\state_n,\action_n) \mid \state_0 = x\Big], \action_n = \policy_\th(\state_n)
\end{equation}
Suppose now that a forward learner runs a policy gradient RL  algorithm  that evaluates
estimates $\nabla_\th r_k(\th)$ of $\nabla_\th \Reward(\th)$ in order to estimate $\th^*$.
Given these gradient estimates, how can an IRL algorithm estimate $\Reward(\th)$?

In Sec.\ref{sec:mdp}, motivated by widely used fairness constraints
  in wireless communications, we consider more general  average cost
  \textit{constrained} MDPs (CMDPs), see \cite{Alt99,NK10,BJ10}. In CMDPs \cite{Alt99}, the optimal policy is randomized.  Since the optimal policy is randomized, classical stochastic dynamic programming or Q-learning cannot be used to solve CMDPs as they yield deterministic policies. One can construct a Lagrangian dynamic programming formulation \cite{Alt99} and Lagrangian Q-learning algorithms \cite{DK07}.
  Sec.\ref{sec:mdp}  considers the case where the RL agents deploy a  policy gradient  algorithm. By observing these gradient estimates, our IRL algorithm reconstructs the Lagrangian of the CMDP.

  Notice that our non-parametric setup is different to classical IRL in \cite{NR00}
where the inverse learner has access to actions from the optimal policy, knows the controlled transition probabilities, and formulates a set of linear inequalities that the reward function
$\mreward(\state,\action)$ satisfies.
Our IRL framework only has access to gradient estimates $\nabla_\th \reward_\dtime(\th)$ evaluated at random points $\th$, and does not require knowledge of the parameters of the CMDP.
Also  our IRL framework is adaptive (see point (vii) below): the IRL algorithm
(\ref{eq:irl}) can track a time evolving $\Reward(\th)$ due to  the transition probabilities or rewards  of the MDP evolving
over time (and unknown to the inverse learner).

\item  {\em Multikernel IRL}.  IRL algorithm (\ref{eq:irl}) requires the gradient $\nabla_\th \reward(\th_\dtime)$
and knowing the density $\belief(\cdot)$. In Sec.\ref{sec:highdim} we will discuss a two-time scale multi-kernel IRL algorithm, namely (\ref{eq:mcmcirl}),  that does not require knowledge of the density $\belief(\cdot)$. All that is required
is a sequence of samples $\{\nabla_\th \reward_\dtime(\th_i),i=1\,\ldots,\numparticles\}$ when the IRL estimate is $\eth_\dtime$.  The multi-kernel  IRL algorithm (\ref{eq:mcmcirl}) incorporates variance reduction and is suitable for high dimensional inference. In Sec.\ref{sec:alternate}, we also discuss several other variations of
IRL algorithm (\ref{eq:irl}) including a mis-specified active IRL algorithm where the gradient is evaluated at a point $\th_\dtime$ that is a corrupted value of $\eth_\dtime$.

\item  {\em Global Optimization vs IRL}. Langevin dynamics based gradient algorithms have been studied  as a means for achieving global minimization for non-convex stochastic
optimization problems, see for example \cite{GM91}. The papers \cite{TTV16,RRT17} give a comprehensive study of convergence of the Langevin dynamics stochastic gradient algorithm in a non-asymptotic setting.
Also \cite{WT11} studies Bayesian learning, namely, sampling from the posterior using stochastic gradient Langevin dynamics.

Langevin dynamics for global optimization  considers the limit as $\temperature \rightarrow \infty$.
In comparison, the IRL algorithms in this chapter consider the case of fixed $\temperature = \step/\stepa$, since we are interested in sampling from the reward $\Reward(\cdot)$.   Also, we consider  passive Langevin dynamics algorithms in the context of IRL. Thus the IRL algorithm (\ref{eq:irl}) is non-standard in two ways. First, as mentioned above, it has a kernel to facilitate passive  learning. Second, the IRL algorithm (\ref{eq:irl}) incorporates
the initialization probability $\belief(\cdot)$ which appears in the RL algorithm (\ref{eq:rl}). Thus
(\ref{eq:irl}) is  a non-standard generalized Langevin dynamics algorithm (which still has  reversible diffusion dynamics).

\item {\em Constant step size Adaptive IRL for Time Evolving Utility}.  An important feature of the IRL algorithm (\ref{eq:irl}) is the  constant step size $\stepa$  (as opposed to a decreasing step size).  This facilities  estimating (adaptively tracking) rewards  that evolve over time. Sec.\ref{sec:markov}  gives a formal weak convergence analysis of the asymptotic tracking capability of the IRL algorithm (\ref{eq:irl}) when the reward $\Reward(\cdot)$  jump changes over time according to an  unknown Markov chain. The Markov chain constitutes a hyper-parameter in the sense that
it is not known or used by the IRL algorithm; it is used in our convergence analysis to determine how well the IRL algorithm can learn a time evolving reward.


  The analysis is very different to classical
  tracking analysis of stochastic  gradient algorithms \cite{BMP90} where the underlying hyper-parameter   evolves continuously over time.

The assumptions used in analyzing the adaptation of the  IRL algorithm in
Sec.\ref{sec:markov}
    are  similar to those for the passive IRL in Sec.\ref{sec:weak}; the main additional assumption is that the Markov chain's transition matrix is parametrized by a small parameter $\mcstep$. Depending on how $\mcstep$ compares to the  IRL algorithm step size $\stepa$, we analyze
 three cases of adaptive IRL   in Sec.\ref{sec:markov}: (i)  the reward jump changes on a slower time scale than the dynamics of the Langevin IRL  algorithm, i.e., $\mcstep = o(\stepa)$;  (ii) the reward jump changes on the same time scale as the Langevin IRL algorithm,  i.e., $\mcstep = O(\stepa)$;
 (iii) The reward jump changes on a faster time scale compared to the Langevin IRL algorithm, i.e., $\stepa = o(\mcstep)$.

  The most interesting (and difficult) case considered in Sec.\ref{sec:markov}  is when the reward changes at the same rate as the IRL algorithm,  i.e., $\mcstep = O(\stepa)$. Then  stochastic averaging theory yields a Markov switched diffusion limit as the asymptotic behavior of the IRL algorithm.
This is in stark contrast to classical averaging theory of stochastic gradient algorithms which yields a deterministic ordinary differential equation \cite{KY03,BMP90}.
  Due to the constant step size, the appropriate notion of convergence is weak convergence \cite{KY03,EK86,Bil99}.  The Markovian hyper-parameter tracking analysis generalizes our earlier work \cite{YKI04,YIK09} in stochastic gradient algorithms to the current case of  passive Langevin dynamics with a kernel.

\item  {\em Estimating utility functions}. Estimating a utility function given the response of agents is studied under the area of revealed preferences in microeconomics.
 Afriat's theorem \cite{Afr67,Die12,Var12} in
revealed preferences uses the response of a linearly constrained optimizing agent to construct a set of linear inequalities that are necessary and sufficient for an agent to be an utility maximizer; and gives a set valued estimate
of the class of utility functions that rationalize the agents behavior.
Different to revealed preferences, the current chapter uses noisy gradients to recover the utility function and that too in real time via a constant step size Langevin diffusion algorithm.

We already mentioned classical IRL \cite{NR00,AN04} which aims to estimate an unknown deterministic reward function  of an agent by observing the optimal actions of the agent in a Markov decision process setting. 
  \cite{ZMB08} uses the principle of maximum entropy for achieving IRL of optimal agents.
More abstractly, IRL falls under the area of {\em imitation learning}  \cite{HE16,OPN18} which is the process of learning from demonstration.
Our IRL approach can be considered as imitation learning from mis-specified noisy gradients
evaluated at random points in Euclidean space.
  We perform {\em adaptive} (real time) IRL:  given  samples from a stochastic gradient algorithm (possibly a forward RL algorithm), we propose a Langevin dynamics algorithm to estimate the utility function. Our  real-time IRL algorithm  facilitates  adaptive  IRL, i.e.,   estimating (tracking) time evolving  utility functions. In Sec.\ref{sec:markov}, we analyze the tracking properties of such  non-stationary IRL algorithms  when the utility function jump changes according to an unknown Markov process.

\item  {\em Interpretation as a numerical integration algorithm}. Finally, it is helpful to view IRL algorithm (\ref{eq:irl}) as a numerical integration algorithm when the integrand (gradients to be integrated) are presented at  random points and the integrand terms are corrupted by noise (noisy gradients).
One possible offline approach  is to discretize $\reals^\thdim$  and numerically build up an estimate of the integral at the discretized points by rounding off the evaluated integrands terms to the nearest discretized point. However, such an approach suffers from the curse of dimensionality: one needs
$O(2^\thdim)$ points to construct the integral with a specified level of tolerance.
In comparison, the passive IRL algorithm (\ref{eq:irl}) provides a principled real time approach for generating samples from the integral, as depicted by main result~(\ref{eq:stationary1}).

\item Although our main motivation for {\em passive} Langevin dynamics stems from IRL, namely estimating a utility function, we mention  the interesting recent paper by \cite{KHH20}  which shows that classical Langevin dynamics   yields more robust RL algorithms compared to classic stochastic gradient. In analogy to \cite{KHH20}, in future work it is worthwhile exploring if our passive Langevin dynamics algorithm can be viewed as a robust version of classical passive stochastic gradient algorithms.

\item Finally, we assumed in  \eqref{eq:rl} that the RL  agents use a fixed step size $\step$ which is not necessarily known to the inverse learner deploying \eqref{eq:irl}. More generally, 
    the step size of each  RL agent $n$ in \eqref{eq:rl} can be chosen as
$\step_n = \step(1+\stepn_n)$ where $\{\stepn_n\}$ is an iid bounded zero mean process.
This models the case where the RL agents deploy different step sizes, for example, due to  separate hyper-parameter tuning methods, that are not known to the inverse learner.  Then our main  result  \eqref{eq:stationary1}
continues to hold.
In particular, denote $\temperature_n = \step_n/\stepa$. Then
 as explained in Footnote \ref{foot:step} in Sec.\ref{sec:informal}, 
by averaging theory arguments, $\temperature_n$ ``averages out''' to $\temperature = \step/\stepa$ on the IRL algorithm  time scale.

\end{enumerate}

\section{Informal Proof and Alternative IRL Algorithms} \label{sec:context}

The RL algorithm (\ref{eq:rl}) together with IRL algorithm (\ref{eq:irl}) constitute our main setup.
In this section, we first start with  an informal  proof of convergence of (\ref{eq:irl}) based on stochastic averaging theory; the formal proof is in Sec.\ref{sec:weak}. The informal proof provided below is useful  because it
gives additional insight into the design of related  IRL algorithms. We then discuss several  related IRL algorithms including a novel multi-kernel version with variance reduction.

\subsection{Informal  Proof of Main Result
\eqref{eq:stationary1}}    \label{sec:informal}
Since
the IRL algorithm (\ref{eq:irl}) uses a constant step size, the  appropriate notion of  convergence  is  weak convergence.
Weak convergence  (for example, \cite{EK86}) is a function space generalization of convergence in distribution; function space because we  prove convergence of the entire trajectory (stochastic process) rather than simply the estimate at a fixed time (random variable).

A few words about our proof approach. Until the mid 1970s, convergence  proofs of stochastic gradient algorithms assumed martingale difference type of uncorrelated noises. The so-called ordinary differential equation (ODE) approach  was proposed by \cite{Lju77} for correlated
noises and decreasing step size, yielding with probability one convergence. This was subsequently generalized by Kushner and coworkers (see for example \cite{Kus84}) to weak convergence analysis of constant step size algorithms. The assumptions required   are weaker and
  the results  more general than that used in classical mean square error analysis
because  we are
dealing with
suitably scaled sequences of the iterates that are treated
as stochastic processes rather than random variables. Our
approach captures the dynamic evolution of the
algorithm. As a consequence, using weak convergence
methods we can also analyze the tracking properties of the
IRL  algorithms when the parameters are time varying (see Sec.\ref{sec:markov}).

As is typically done in weak convergence analysis,
we first represent the sequence of estimates $\{\eth_k\}$  generated by the IRL algorithm as a continuous-time random process. This is done by constructing the continuous-time trajectory  via piecewise
constant interpolation as follows:
For
 $t   \in [0,\horizon]$,
define the continuous-time piecewise constant interpolated processes parametrized by the step size
$\stepa$ as
\beq
\eth^\stepa(t) = \eth_\dtime , \; \text{ for } \ t\in [\stepa \dtime, \stepa \dtime+ \stepa). \label{eq:interpolatedp} \eeq

Sec.\ref{sec:weak} gives the detailed weak convergence  proof using the martingale problem formulation of Stroock and Varadhan \cite{EK86}.

Our informal proof of the main result  (\ref{eq:stationary1})  proceeds in two steps:

\underline{\em Step I}.
We  first {\it fix} the kernel step size  $\kernelstep$ and apply stochastic averaging theory arguments: this says that at the slow time scale, we can replace the fast variables by their expected value.
For small step sizes $\step$ and $\stepa=\step/\temperature$, there are three time scales in IRL algorithm (\ref{eq:irl}):
\begin{compactenum} \item
$\{\th_\dtime\}$ evolves slowly on intervals $ \dtime \in \{\stoptimeirl_\tslow, \stoptimeirl_{\tslow+1}-1\}$, and $\{\eth_\dtime\}$ evolves slowly versus $\dtime$.
\item
 We assume that the run-time of the RL algorithm (\ref{eq:rl}) for each agent $\tslow$  is bounded by some finite constant, i.e.,
 $\stoptimeirl_{\tslow+1} - \stoptimeirl_\tslow < M $ for some constant $M$.
 So $\{\th_{\stoptimeirl_{\tslow}}\} \sim \belief$
 is a fast variable compared to
 $\{\eth_\dtime\}$.
\item Finally  the noisy gradient process $\{\nabla_\th \reward_\dtime(\cdot)\}$ evolves at each time $\dtime$ and is  a faster  variable than  $\{\th_{\stoptimeirl_{\tslow}}\}$ which is updated at stopping times $\stoptimeirl_\tslow$.
\end{compactenum}
 With the above time scale separation,
there are two levels of averaging involved. First  averaging the noisy gradient $\nabla_\th \reward_\dtime(\th_\dtime)$ yields $\nabla_\th\Reward(\th)$. Next\footnote{\label{foot:step} Recall item (xi) of Sec.\ref{sec:context} discussed the case where each agent $n$ chooses step size $\step_n = \step(1+\stepn_n)$. Then $\temperature_n = \step_n/\stepa$  is averaged at this time scale yielding $\temperature$.} averaging $\{\th_{\stoptimeirl_{\tslow}}\} $ yields $\th \sim \belief$. Thus applying
 averaging theory to IRL algorithm (\ref{eq:irl}) yields the following averaged system:
 \begin{multline}  \bar{\eth}_{\dtime+1} = \bar{\eth}_\dtime +  \stepa \, \E_{\th \sim \belief} \Big[ \kerneln\big(\frac{\th-\bar{\eth}_\dtime}{\kernelstep}\bigr)\,
 \frac{\temperature}{2}\, \nabla_\th \Reward(\th) + \nabla_\eth \belief(\bar{\eth}_\dtime)\Big] \,  \belief(\bar{\eth}_\dtime)+  \sqrt{\stepa}\, \belief(\bar{\eth}_\dtime)\, \noise_\dtime \\
  = \bar{\eth}_\dtime +  \stepa \,\int_{\reals^\thdim} \, \kerneln\big(\frac{\th-\bar{\eth}_\dtime}{\kernelstep}\bigr)\,
 \frac{\temperature}{2}\, \belief(\bar{\eth}_\dtime) \nabla_\th \Reward(\th) \belief(\th) d\th + \belief(\bar{\eth}_\dtime)\, \nabla_\eth \belief(\bar{\eth}_\dtime)   +  \sqrt{\stepa}\, \belief(\bar{\eth}_\dtime)\, \noise_\dtime . \label{eq:discavg}
 \end{multline}
Given the sequence $\{\bar{\eth}_\dtime\}$, define the interpolated continuous time process  $ \bar{\eth}^\stepa$  as in (\ref{eq:interpolatedp}).  Then
 as $\stepa$ goes to zero,
$ \bar{\eth}^\stepa$ converges weakly to the  solution of the stochastic differential equation
\begin{equation} \begin{split}
      d\eth(t) &=   \int_{\reals^\thdim}\kerneln\bigl(\frac{\th - \eth}{\kernelstep}\bigr)\, \biggl[\frac{\temperature}{2} \belief(\eth)\, \nabla_\th \Reward(\th)\, dt \biggr] \,\belief(\th)  \,d\th + \belief(\eth)\, \nabla_\eth \belief(\eth)\, dt +  \belief(\eth)\,  d\bm(t) , \\
      \eth(0) & = \eth_0,
        \end{split} \label{eq:2levelaverage}
      \end{equation}
      where $\bm(t)$ is standard Brownian motion. Put differently,
the Euler-Maruyama time discretization of (\ref{eq:2levelaverage}) yields (\ref{eq:discavg}).
To summarize (\ref{eq:2levelaverage}) is the continuous-time averaged dynamics of  IRL algorithm~(\ref{eq:irl}). This is formalized
in Sec.\ref{sec:weak}.


\underline{\em Step II}. Next, we set the kernel step size $\kernelstep \rightarrow 0$.
  Then $\kernel(\cdot)$ mimics a Dirac delta function and so  the asymptotic dynamics of (\ref{eq:2levelaverage}) become the  diffusion
     \beq
     d\eth(t) =  \biggl[\frac{\temperature}{2} \belief(\eth)\, \nabla_\eth \Reward(\eth)\, dt +  \nabla_\eth \belief(\eth)\, dt +  d\bm(t) \biggr]\,\belief(\eth),  \quad \eth(0) = \eth_0 \label{eq:gld}
     \eeq
     %
     Finally,  (\ref{eq:gld}) is a reversible diffusion and its  stationary measure  is the Gibbs measure $\stat(\eth)$ defined in~(\ref{eq:stationary1}).  Showing this is  straightforward:\footnote{Note  \cite[Eq.34]{ST99}  has a typographic  error in specifying the determinant.}
Recall \cite{KS91} that for a generic diffusion process denoted as $d\state(t) = \statem(\state) dt + \snoisecov(\state) d\bm(t)$, the stationary distribution $\stat$ satisfies
\beq  \forward \stat = \frac{1}{2} \Tr [ \nabla^2 ( \kalmancov \stat)] - \div ( \statem \stat) = 0 , \qquad  \text{ where }
\kalmancov = \snoisecov \snoisecov^\p \label{eq:forward} \eeq
and  $\forward$ is the  forward operator.
From (\ref{eq:gld}),
$\statem(\eth) = [\frac{\temperature}{2} \belief(\eth)\, \nabla_\eth \Reward(\eth)  +  \nabla_\eth \belief(\eth)] \belief(\eth)$, $\snoisecov = \belief(\eth) I$. Then it is  verified by elementary calculus  that $\stat(\eth) \propto  \exp \bigl(  \temperature \Reward(\eth ) \bigr) $ satisfies (\ref{eq:forward}).

To summarize, we have shown informally that IRL algorithm (\ref{eq:irl})  generates  samples from (\ref{eq:stationary1}).
Sec.\ref{sec:weak} gives the formal weak convergence  proof.

(v)
{\em  Why not use classical Langevin dynamics?}
The passive version of the  classical Langevin dynamics  algorithm reads:
\begin{equation}
   \eth_{\dtime+1} = \eth_\dtime + \stepa\, \kerneln\big(\frac{\th_\dtime-\eth_\dtime}{\kernelstep}\bigr)\,  \nabla \reward_\dtime(\th_\dtime)
   + \sqrt{\stepa}\,  \sqrt{\frac{2}{\temperature}}\, \noise_\dtime,
    \quad \dtime =1,2,\ldots
\label{eq:irlstandard}
\end{equation}
 where $\th_\dtime$ are computed by RL (\ref{eq:rl}).
Then averaging theory (as $\stepa\rightarrow 0$ and then $\kernelstep \rightarrow 0$)
 yields the following asymptotic dynamics  (where $\bm(t)$ denotes standard Brownian motion)
 \beq
 d\eth(t) =  \nabla_\eth \Reward(\eth)\, \belief(\eth) dt + \sqrt{\frac{2}{\temperature}} d\bm(t)   ,  \quad \eth(0) = \eth_0 \label{eq:hard}
 \eeq
 Then  the stationary distribution of (\ref{eq:hard}) is proportional to
 $\exp(\temperature \int[ \nabla_\eth \Reward(\eth)\, \belief(\eth) ] d\eth )$. Unfortunately, this is difficult to relate to $\Reward(\eth)$ and therefore less useful.
  In comparison, the generalized Langevin algorithm (\ref{eq:irl})   yields samples from stationary distribution proportional to $\exp(\temperature\Reward(\eth))$
  from which $\Reward(\eth)$ is easily estimated (as discussed below (\ref{eq:stationary1})).
 This is the reason why we  will use the passive generalized Langevin dynamics (\ref{eq:irl}) for IRL instead of the passive classical  Langevin dynamics (\ref{eq:irlstandard}).

 \subsection{Alternative IRL Algorithms} \label{sec:alternate}
 IRL algorithm (\ref{eq:irl}) is the vanilla  IRL algorithm.
 In this section we discuss  several variations of  IRL algorithm  (\ref{eq:irl}). The algorithms discussed below include a passive version of the classical Langevin dynamics, a two-time scale multi-kernel MCMC based IRL algorithm (for variance reduction) and finally,  a non-reversible diffusion algorithm. The construction of these algorithms are based on the informal proof discussed above.

 \subsubsection{Passive Langevin Dynamics Algorithms for IRL}
 IRL algorithm (\ref{eq:irl}) can be viewed as a passive modification of the generalized Langevin dynamics proposed in \cite{ST99}. Since  generalized Langevin dynamics includes  classical Langevin dynamics as a special case, it stands to reason that we can construct a passive version of the classical Langevin dynamics algorithm.  Indeed, instead of (\ref{eq:irl}),  the following   passive Langevin dynamics can be used for IRL
 (initialized by $\eth_0 \in \reals^\thdim$):
     \begin{equation}
 \boxed{   \eth_{\dtime+1} = \eth_\dtime +  \stepa  \, \kerneln\big(\frac{\th_\dtime-\eth_\dtime}{\kernelstep}\bigr)\,
 \frac{\temperature}{2\, \belief(\eth_\dtime)}\, \nabla_\th \reward_\dtime(\th_\dtime) +  \sqrt{\stepa}\,  \noise_\dtime,
    \quad \dtime =1,2,\ldots }
\label{eq:irl2}
\end{equation}
Note that this algorithm is different to (\ref{eq:irlstandard}) due to the term $\belief(\eth_k)$ in the denominator,
 which makes a  crucial difference. Indeed, unlike  (\ref{eq:irlstandard}), algorithm  (\ref{eq:irl2}) generates samples from (\ref{eq:stationary1}), as we now explain:
By stochastic averaging theory arguments as $\stepa$ goes to zero,
the interpolated processes $  \eth^\stepa$ converges weakly to  (where $\bm(t)$ below is standard Brownian motion)
\begin{equation} \begin{split}
    d\eth(t) &=   \int_{\reals^\thdim}\kerneln\bigl(\frac{\th - \eth}{\kernelstep}\bigr)\, \biggl[\frac{\temperature}{2\,\belief(\eth)} \, \nabla_\th \Reward(\th)\, dt \biggr] \,\belief(\th)  \,d\th +   d\bm(t) , \quad \eth(0) = \eth_0
        \end{split}
      \end{equation}
          Again   as $\kernelstep \rightarrow 0$, $\kernel(\cdot)$ mimics a Dirac delta function and so  the
     $\belief(\cdot)$ in the numerator and denominator cancel out. Therefore
the      asymptotic dynamics become the  reversible diffusion
     \beq
     d\eth(t) = \frac{\temperature}{2}\, \nabla_\eth \Reward(\eth)\, dt  +  d\bm(t),  \quad \eth(0) = \eth_0 \label{eq:ld2}
     \eeq
     Note that (\ref{eq:ld2}) is the classical Langevin diffusion and has stationary distribution
     $\stat$ specified by (\ref{eq:stationary1}). So algorithm  (\ref{eq:irl2}) asymptotically  generates samples from (\ref{eq:stationary1}).

Finally, we note that Algorithm  (\ref{eq:irl2}) can be viewed  as a special case of IRL algorithm (\ref{eq:irl}) since its  limit dynamics (\ref{eq:ld2})
     is a special case of the limit dynamics (\ref{eq:gld}) with $\belief(\cdot) = 1$.

     \subsubsection{Multikernel Adaptive IRL for Variance Reduction in High Dimensions} \label{sec:highdim}
     For large dimensional problems (e.g., $\thdim=124$ in the numerical example of Sec.\ref{sec:numerical}), the passive IRL algorithm (\ref{eq:irl}) can take a very large number of iterations to converge to its stationary distribution.
This is because with high probability, the kernel $\kernel(\th_k,\eth_k)$ will be close to zero and so updates of $\eth_k$  will occur very rarely.

There is strong motivation to introduce variance reduction in the algorithm. Below we propose a
two time step, multi-kernel variance reduction IRL algorithm motivated by importance sampling.
Apart from the ability to deal with high dimensional problems, the algorithm also does not require
 knowledge of the initialization probability density $\belief(\cdot)$.

Suppose the IRL operates at a slower time scale than the RL algorithm.
At each time $k$ (on the slow time scale), by observing the RL  algorithm, the IRL obtains a  pool of samples
of the
gradients $  \nabla_\th \reward_\dtime(\th_{\dtime,i}) $ evaluated  at a large number of points $\th_{\dtime,i}$,
$i=1,2,\ldots,\numparticles$ (here $i$ denotes the fast time scale).
As previously, each sample $\th_{\dtime,i}$ is chosen randomly from $\belief(\cdot)$.
     Given these sampled derivatives, we propose the following multi-kernel IRL algorithm:
\beq
\boxed{\begin{split}
  \eth_{\dtime+1} &= \eth_\dtime + \stepa \,\frac{\temperature}{2} \, \frac{\sum_{i=1}^\numparticles   \pdf(\eth_\dtime|\th_{\dtime,i}) \nabla_\th \reward_\dtime(\th_{\dtime,i}) }{\sum_{l=1}^\numparticles \pdf(\eth_\dtime|\th_{\dtime,l})}
  + \sqrt{\stepa}  \noise_\dtime , \quad \th_{\dtime,i} \sim \belief(\cdot)
     \end{split}} \label{eq:mcmcirl}
     \eeq
      In (\ref{eq:mcmcirl}),  we choose the conditional probability density function $\pdf(\th|\eth)$ as follows:
      \begin{equation}
        \label{eq:cond-al}
        \pdf(\eth|\th) = \pdf_\obsnoise(\th - \eth)\quad \text{ where } \pdf_\obsnoise(\cdot) = \normal(0,\sigma^2I_\thdim).
      \end{equation}
      For notational convenience, for each $\eth$, denote the normalized weights in (\ref{eq:mcmcirl})  as
\beq  \weight_{\dtime,i}(\eth)  =   \frac{\pdf(\eth|\th_{\dtime,i}) }{\sum_{l=1}^\numparticles 
\pdf (\eth| \theta_{k,l})}
\quad i = 1,\ldots, \numparticles  \eeq
     Then these
     $\numparticles $ normalized weights  qualify  as
   symmetric kernels in the sense of (\ref{eq:kernel_properties}). Thus IRL algorithm (\ref{eq:mcmcirl}) can be viewed as  a multi-kernel  passive stochastic approximation
   algorithm. Note that the algorithm does not require knowledge of $\belief(\cdot)$.

Since for each $\dtime$,  the samples $\{\th_{\dtime,i}, i=1,\ldots,\numparticles\}$ are generated i.i.d. random variables,
it is well known from self-normalized importance sampling \cite{CMR05} that as  $\numparticles\rightarrow \infty$, then for fixed $\eth$,
\beq   \sum_{i=1}^\numparticles \weight_{\dtime,i}(\al) \,
\nabla_\th \reward_\dtime(\th_{\dtime,i}) \rightarrow \E\{\nabla_\th \reward_\dtime(\th) | \eth \}
\quad \text{ w.p.1,}
   \label{eq:sir}
   \eeq
   provided $\E| \pdf(\th|\eth) \,\nabla \reward_\th(\th)| < \infty$.
Similar results can also be established more generally if
$\{\th_{\dtime,i}, i=1,\ldots,\numparticles\}$ is a geometrically  ergodic Markov process with stationary distribution $\belief(\cdot)$.

{\em    Remark}: Clearly  the conditional expectation $ \E\{\nabla_\th \reward_\dtime(\th) | \eth_\dtime \}$ always has smaller variance than
$ \nabla_\th \reward_\dtime(\th)$; therefore variance reduction is achieved in IRL algorithm
(\ref{eq:mcmcirl}).
In sequential Markov chain Monte Carlo  (particle filters), to avoid degeneracy, one resamples from the pool of ``particles''  $\{\th_i,i=1\ldots,\numparticles\}$ according to the probabilities (normalized weights) $\weight_i$.  For large $\numparticles$, the resulting resampled particles have a density $\pdf(\th|\eth_k)$.
However,   we are only interested in computing an estimate of the gradient (and not in propagating particles  over time). So we use the estimate  $\sum_i \gamma_{\dtime,i}  \nabla_\th \reward_\dtime(\th_{\dtime,i})$ in (\ref{eq:mcmcirl}); this always has a smaller variance than resampling and then estimating the gradient; see \cite[Sec.12.6]{Ros13} for an elementary proof.

Why not use the popular MCMC tool of  sequential importance sampling with resampling?   Such a process resamples from the pool of particles and  pastes together components of $\th_i$ from other more viable candidates $\th_j$. As a result, $\numparticles$ composite vectors are obtained, which are more viable. However, since our IRL framework  is passive, this is of no  use since we cannot obtain  the gradient for these $\numparticles$ composite vectors. Recall that in our passive framework, the IRL
has no control over where the gradients $\nabla_\th \reward_k(\th)$ are evaluated.

{\em Informal Analysis of IRL algorithm (\ref{eq:mcmcirl})}.
By stochastic averaging theory arguments as $\stepa$ goes to zero, the interpolated process
  $\eth^\stepa$ from IRL algorithm (\ref{eq:mcmcirl}) converges weakly to
  \beq \label{eq:avgmcmc}
  d \eth(t) = \int_{\reals^\thdim} \frac{\temperature}{2}\,\nabla_\th \Reward(\th) \, \pdf\big(\th| \eth(t)\big) \, d\th \, dt + d\bm(t)  , \qquad \eth(0) = \eth_0
  \eeq
  where $\bm(t)$ is standard Brownian motion.
  Notice that even though $\th_i$ are sampled from the density $\belief(\cdot)$, the above averaging is w.r.t. the conditional density $\pdf(\th|\alpha)$ because  of
  (\ref{eq:sir}).
  For small variance $\sigma^2$, by virtue of the classical  Bernstein von-Mises theorem \cite{Vaa00},
    the conditional density  $\pdf(\th| \eth)$ in
  (\ref{eq:avgmcmc}) acts as a Dirac delta yielding the classical Langevin diffusion
  \beq \label{eq:classicaldiffusion}
  d \eth(t) =\frac{\temperature}{2}\, \nabla_\eth \Reward(\eth(t)) \, dt + d\bm(t)  \eeq
  Therefore  algorithm  (\ref{eq:mcmcirl})  generates samples from distribution (\ref{eq:stationary1}). The formal proof is in \cite{KY21}.

  \subsubsection{Active IRL with Mis-specified Gradient} \label{sec:activeirl}

Thus far we have considered the case where  the RL algorithm provides estimates $\nabla_\th \reward_\dtime(\th_\dtime)$  at randomly chosen points
 independent of the IRL estimate $\eth_\dtime$. In other words, the IRL is passive and has no role in determining where the RL algorithm evaluates gradients.

We now consider
 a modification where  the RL algorithm gives a noisy version of the gradient evaluated
 at a stochastically perturbed value of $\eth_\dtime$.  That is, when the IRL   estimate is $\eth_\dtime$, it  requests the RL algorithm to provide a  gradient estimate
 $\nabla_\th \reward_\dtime(\eth_\dtime)$. But the RL algorithm evaluates the gradient at a mis specified point
 $\th_\dtime = \eth_\dtime + \obsnoise_\dtime$, namely,   $\nabla_\th \reward_\dtime(\eth_\dtime+ \obsnoise_\dtime)$. Here $\obsnoise_\dtime \sim \normal(0,\sigma^2 I_{\thdim})$ is an i.i.d. sequence.  The RL algorithm then provides the IRL algorithm with $\th_\dtime$ and $\nabla_\th \reward_\dtime(\th_\dtime)$. So,  instead of $\th_\dtime$ being independent of $\eth_\dtime$, now  $\th_\dtime$ is conditionally dependent on $\eth_\dtime$ as
 \beq \pdf(\th_\dtime|\eth_\dtime) = \frac{1}{(2 \pi)^\thdim\,\sigma^{\thdim}} \exp( - \frac{1}{2 \sigma^2} \|\th_k - \eth_k\|^2), \qquad \th_k,\eth_k \in \reals^\thdim
 \label{eq:conditional}
 \eeq
 In other words, the IRL now actively specifies where to evaluate the gradient; however, the RL algorithm evaluates a noisy gradient  and that too at  a stochastically perturbed (mis-specified) point~$\th_k$.

 The active  IRL algorithm we propose is  as follows:
 \beq \label{eq:activeirl}
 \begin{split}
   \eth_{k+1} &= \eth_k + \stepa\, \frac{1}{\kernelstep^\thdim}\, \kernel(\frac{\th_k-\eth_k}{\kernelstep})
   \frac{\temperature}{2\, \pdf(\th_k|\eth_k)}\, \nabla_{\th} \reward_k(\th_k) + \sqrt{\stepa} \noise_k , \quad
 \text{ where }   \th_k = \eth_k + \obsnoise_k
\end{split}
\eeq
The  proof of convergence again follows using averaging theory arguments. Since $\{\th_k\} \sim \pdf(\th|\eth_k)$ is the fast signal and $\{\eth_k\}$ is the slow signal, the averaged system  is
$$ d\eth(t) =  \int_{\reals^\thdim} \kerneln\bigl(\frac{\th - \eth}{\kernelstep}\bigr)\ \frac{\temperature}{2\,\pdf(\th|\eth(t))}\,\nabla_\th \Reward(\th)\, {\pdf(\th|\eth(t))} \,d\th\, dt + d\bm(t)  $$
So the $\pdf(\th|\eth(t))$ cancel out in the numerator and denominator.  As $\kernelstep \rightarrow 0$,
the kernel acts as
a  Dirac delta thereby yielding the classical Langevin diffusion (\ref{eq:classicaldiffusion}).

{\em Remark}: The active IRL algorithm (\ref{eq:activeirl}) can be viewed as an idealization of the multi-kernel IRL algorithm (\ref{eq:mcmcirl}). The multi-kernel algorithm  constructs weights to  approximate  sample from the conditional distribution $\pdf(\th|\eth)$. In comparison, the active IRL has direct measurements from this conditional density. So the active IRL can be viewed as an upper bound to the performance of the multi-kernel IRL
Another motivation is inertia. Given the dynamics of the RL algorithm, it may not be possible to the RL to abruptly  jump to evaluate a gradient at $\eth_k$, at best the RL can only evaluate a gradient at a point $\eth_k + \obsnoise_k$. A third motivation stems from mis-specification:  if the IRL represents a machine (robot) learning from a human, it is difficult to specify to the human exactly what policy
$\eth_\dtime$  to    perform. Then $\th_k = \eth_k + \obsnoise_k$ can be viewed as an approximation to this  mis-specification.

     \subsubsection{Non-reversible Diffusion for IRL} So  far we have defined four  different  passive Langevin dynamics algorithms for IRL, namely (\ref{eq:irl}),
 (\ref{eq:irl2}),  (\ref{eq:mcmcirl}), and (\ref{eq:activeirl}).
These algorithms
yield  reversible diffusion processes that asymptotically sample from the stationary distribution
(\ref{eq:stationary1}).
  It is well  known \cite{HHS93,HHS05,Pav14} that adding a skew symmetric matrix to the gradient always improves the convergence rate of Langevin dynamics to its stationary distribution.
  That is for any $\thdim\times \thdim$ dimensional  skew symmetric matrix $S = -S^\p$, the non-reversible diffusion process
  \beq
   d\eth(t) = \frac{\temperature}{2} \, (I_\thdim + S) \, \nabla_\eth \Reward(\eth)  dt  +   d \bm(t) , \quad \eth(0) = \eth_0
 \label{eq:skew1}
\eeq
has a larger spectral gap and therefore converges to stationary distribution $\belief(\eth)$ faster than (\ref{eq:gld}).  The resulting IRL algorithm obtained by a  Euler-Maruyama time discretization of (\ref{eq:skew1})  and then introducing a kernel $\kernel(\cdot)$ is
  \begin{equation}
 \boxed{  \eth_{\dtime+1} = \eth_\dtime +  \stepa  \, \kerneln\big(\frac{\th_\dtime-\eth_\dtime}{\kernelstep}\bigr)\,
 \frac{\temperature\,(I_\thdim+ S)}{2\, \belief(\eth_\dtime)}\, \nabla_\th \reward_\dtime(\th_\dtime) +  \sqrt{\stepa}\,  \noise_\dtime,
    \quad \dtime =1,2,\ldots }
\label{eq:irl3}
\end{equation}
initialized by  $\eth_0 \in \reals^\thdim$.
Again  a stochastic  averaging theory argument  shows that IRL algorithm (\ref{eq:irl3}) converges weakly to the
non-reversible diffusion (\ref{eq:skew1}).
In numerical examples, we  found empirically that the convergence of (\ref{eq:irl3})  is  faster than (\ref{eq:irl}) or~(\ref{eq:irl2}). However,
the faster convergence comes at the expense of  an order of magnitude increased computational cost. The  computational cost of IRL algorithm  (\ref{eq:irl3}) is $O(\thdim^2)$ at each iteration due to  multiplication with  skew symmetric matrix $S$. In comparison  the computational costs of IRL
algorithms (\ref{eq:irl}) and (\ref{eq:irl2}) are each $O(\thdim)$.

\section{Numerical Examples} \label{sec:numerical}
This section presents two examples to illustrate the performance of the proposed IRL algorithms.

\subsection{Example 1. IRL  for Bayesian KL divergence and Posterior Reconstruction} \label{sec:bayesian}

 This section illustrates the performance of our proposed IRL algorithms in reconstructing the Kullback Leibler (KL) divergence and multi-modal posterior distribution. Our formulation is  a stochastic  generalization of  adaptive Bayesian learning in \cite{WT11} as  explained below.

{\bf Motivation.} Exploring and estimating the KL divergence of a multimodal posterior distribution is important in Bayesian inference \cite{RC13},  maximum likelihood estimation, and also stochastic control with KL divergence cost \cite{GRW14}.
To motivate the problem,
suppose random variable  $\th$ has prior probability density $ \pdf(\th)$.
Let $\thtrue$ denote a fixed (true) value of $\th$ which is unknown to the optimizing agents and inverse learner.
Given a sequence of observations
$\obs_{1:\horizon} = (\obs_1,\ldots,\obs_\horizon)$,  generated from distribution $\pdf(\obs_{1:\horizon}|\thtrue)$, the KL divergence of the posterior distribution is
\beq \label{eq:kldiv}
 J(\thtrue,\th)=
\E_\thtrue\{ \log \pdf(\thtrue| \obs_{1:\horizon})  -
 \log \pdf(\th| \obs_{1:\horizon})
\}  = \int
\log \frac{\pdf(\thtrue|\obs_{1:\horizon})}
{\pdf(\th|\obs_{1:\horizon})}
\, \pdf(\obs_{1:\horizon}|\thtrue) d\obs_{1:\horizon}
\eeq
It is well known  (via Jensen's inequality) that the global minimizer $\th^*$ of  $J(\thtrue,\th)$ is
$\thtrue$. Therefore minimizing the KL divergence yields a consistent estimator
of $\thtrue$. Moreover, under mild stationary conditions,   when the prior is non-informative (and so possibly improper), the Shannon-McMillan-Breiman theorem \cite{Bar85} implies that
the global minimizer of the  KL divergence converges with probability 1 to the maximum likelihood estimate as $\horizon \rightarrow \infty$. So there is strong motivation to explore and estimate the KL divergence.

Typically the KL divergence $J(\thtrue,\th)$ is non-convex in $\th$. So we are in the non-convex optimization setup of (\ref{eq:rl})  where multiple agents seek to estimate the global
minimizer of the KL divergence.

\subsubsection{Model Parameters}

We consider a stochastic optimization problem where a RL system chooses actions
$\action_k$
from randomized policy $\pdf(\th|\obs_{1:\horizon})$. In order to learn the optimal policy, the RL system aims to
estimate the global minimizer
$ \th^* = \argmin_\th J(\thtrue,\th) $; see for example \cite{GRW14} for motivation of KL divergence minimization in stochastic control.
Then by observing the  gradient estimates of the RL agents, we will use our proposed passive
IRL algorithms to reconstruct the KL divergence.

Ignoring the constant term  $\pdf(\thtrue| \obs_{1:\horizon}) $ in (\ref{eq:kldiv}),
minimizing $J(\thtrue,\th)$ wrt $\th$ is equivalent to maximizing the relative
entropy $\Reward(\th) = \E_{\thtrue}\{ \log \pdf(\th|\obs_{1:\horizon})\}$.
So  multiple RL agents aim to 
solve the following  non-concave stochastic maximization problem: Find
\beq
\th^* = \argmax_\th  \Reward(\th), \quad \text{ where } \Reward(\th)  = \E_{\thtrue}\{ \log \pdf(\th|\obs_{1:\horizon})\}
\label{eq:relative}
\eeq
In our numerical example we choose  $\th =[\th(1),\th(2)]^\p\in \reals^2$  and $\thtrue$ is  the true parameter value which is  unknown to the learner. The prior is $\pdf(\th) = \normal(0,\Sigma)$ where $\Sigma=\diag[10,2]$. The observations $\obs_\dtime$ are independent and generated from the multi-modal mixture likelihood
$$ \obs_\dtime \sim \pdf(\obs|\thtrue) = \frac{1}{2} \normal(\thtrue(1), 2) + \frac{1}{2} \normal(\thtrue(1) + \thtrue(2), 2) $$
Since $\obs_1,\ldots,\obs_\horizon$ are independent and identically distributed,
 the objective $\Reward(\th)$ in (\ref{eq:relative})  is
 \beq  \Reward(\th) =  \E_\thtrue\{ \log \pdf(\th) +  \horizon \, \log \pdf(\obs|\th) \} + \text{ constant indpt of $\th$}
 \label{eq:jtheta}
\eeq
For  true parameter value  $\thtrue=[0, 1]^\p$, it can be verified that
the
objective  $\Reward(\th)$ is non-concave and has two maxima at
$\th = [0, 1]^\p$ and $\th=[1, -1]^\p$.

\subsubsection{Classical Langevin Dynamics}
To  benchmark the performance of our passive  IRL algorithms (discussed below),
we  ran the classical Langevin dynamics algorithm:
\begin{equation} \label{eq:classical_langevin}
   \th_{\dtime+1} = \th_\dtime + \stepa\,  \frac{\temperature}{2}  \nabla_\th \reward_\dtime(\th_\dtime)
   + \sqrt{\stepa}\,   \noise_\dtime,
   \quad \dtime =1,2,\ldots,
\end{equation}
Note that the classical Langevin dynamics \eqref{eq:classical_langevin}
evaluates the gradient estimate $\nabla_\th \reward_\dtime(\th_\dtime)$  unlike our passive IRL algorithm which has no control of where the gradient is evaluated.
Figure \ref{fig:classicalLang}  displays
both the empirical histogram and a contour plot of the estimate $\Reward(\th)$ generated by classical Langevin dynamics. The classical Langevin dynamics can be viewed as an upper bound for the performance of our passive IRL algorithm; since our passive algorithm cannot specify where the gradients are evaluated.

\begin{figure}[h] \centering
 \mbox{\subfigure[]{
    \includegraphics[scale=0.5]{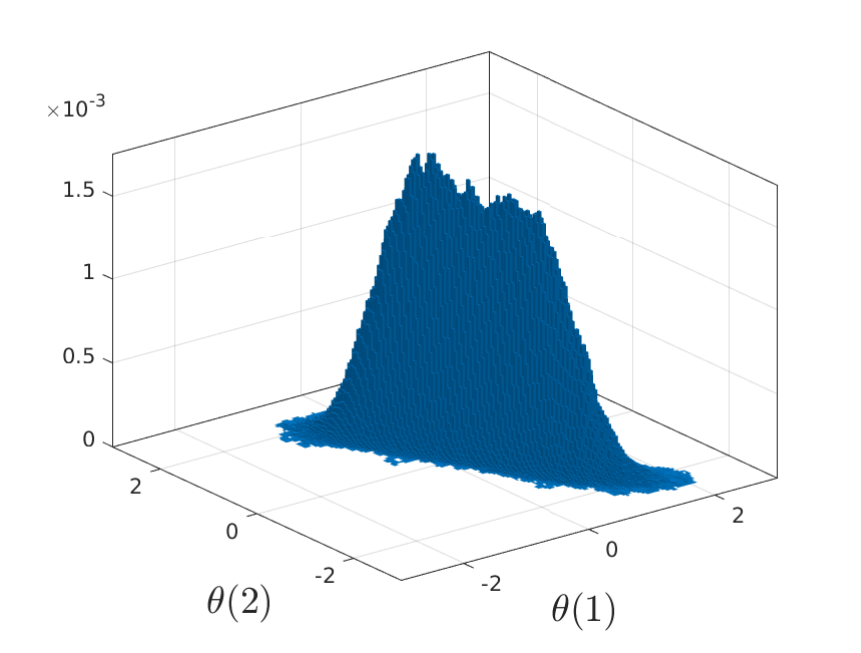}}
   \subfigure[]{\includegraphics[scale=0.5]{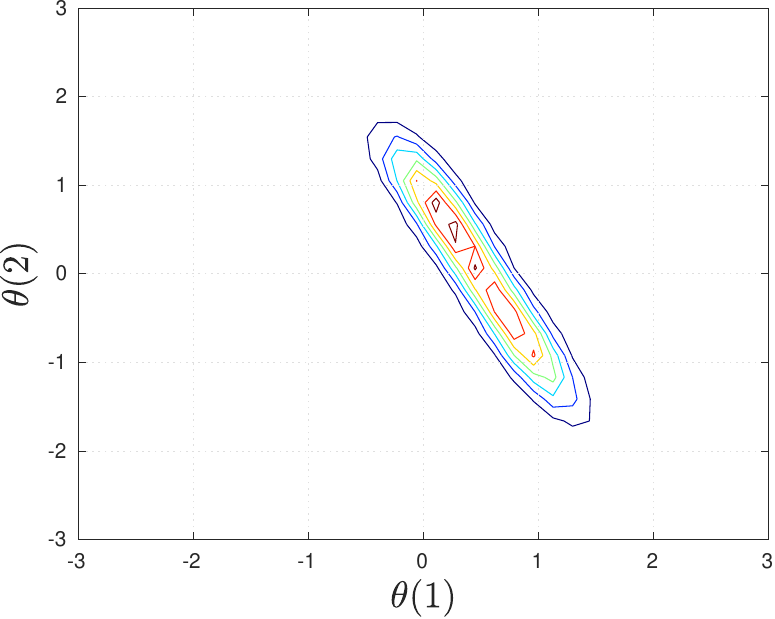}
}}
  \caption{Classical Langevin dynamics (ground truth)}
  \label{fig:classicalLang}
\end{figure}

\begin{figure}[h] \centering
  \mbox{\subfigure[]{
     \includegraphics[scale=0.5]{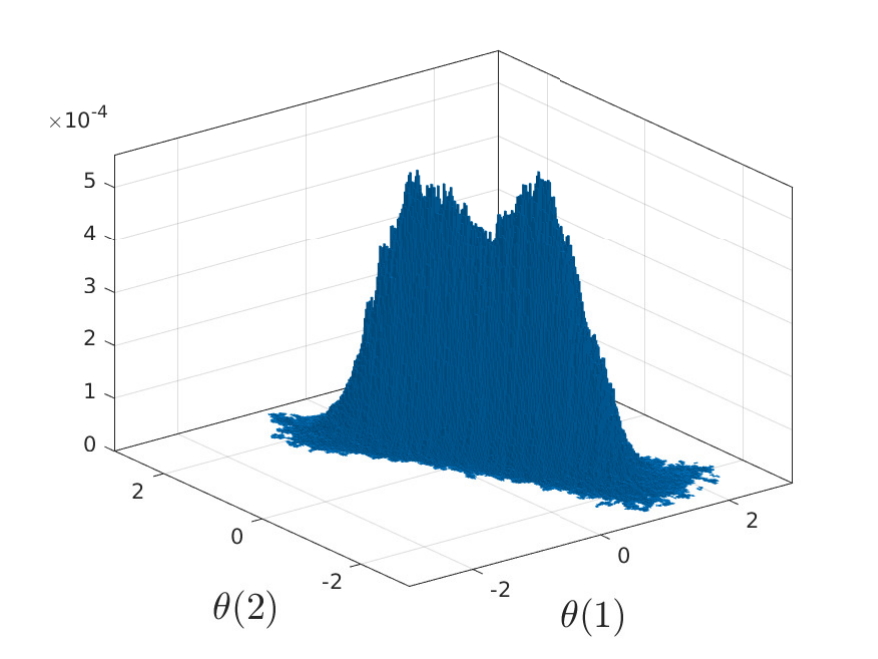}}
  \subfigure[]
        {\includegraphics[scale=0.5]{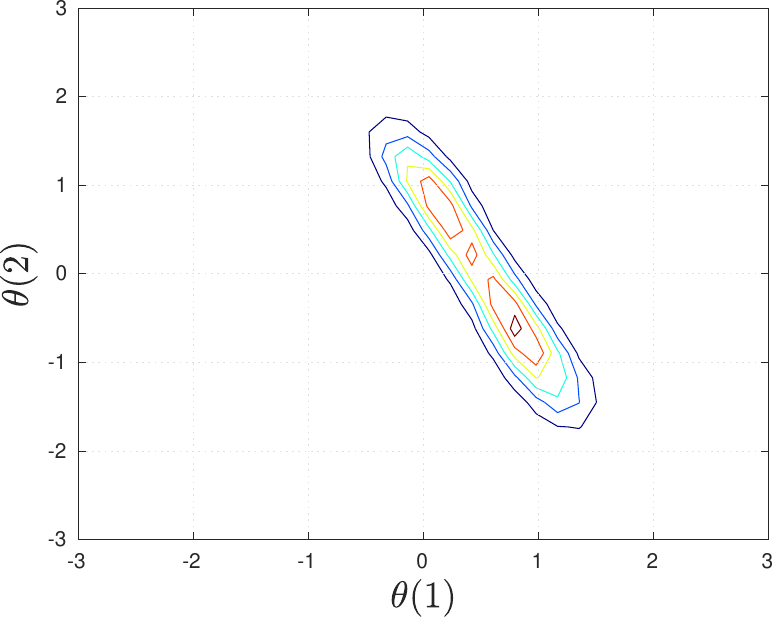}}}
  \caption{IRL Algorithm (\ref{eq:irl})}
  \label{fig:mh}
\end{figure}

\begin{figure}[h]
   \mbox{\subfigure[]{
    \includegraphics[scale=0.5]{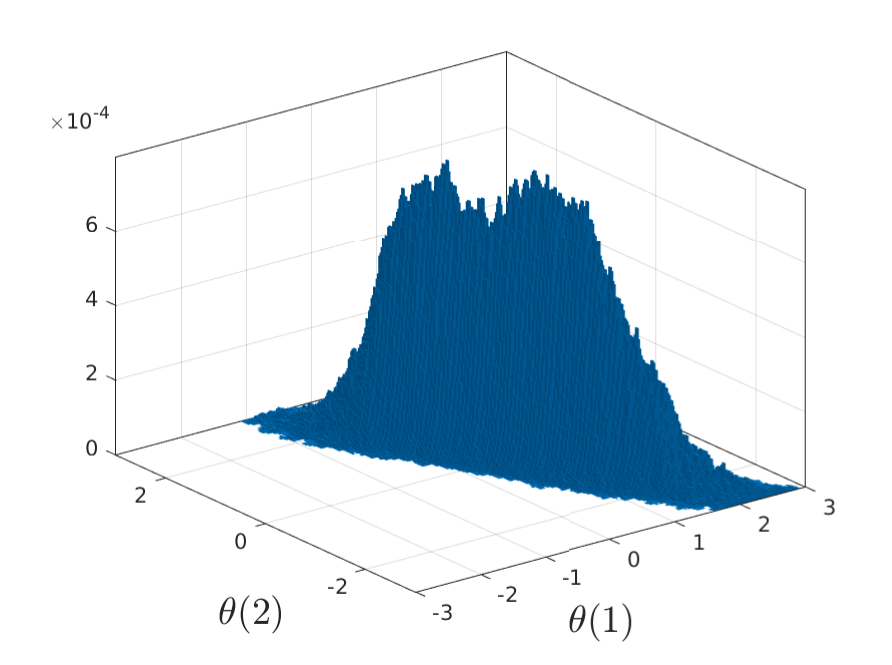}}
 \subfigure[]
    {\includegraphics[scale=0.5]{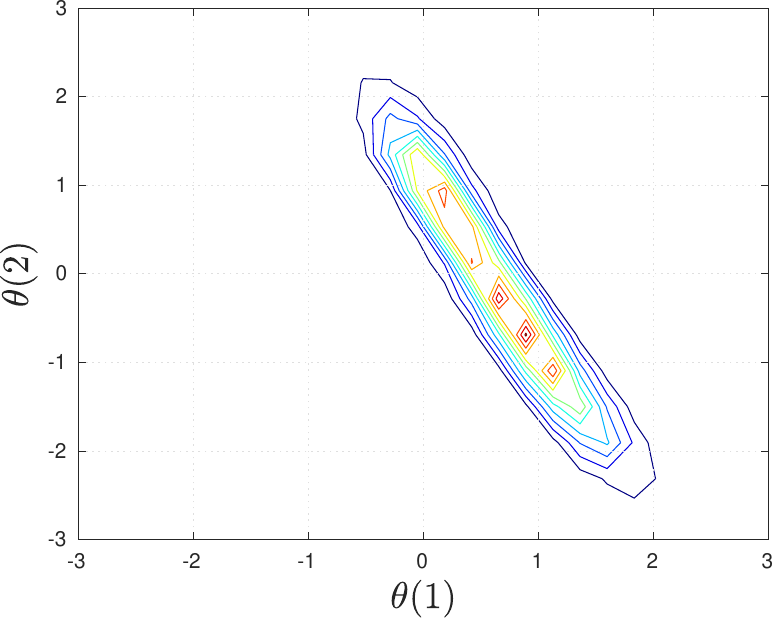}}}
  \caption{Two time scale multi-kernel IRL Algorithm (\ref{eq:mcmcirl})}
  \label{fig:mh2}
\end{figure}

\subsubsection{Passive IRL Algorithms}

We now illustrate the performance of our proposed passive IRL algorithms for the above model.
Recall that the framework comprises two parts: First, multiple RL agents run randomly initialized stochastic gradient algorithms to maximize $\Reward(\th)$.  Second, by observing these gradients, our passive IRL Langevin based algorithms construct a non-parametric estimate of the $\Reward(\th)$. We discuss these two parts below:

{\em 1. Multiple agent Stochastic Gradient Algorithm}.
Suppose  multiple RL agents aim to learn the optimal policy by estimating the optimal parameter $\th^*$. To do so, the  agents use the stochastic gradient algorithm (\ref{eq:rl}):
\begin{equation} \label{eq:rlsim}  \begin{split}
\th_{\dtime+1} &= \th_\dtime + \step \nabla_\th \reward_\dtime(\th_\dtime)
\\
\nabla_\th \reward_\dtime(\th_\dtime) &=   \nabla_\th \log  \pdf(\th_\dtime) + \horizon\, \nabla_\th \log \pdf(\obs_\dtime|\th_\dtime)
\end{split}
\end{equation}
with multiple random initializations, depicted by agents $\tslow=1,2,\ldots$.
For each agent $\tslow$, the initial estimate  was sampled randomly as $\th_{\stoptimeirl_\tslow} \sim \belief(\cdot) =  \normal(0,I_{2\times 2})$. Each agent runs the gradient algorithm for 100 iterations with step size  $\step = 10^{-3}$ and the number of agents is  $10^5$.  Thus the sequence $\{\th_\dtime;\dtime=1,\ldots10^7\}$ is generated. \\

{\em 2. IRL algorithms and performance}.
 Given   the sequence of estimates $\{\th_\dtime\}$ generated by the RL agents above, and initialization density $\belief$,  the inverse learner aims to estimate $\Reward(\th)$  in (\ref{eq:jtheta}) by
 generating samples $\{\eth_k\}$ from
 $\exp(\temperature \Reward(\th))$. Note that the IRL algorithm
 has no knowledge of $\pdf(\th)$ or $\pdf(\obs|\th)$.
 Since the inverse learner has no control of where the reinforcement learner evaluates its gradients,  we are in passive IRL setting.
We ran
  the IRL algorithm (\ref{eq:irl}) with kernel
$ \kernel (\th,\eth)\propto \exp(- \frac{\|\eth- \th\|^2}{0.02})$, step size $\stepa = 5 \times 10^{-4}$, $\temperature=1$.
Figure \ref{fig:mh} displays both the empirical histogram and a contour plot.
Notice that the performance of our IRL is  very similar to classical Langevin dynamics (where the gradients are fully specified).

We compared the performance of the classical Langevin with the passive Langevin IRL algorithm averaged over 100 independent runs. The comparison is with respect to the variational distance\footnote{Recall the variational distance is half the $L_1$ norm} $d(1)$ and $d(2)$ between the two marginals of the empirical
density $\pdf(\th)  \propto \exp(\Reward(\th))$. The values obtained from
our simulations
are
\beq d(1) = 0.0122, \quad d(2)= 0.0202 . \label{eq:l1dist}\eeq

Finally,  we illustrate the performance of the two-time scale multikernel algorithm (\ref{eq:mcmcirl}). Recall this algorithm does not require knowledge of the initialization probabilities $\belief(\cdot)$. Figure \ref{fig:mh2} displays
both the empirical histogram and a contour plot. Again the performance of the IRL is  very similar to the classical Langevin dynamics performance.

\subsubsection{Multiple Inverse Learners}
We also considered the case where  multiple inverse learners act in parallel. Suppose each inverse learner $l \in \{1,2,\ldots, L\}$ deploys IRL algorithm (\ref{eq:irl})
with its own noise  sample path denoted by $\{\noise_k^{(l)}\}$, which is independent of
other inverse learners. Obviously, if the estimate $\eth_k^{(l)}$ of one of the  inverse learners
(say $l$) is close to
$\th_k$, then $\nabla_\th \reward_\dtime(\th_k)$ is  a more accurate gradient estimate
for $\nabla_\th \reward_\dtime(\eth_k^{(l)})$. However, for high dimensional problems, our numerical experiments (not presented here) show  very little benefit unless the number of inverse learners is chosen as $L = O(2^\thdim)$ which is intractable.

\subsubsection{IRL for Adaptive   Bayesian Learning} \label{sec:irlabl}

Having discussed reconstructing the KL divergence via IRL, we now   extend  the
Bayesian learning framework proposed in \cite{WT11}  to our IRL framework.

{\bf Bayesian Learning}. First a few words about the Bayesian learning framework in \cite{WT11}.
In comparison to the stochastic optimization problem (\ref{eq:relative}),
 they consider
 a \textit{fixed} sample path $\obs_{1:\horizon}$ and  the associated {\em deterministic} optimization problem of finding  global maximizers of
\beq  \label{eq:rllog1}  \Reward(\th) =  \log \pdf(\th|\obs_{1:\horizon}). \eeq
 \cite{WT11}  use the
 classical Langevin dynamics to generate samples from
 the posterior  $\pdf(\th|\obs_{1:\horizon})$ as follows:
First,
since $\obs_1,\ldots, \obs_\horizon$ are independent,
\beq  \nabla_\th \log \pdf(\th|\obs_{1:\horizon}) \propto   \nabla_\th \log  \pdf(\th) +  \sum_{\dtime=1}^\horizon \nabla_\th \log \pdf(\obs_{\dtime}|\th)
\label{eq:gradlang}
\eeq
Next it is straightforward to see that $\horizon$ iterations of the  classical Langevin algorithm (or a fixed step size deterministic gradient ascent algorithm) using the  gradient  $\nabla_\th \log  \pdf(\th) +
\horizon\, \nabla_\th \log \pdf(\obs_{k}|\th)$  is identical to running $\horizon$ sweeps
of the  algorithm through the sequence $\obs_{1:\horizon}$ with gradient (\ref{eq:gradlang}).
So \cite{WT11} run the classical Langevin algorithm using the  gradient  $$\nabla_\th \log  \pdf(\th) +
\horizon\, \nabla_\th \log \pdf(\obs_{k}|\th). $$
  Notice unlike the KL estimation framework (\ref{eq:relative}) which has an expectation $\E_\thtrue$ over the observations,   the underlying optimization of $\log \pdf(\th|\obs_{1:\horizon})$
 is {\em deterministic}
 since we have a  fixed sequence  $\obs_{1:\horizon}$.
Then clearly  the Langevin dynamics generates samples from the stationary distribution
 \beq \belief(\th) = \exp\big( \log (\pdf(\th| \obs_{1:\horizon}) ) \big) = \pdf(\th| \obs_{1:\horizon}) \label{eq:wt11}\eeq
  namely, the posterior distribution.\footnote{This is in contrast to our KL divergence estimation setup
    \eqref{eq:jtheta} where the stationary distribution is  $\belief(\th) = \exp\big( \E_\thtrue \{\log \pdf(\th| \obs_{1:\horizon}) \}\big) $ and
    $\E_\thtrue $ denotes expectation wrt $\pdf(\obs_{1:\horizon}|\thtrue)$.}
So  the classical Langevin algorithm which  sweeps repeatedly through the dataset $\obs_{1:\horizon}$ generates samples
  from the posterior distribution - this is the main idea of Bayesian learning in \cite{WT11}.

  {\bf IRL}.
 We now consider  IRL in this Bayesian learning framework to reconstruct the posterior density.
Given the sample path $\obs_{1:\horizon}$,
suppose multiple forward learners seek to estimate the maximum (mode) of the multimodal posterior $\log \pdf(\th|\obs_{1:\horizon})$.
The agents run the (deterministic) gradient ascent algorithm (\ref{eq:rl})
 with gradient
 $$  \nabla_\th \reward_\dtime(\th_\dtime) = \nabla_\th \log  \pdf(\th_\dtime) + \horizon\, \nabla_\th \log \pdf(\obs_\dtime|\th_\dtime)$$
The IRL problem we consider is:  By passively observing these gradients, how can the IRL algorithm reconstruct the posterior
 distribution  $\pdf(\th|\obs_{1:\horizon})$?
We use  our IRL algorithm (\ref{eq:irl}).
  The implementation of  IRL algorithm (\ref{eq:irl}) follows the   \cite{WT11} setup: The RL agents choose random initializations $\th_0 \sim \belief$ and  then run gradient algorithms  sweeping repeatedly through the dataset $\obs_{1:\horizon}$. The IRL algorithm (\ref{eq:irl}) passively views these estimates $\{\th_k\}$ and  reconstructs the posterior distribution  $p(\th| \obs_{1:\horizon})$  from these estimates.

  We now illustrate the performance of IRL algorithm (\ref{eq:irl}) in this Bayesian learning setup.
For the same parameters as in the example above (recall $\horizon=100$),   Table \ref{tab:irl} compares the performance of the classical Langevin
  algorithm and our IRL algorithm with the Metropolis Hastings sampler.
  The Metropolis Hastings sampler can be considered as the ground truth
  for the posterior  $p(\th| \obs_{1:100})$.

  \begin{table}[h]
    \centering
  \begin{tabular}{|c|c|c|} \hline
    &  Classical Langevin &  Passive IRL Langevin \\ \hline
      $    d(1) $ & 0.0213  & 0.0264 \\
      $d(2) $ &  0.0229 &  0.0305 \\                   \hline
  \end{tabular}
  \caption{Variational distance between marginals and Metropolis Hastings
    algorithm}  \label{tab:irl}
\end{table}

\subsection{Example 2. IRL for Constrained Markov Decision Process (CMDP)}  \label{sec:mdp}
In this section we illustrate the performance of the IRL algorithms for reconstructing the cumulative reward of a constrained Markov decision process (CMDP) given gradient information from a RL algorithm.
This is in contrast to classical IRL \cite{NR00} where the transition matrices of the MDP are assumed known to the inverse learner.

Consider  a unichain\footnote{By {\em unichain} \cite[pp.~348]{Put94} we mean that every policy where $\action_n$  is a deterministic function of $\state_n$ consists of a
single recurrent class  plus possibly an empty set of transient states.} average reward CMDP  $\{\state_n\}$ with finite state space $\statespace
= \{1,\ldots,\statedim\}$  and  action space $\actionspace= \{1,2,\ldots,\actiondim\}$. The CMDP evolves with
transition probability matrix $\tp(\action)$ where
\begin{equation}
\label{Aij}
\tp_{ij}(\action) \ole \prob[\state_{n+1} = j | \state_n = i, \action_n = \action],  \quad u \in \actionspace.
\end{equation}
When the system is in state $\state_n\in \statespace$, an action $\action_n = \policy(\state_n) \in \actionspace$ is chosen, where $\policy$ denotes (a possible randomized)  stationary policy.
The reward incurred at stage $n$ is  $\mreward(\state_n,\action_n)\geq 0$.

Let $\admissible$ denote the class of stationary randomized Markovian policies.
For any stationary  policy $\policy \in \admissible$, let $ \E_{\policy}$ denote the corresponding expectation and define the infinite horizon average reward
\begin{equation}
   \label{J}
J(\policy) = \lim_{\finaltime\to\infty} \inf \frac{1}{\finaltime} \E_{\policy} \Big[\sum_{n=1}^\finaltime
\mreward(\state_n,\action_n) \mid \state_0 = x\Big].
\end{equation}
Motivated by modeling fairness constraints  in network optimization  \cite{NK10},
 we consider the reward  (\ref{J}), subject to the
average constraint:
\begin{equation}
 \label{costconstraint}
  \Cons(\policy) =  \lim_{\horizon\to\infty} \inf {1\over \horizon} \E_{\policy}\Big[ \sum_{n=1}^\horizon\con(\state_n,\action_n)
\Big] \leq \rcon,
\end{equation}
(\ref{J}), (\ref{costconstraint}) constitute a CMDP. Solving a CMDP involves computing
the optimal policy $\policy^* \in \admissible$ that satisfies
\begin{equation}
J(\policy^*) = \sup_{\policy \in \admissible} J(\policy) \quad
\forall x_0 \in \statespace,  \text{ subject to } (\ref{costconstraint}) \label{eq:objective1}
\end{equation}

To solve a  CMDP,  it is sufficient to consider randomized stationary policies:
\begin{equation} \label{eq:randpol}
 \policy(\state)  =  \action \text{ with probability } \; \cond(\action|\state)= \frac{\statpi({\state,\action})}{\sum_{\tilde \action \in \actionspace} \statpi({\state,\tilde \action})} ,\end{equation}
where the conditional probabilities $\cond$ and joint probabilities $\statpi$ are defined as
\beq  \cond(\action|\state)  = \prob(\action_n =\action| \state_n = \state),  \quad \statpi(\state,\action) = \prob(\action,\state). \label{eq:conditionalprob} \eeq
Then the optimal policy $\optpolicy$ is obtained as the solution of a   linear programming problem in terms of the $\statedim \times \actiondim$ elements of $\statpi$; see \cite{Put94} for the precise equations.

Also  \cite{Alt99}, the optimal policy $\policy^*$ of the CMDP
is {\em randomized} for at most one of the states. That is,
\begin{equation}
\optpolicy(\state)  = \randmix \,\optpolicy_1(\state)  + (1 - \randmix)\, \optpolicy_2(\state)  \label{eq:randomizedt} \end{equation}
where $\randmix \in [0,1]$ denotes the randomization probability and $\optpolicy_1,\optpolicy_2$ are pure (non-randomized) policies.
Of course,
when there is no constraint   (\ref{costconstraint}),  the CMDP reduces to classical MDP and the optimal stationary  policy $\optpolicy(\state)$ is  a pure policy.  That is, for each state $\state \in \statespace$, there exists an action $\action$ such that $\cond(\action|\state)= 1$.

{\em Remarks}. (i) (\ref{costconstraint}) is a global constraint that applies to the entire sample path \cite{Alt99}. Since the optimal policy is randomized, classical value iteration based approaches and  Q-learning cannot be used to solve CMDPs as they yield deterministic policies. One can construct a Lagrangian dynamic programming formulation \cite{Alt99} and Lagrangian Q-learning algorithms \cite{DK07}. Below for brevity, we consider a  policy gradient RL algorithm.

(ii) {\em Discounted CMDPs}. Instead of an average cost CMDP,  a discounted cost CMDP can be considered. Discounted CMDPs are  less technical in the sense that an optimal policy always exists (providing  the constraint set is non-empty); whereas average cost CMDPs require a unichain assumption. It is easily shown \cite{Kri25}  that the dual linear program of a discounted CMDP can be expressed in terms of the conditional probabilities
  $\cond(\action|\state)$ and the optimal randomized policy is of the form
  (\ref{eq:randomizedt}).
  The final IRL algorithm is identical to
\eqref{eq:irl_mdp} below.

 \subsubsection{Policy Gradient for RL of  CMDP}
Having specified the CMDP model, we next turn to the RL algorithm.
RL algorithms\footnote{In adaptive control, RL algorithms such as policy gradient are   viewed as simulation based  {\em implicit} adaptive control methods that bypass estimating the MDP parameters (transition probabilities) and  directly estimate the optimal policy.} are   used to estimate the optimal policy of an MDP    when the transition matrices are not known. Then the LP formulation in terms of joint probabilities $\statpi$ is not useful since the constraints depend on the transition  matrix.
In comparison, \textit{policy gradient RL algorithms} are stochastic gradient algorithms of the form (\ref{eq:rl}) that  operate on the  conditional action probabilities $ \cond(\action| \state)$ defined in  (\ref{eq:conditionalprob})  instead of the joint probabilities $\statpi(\state,\action)$. 

Note that
 (\ref{eq:objective1}) written as a minimization (in terms of $-J$), together with constraint (\ref{costconstraint}) is in general, no longer a convex optimization problem in the variables $\cond$;
 see Figure \ref{fig:mdp} for an illustration. So it is not possible  to guarantee that simple gradient descent schemes\footnote{Consider minimizing the negative of the objective function, namely  $-J$ without constraint (\ref{costconstraint}). Even though $-J$  is nonconvex in $\cond$, one can show (using Lyapunov function arguments)  that for this unconstrained MDP case, the gradient algorithm  will converge to a global optimum. However for the constrained MDP case this is not true; the nonconvex  objective and constraints  results in a duality gap.} can achieve the global optimal policy. This motivates  the setting of (\ref{eq:rl}) where multiple agents that are initialized randomly aim to estimate the optimal policy.

Since the problem is non-convex, and the inequality constraint is active (i.e.,\ achieves equality) at the global maximum, we assume that the RL agents use a quadratic penalty method:
For $\lambda \geq 0$, denote the quadratic penalized objective to be maximized  as
\beq
\Reward(\cond) = J(\cond) - \lambda \, (\Cons^2(\cond) - \rcon)
\label{eq:penalty_reward}
 \eeq
 Such quadratic penalty functions are used widely for equality constrained non-convex problems.

The RL agents aim to minimize the $\finaltime$-horizon sample path penalized objective which at batch $k$ is
\beq
\begin{split}
  \reward_k (\cond)  &\ole  J_{k,\horizon}(\cond) + \lambda \,\Big(\Cons^2_{k,\horizon}(\cond)-\rcon\Big),  \quad \lambda \in \reals_+\\
  J_{k,\horizon} &=  \frac{1}{\finaltime} \sum_{n=1}^\finaltime
  \mreward(\state_n, \policy_\cond(\action_n)) , \quad
  \Cons_{k,\horizon} = \frac{1}{\finaltime} \sum_{n=1}^\finaltime
  \con(\state_n, \policy_\cond(\action_n))
\end{split} \label{eq:penalized}
\eeq
There are several methods for estimating the  policy gradient $\nabla_\cond \reward_\dtime(\cond_\dtime)$  \cite{Pfl96} including the score function method, weak derivatives \cite{VK03} and finite difference methods. A useful finite difference gradient estimate is given by the SPSA algorithm \cite{Spa03}; useful because SPSA evaluates the gradient along a single random direction.

\subsubsection{IRL for CMDP}
Consider the CMDP (\ref{Aij}), (\ref{J}), (\ref{eq:randpol}). Assume we are given a sequence of  gradient estimates $\{\nabla_\cond \reward_\dtime(\cond_\dtime)\}$ of the sample path wrt to the parametrized policy  $\cond$ from (\ref{eq:penalized}).  The aim of the inverse learner is to reconstruct the reward $\Reward(\cond) $ in  (\ref{eq:penalty_reward}). Since by construction the constraint is active at the optimal policy, the aim of the inverse learner is to explore regions of $\cond$ in the vicinity where the constraint $\{\cond: \Cons(\cond) \approx \rcon\}$ is active in order  to estimate  $\Reward(\cond) $.


A naive application of Langevin IRL algorithm (\ref{eq:irl}) to update the conditional probabilities $\{\cond_\dtime\}$  will not work. This is because there is no guarantee that the  estimate sequence $\{\cond_\dtime\}$ generated by the algorithm are valid  probability vectors, namely
\beq \cond_\dtime(\action|\state) \in [0,1], \quad \sum_{\action\in \actionspace} \cond_\dtime(\action|\state) = 1, \quad x \in \statespace.
\label{eq:simplex}
\eeq
We will use spherical coordinates\footnote{Another parametrization  widely used in machine learning is  exponential coordinates: $ \cond(\action|\state)  =  \frac{\exp(\param({\state, \action}))}{\sum_{a\in \actionspace} \exp(\param({\state, a}))}$,  where $ \param({\state,\action}) \in \reals$ is unconstrained. However, as shown in \cite{Kri25,KV18}, spherical coordinates typically yield faster convergence. We also found this in numerical studies on IRL (not presented here).}
to ensure that  the conditional probability estimates $\cond_\dtime$  generated by the IRL algorithm satisfies  (\ref{eq:simplex}) at each iteration $\dtime$. The idea is to parametrize $\sqrt{\cond_\dtime(\action|\state)}$ to lie on the unit hyper-sphere in $\reals^\actiondim$.  Then all needed are the $\actiondim-1$ angles for each $\state$, denoted as $\param(i,1),\ldots \param(i,\actiondim-1)$.
Define the spherical coordinates in terms of the mapping:
\begin{equation} \cond = \logistic(\param), \quad \text{ where }
      \cond(\action|\state)  =  \begin{cases} \cos^2 \param(i,1)  & \text{ if } \action =1  \\
        \cos^2 \param(i,u) \prod_{p=1}^{u-1} \sin^2 \param(i,p)  & u \in \{2,\ldots, \actiondim-1\}           \\
        \sin^2 \param(i,\actiondim-1) \prod_{p=1}^{\actiondim-2} \sin^2\param(i,p) & u = \actiondim
                                                                        \end{cases}
                                                                        \label{eq:expo}
\end{equation}
Then clearly $\cond({\action|\state})$ in (\ref{eq:expo})  always satisfies feasibility
(\ref{eq:simplex})
for any real-valued (un-constrained)  $\param({\state,\action})$.
To summarize,  there are  $(\actiondim-1) \statedim$ unconstrained parameters in $\param$.
Also for $\param(i,u) \in [0,\pi/2]$,
the mapping $\logistic: \reals^{\actiondim\times \statedim}\rightarrow \reals^{\actiondim\times \statedim}$ in (\ref{eq:expo}) is  one-to-one and therefore invertible. We denote the inverse as $\logistic^{-1}$.

{\em Remark}: As an example, consider $\actiondim=2$. Then  in spherical coordinates $\cond(1|i) = \sin^2\param(i,1)$, $\cond(2|i) = \cos^2\param(i,1)$, where $\param(i,1)$ is un-constrained.; clearly
$\cond(1|i) + \cond(2|i) = 1$, $\cond(u|i) \geq 0$.

With the above re-parametrization,
we can run any of the  passive Langevin dynamics IRL algorithms proposed above.
In the numerical example below,
we ran the two-time scale multi-kernel  IRL algorithm (\ref{eq:mcmcirl}). Recall this does not require knowledge of $\belief(\cdot)$ and also provides variances reduction:
Given the current IRL estimate $\eth_\dtime$, the RL gives us  a sequence
  $\{\cond_i, \nabla_\cond \reward_{\dtime}(\cond_i), i=1,\ldots,L\}$
 The IRL algorithm (\ref{eq:mcmcirl})  operating on the $(\actiondim-1) \statedim$ unconstrained parameters of $\param$ is:
\beq
\begin{split}
 \eth_{\dtime+1} &= \eth_\dtime + \stepa \,\frac{\temperature}{2} \, \frac{\sum_{i=1}^\numparticles   \pdf(\th_i|\eth_\dtime) \nabla_\th \reward_\dtime(\th_i) }{\sum_{l=1}^\numparticles \pdf(\th_l|\eth_\dtime)}  
  + \sqrt{\stepa}  \noise_\dtime,\qquad \cond_i \sim \belief(\cdot)  \\
    \text{ where } &\quad \th_i = \logistic^{-1}(\cond_i), \quad \nabla_\th \reward_\dtime(\th_i) =
( \nabla_\cond \reward_{\dtime}(\cond_i) )^\p \, \nabla_\th \cond_i, \\
&\pdf(\th|\eth) = \pdf_\obsnoise(\th - \eth), \qquad  \pdf_\obsnoise(\cdot) = \normal(0,\sigma^2I_\thdim) 
\end{split} \label{eq:irl_mdp}
     \eeq
 In the
 second line of (\ref{eq:irl_mdp}), we transformed  $\nabla_\cond \reward_{\dtime}(\cond_\dtime) $ to
$ \nabla_\th \reward_\dtime(\th_\dtime)$  to use in the  IRL algorithm.

To summarize,  the IRL algorithm (\ref{eq:irl_mdp}) generates samples  $\eth_\dtime \sim \exp(\Reward(\logistic(\eth)))$.
Equivalently, $\cond_\dtime = \logistic(\eth_\dtime) \sim \exp(\Reward(\cond))$, where $\Reward(\cond)$ is defined in (\ref{eq:penalty_reward}). Thus given only gradient information from a RL algorithm,  we can reconstruct (sample from) the penalized reward $\Reward(\cdot)$ of the  CMDP without any knowledge of the CMDP parameters.

\subsubsection{Numerical Example}
We generated a CMDP with $\statedim=2$ (2 states),  $\actiondim=2$ (2 actions) and 1 constraint with
\beq
\tp(1) = \begin{bmatrix} 0.8 & 0.2 \\ 0.3 & 0.7
\end{bmatrix}, \quad \tp(2) =  \begin{bmatrix} 0.6 & 0.4 \\ 0.1 & 0.9
\end{bmatrix},  \mreward = \begin{bmatrix} 1 & 100 \\ 30 & 2
\end{bmatrix},  \con = \begin{bmatrix} 0.2 & 0.3 \\ 2 & 1
\end{bmatrix}, \rcon = 1, \lambda = 10^5
\eeq
Recall the transition matrices $\tp(\action)$ are defined in (\ref{Aij}),  the reward matrix $(\mreward(\state,\action))$ in  (\ref{J}),  constraint matrix $(\con(\state,\action))$ and $\rcon$ in (\ref{costconstraint}),  and penalty multiplier  $\lambda$  in (\ref{eq:penalized}).

The  randomized policy $\cond(\action|\state)$, $\action\in \{1,2\}$, $\state \in \{1,2\}$  is a $2\times 2$ matrix. It is completely determined  by $(\cond(1|1), \cond(1|2)) \in [0,1]\times [0,1]$; so it suffices to  estimate $\Reward(\cond)$ over $[0,1]\times[0,1]$.

Figure \ref{fig:mdp}(a) displays the cumulative reward $J(\cond)$; this constitutes the ground truth.
To obtain this figure, we computed  the average reward MDP value function $J(\cond)$ and constraint $\Cons(\cond)$ for each policy $\cond$ where $\cond$ sweeps over $[0,1]\times[0,1]$.
Given a policy $\cond$, $J(\cond)$ and $\Cons(\cond)$ are  computed by first evaluating the joint probability $\statpi$ as \cite[pp.101]{Ros83}
$$ \statpi(j,a) = \sum_i \sum_{\bar{a}} \statpi(i,\bar{a})\, \tp_{ij}(\bar{a})\, \cond(a|j), \quad
\sum_j \sum_a \statpi(j,a)= 1 $$
and then
$ J(\cond)  = \sum_\state\sum_ \action \statpi(\state,\action) \mreward(\state,\action)$,
$\Cons(\cond) =  \sum_\state \sum_\action \statpi(\state,\action) \con(\state,\action) $.

For values of $\cond$ that do not satisfy the constraint $\Cons(\cond) < \rcon$, we plot $J(\cond) = 0$.
Figure \ref{fig:mdp}(a) illustrates the non-convex nature of the constraint set.

Figure \ref{fig:mdp}(b) displays the penalized cumulative reward $\Reward(\cond) = J(\cond) - \lambda\, (\Cons(\cond)-\rcon)^2$ where the quadratic penalty function is $\lambda\, (\Cons(\cond)-\rcon)^2$. As mentioned earlier, since we know that the constraint is active at the optimal policy, we want the IRL  to explore the vicinity of the region of $\cond$ where the constraint is active.

We then ran the  IRL algorithm (\ref{eq:irl_mdp})
 using spherical coordinates with parameters  $\stepa = 5 \times 10^{-6}$, $\sigma =0.1$, $L=50$ for
 $\horizon = 10^5$ iterations.
Figure \ref{fig:mdp}(c)   displays a 3-dimensional stem plots of the log of the  empirical distribution of $\cond_k = \logistic(\eth_k)$.
 wrt coordinates $\cond(1|1)$ and $\cond(1|2)$.  As can be seen from the two plots, the IRL algorithm samples from the high probability regions $\{\cond: \Cons(\cond) \approx \rcon\}$ to reconstruct the penalized reward $\Reward(\cond)$. Specifically, the $C$-shaped curve profile generated by the IRL estimates match the $C$-shaped curve of the penalized cumulative reward Figure \ref{fig:mdp}(b).

\begin{figure}\centering
 \subfigure[Cumulative Reward $J(\cond)$ with active constraint $\Cons(\cond)\leq 1$. The non-convexity of the constraint set is clearly seen.]{
\includegraphics[scale=0.45]{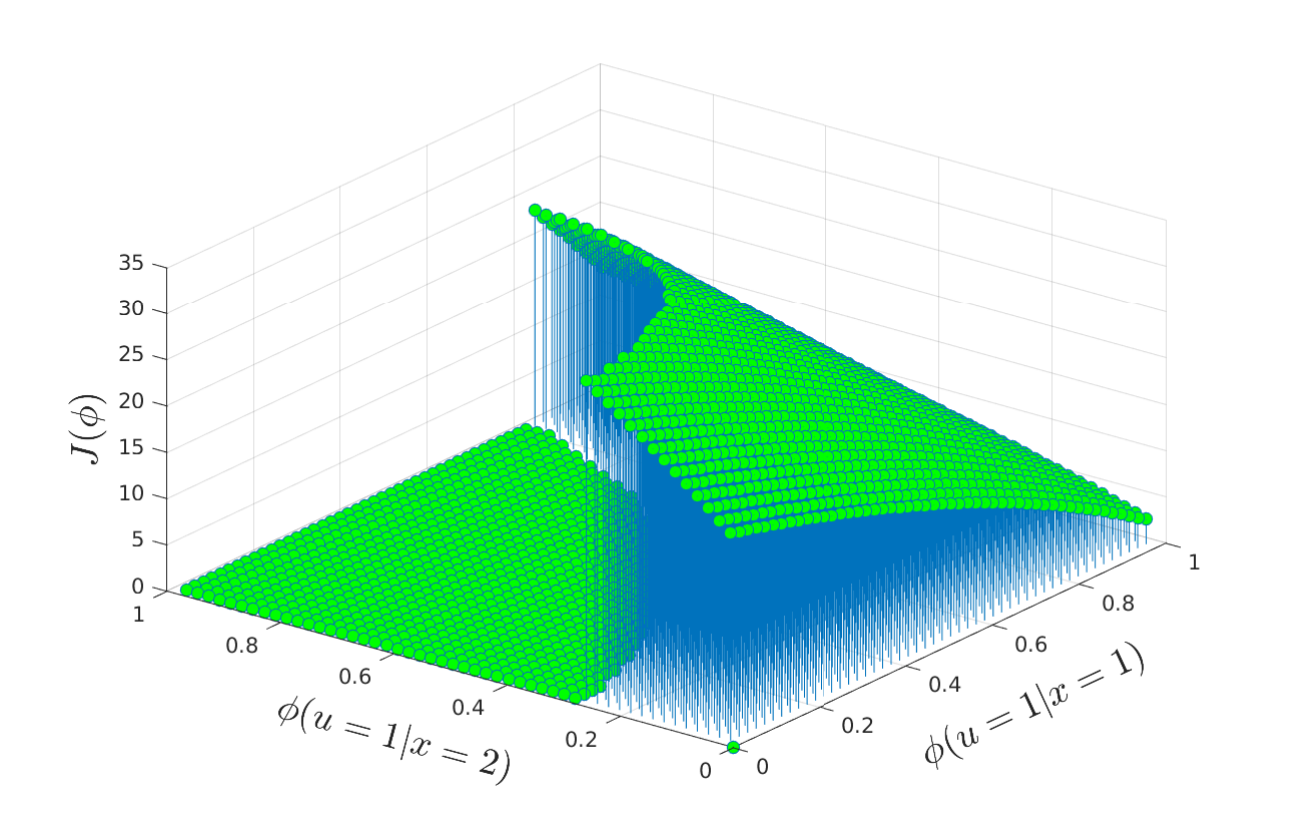}}
\subfigure[Penalized Cumulative Reward with Quadratic Penalty  $\Reward(\cond) = J(\cond) - \lambda\, (\Cons(\cond)-1)^2$. The lighter green shade  on top shows the active constraint. This plot constitutes the ground truth]
{\includegraphics[scale=0.45]{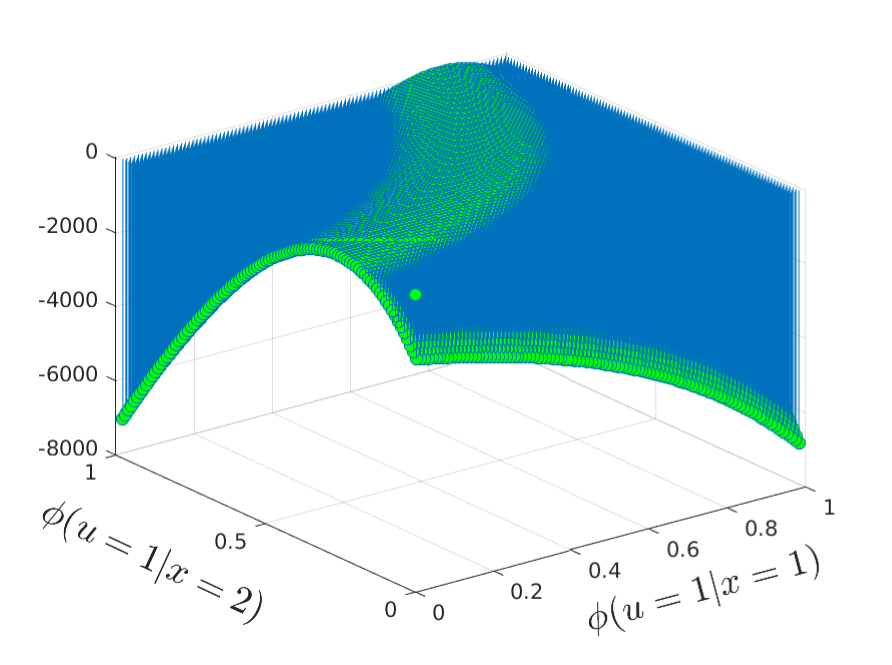}}
\end{figure}

\begin{figure}\centering
  \subfigure[]{\includegraphics[scale=0.4]{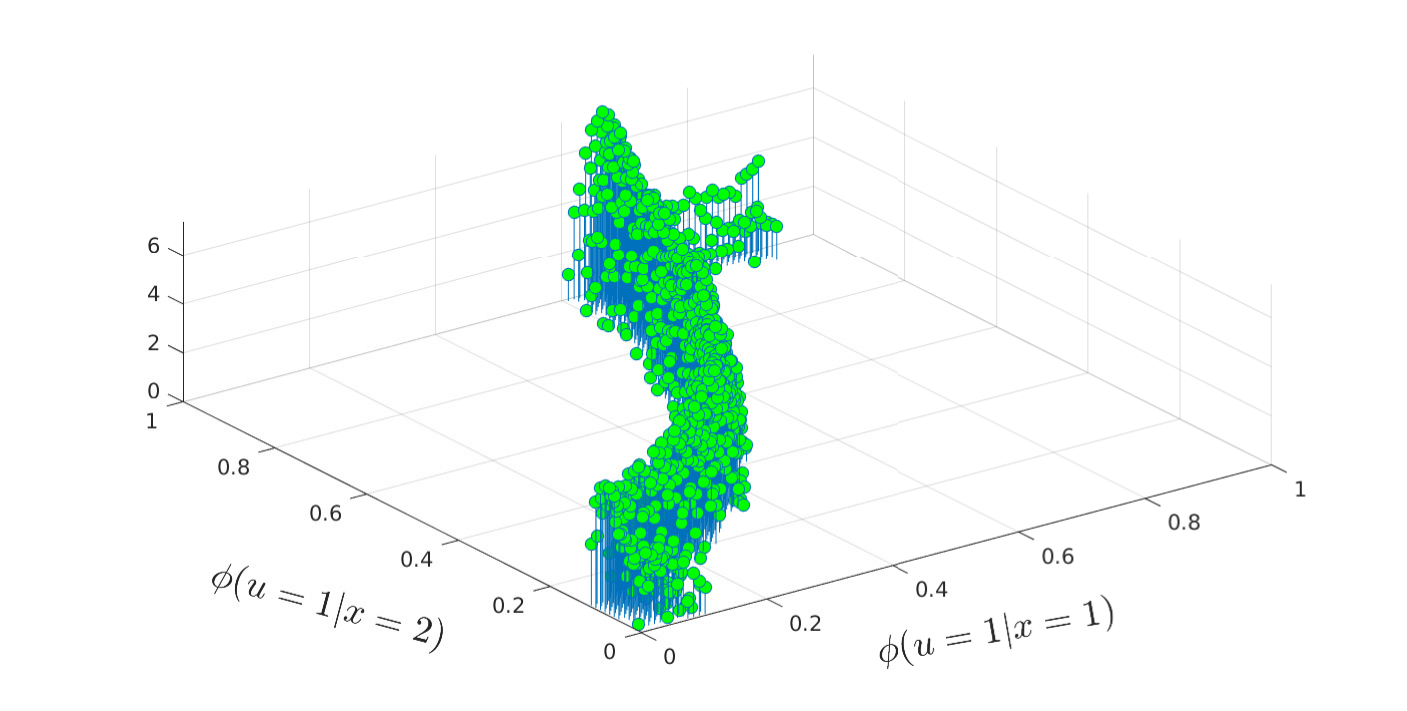}}
  \subfigure[]{\includegraphics[scale=0.4]{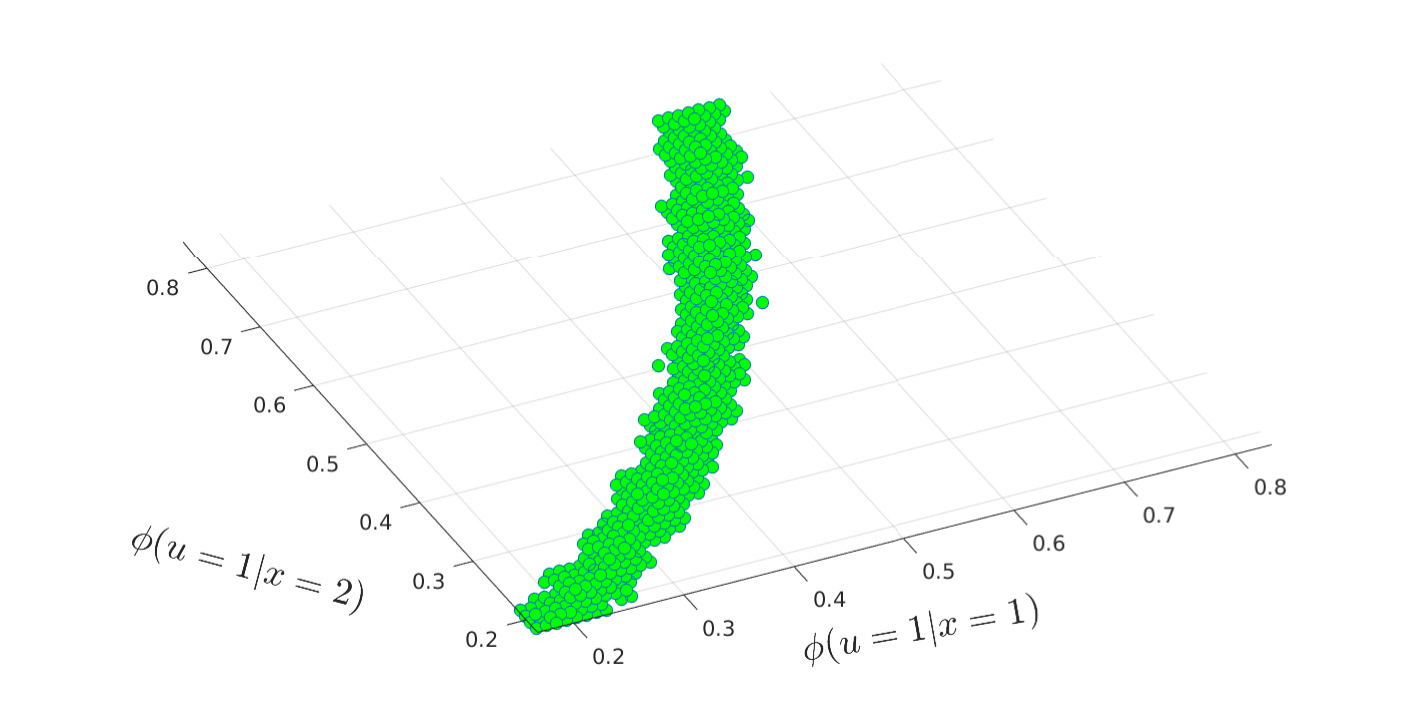}}
  \caption{IRL algorithm estimate. Snapshot 1 shows that the IRL estimates $\Reward(\cond)$   in the vicinity of the active constraint..
Snapshot 2  shows that the IRL explores regions  in the vicinity of the active constraint. Specifically the curve is close to the lighter shade green in Fig (b)}
\caption{IRL for Constrained MDP}
\label{fig:mdp}
\end{figure}

\section{Weak Convergence Analysis of IRL Algorithm}
\label{sec:weak}

This section discusses the main assumptions, weak convergence theorem and proof regarding IRL algorithm~(\ref{eq:irl}). (Recall the informal proof in Sec.\ref{sec:informal} for the motivation of weak convergence.)

\subsubsection*{Notation}
\begin{compactitem}
\item Since  $\nabla_\th \reward_k(\th_k)$ is a noise corrupted estimate of the gradient $\nabla_\th \Reward(\th)$, we write it in more explicit notation as
$\wdt \reward (\th_k,\xi_k)$, where $\{\xi_k\}$ is a sequence of random variables satisfying appropriate conditions specified below.
\item We  use $\belief_\al(\cdot)$ to denote $\nabla_\al \belief(\cdot)$. 
\item Finally,  $\E_m$  denotes the conditional expectation (conditioning up to time $m$), i.e.,
conditioning wrt the $\sigma$-algebra ${\mathcal F}_m=\sigma\{ \al_0,\th_j,\xi_j; \
j<m\} $.
\end{compactitem}

\subsubsection*{Algorithm}
There are two possible implementations of  IRL algorithm (\ref{eq:irl}).
The first implementation is (\ref{eq:irl}), namely,
\begin{equation}
  \label{eq:irl-1}\al_{k+1}=\al_k + {\stepa \over \kernelstep^\thdim} K\( {\th_k-\al_k \over \kernelstep} \) {\temperature \over 2} \wdt \reward(\th_k,\xi_k)\belief(\al_k)+ \stepa \belief_\al(\al_k) \belief(\al_k) +\sqrt \stepa \belief(\al_k) \noise_k,
\end{equation}
and the second implementation is
\begin{equation}
  \label{eq:irl-2}\al_{k+1}=\al_k + {\stepa \over \kernelstep^\thdim}  K\( {\th_k-\al_k \over \kernelstep} \)\Big[ {\temperature \over 2} \wdt \reward(\th_k,\xi_k)\belief(\th_k) + \belief_\al(\th_k)\Big]  +\sqrt {\stepa \over \kernelstep^\thdim}  K\( {\th_k-\al_k \over \kernelstep} \)\belief(\th_k) \noise_k,
\end{equation}
where $\stepa$ is the stepsize and $\kernelstep=\kernelstep(\stepa)$ is chosen so
$\stepa /\kernelstep^\thdim\to 0$ as $\stepa\to 0$.

Both the above  algorithms converge  to the same limit. The proof below is devoted to \eqref{eq:irl-1}, but
\eqref{eq:irl-2} can be handled similarly.
Also the proofs of the other two proposed IRL algorithms, namely   (\ref{eq:irl2}) and  (\ref{eq:irl3}) are similar.

Taking a continuous-time interpolation
\beq \al^\stepa(t)= \al_k
\ \hbox{ for } \ t \in [\stepa k, \stepa k+ \stepa),  \label{eq:interpolated}\eeq
we aim to show that the sequence $\al^\stepa(\cdot)$ converges weakly to $\al(\cdot)$, which give the desired limit.

\subsection{Assumptions}
We begin by stating the conditions needed.

\begin{enumerate}[label=(A{\arabic*})]

\item \label{H1} For each $\xi$, $\wdt \reward(\cdot,\xi)$ has continuous partial
derivatives up to the second order such that the second partial
$\wdt \reward_{\al\al}(\cdot, \xi)$ is bounded.
For each
$b<\infty$ and $T<\infty$,
$\{\wdt r(\al,\xi_j); |\al|\le b, j\stepa \le T\}$
is
uniformly integrable.

\item \label{H2}  The sequences $\{ \th_k\} $ is stationary and  independent of $\{\xi_k\}$.
For each
$k\ge n$, there exists a conditional density of $\th_k$
given ${\mathcal F}_n$, denoted by $\belief_k(\th| {\mathcal F}_n)$ such that
$\belief_k(\th|{\mathcal F}_n)>0$ for each $\th$ and that
$\belief_k(\cdot |{\mathcal F}_n)$ is continuous. The sequence
$\{\belief_k(\cdot|{\cal F}_n)\}_{k\ge n}$ is bounded uniformly.
The probability
density $\belief(\cdot)$ is continuous and bounded with
$\belief(\th)>0$ for each $\th$ such that
\begin{equation}
  \label{pi-ave}\lim_{k-n\to \infty}\E | \belief_k(\th|{\mathcal F}_n) - \belief(\th) | =
  0 .
\end{equation}

\item \label{H3}
The measurement noise $\{\xi_n\}$ is
exogenous, and bounded stationary mixing process with mixing measure $\ph_k$ such that
$\E \wdt r(\al,\xi_k)= R_\al(\al)$ for each $\al$ and $\sum_{k} \ph_k<\infty$. The $\{\noise_k\}$ is a sequence of $\reals^\thdim$-valued i.i.d. random variables with mean $0$ and covariance matrix $I$ (the identity matrix); $\{w_k\}$ and $\{\xi_k\}$ are independent.

\item \label{H4}
 The kernel $\kernel(\cdot)$ satisfies
 \begin{equation}
   \label{ker}
\barray
 \ad K(u)\ge 0, \ K(u)=K(-u), \sup_u K(u) < \infty,\\
 \ad \int K(u) du =1, \  \int |u|^2 K(u)du < \infty.
 \earray
\end{equation}

\end{enumerate}

{\em Remarks}:
We briefly  comment on the assumptions \ref{H1}-\ref{H4}.
\begin{compactitem}
\item Assumption \ref{H1} requires the smoothness of $\wdt r(\cdot, \xi)$, which is natural because we are using $\wdt r(\cdot,\xi_k)$ to approximate the smooth function $\nabla R$. We consider a general noise so the uniform integrability is used. If
    the noise is additive in that $\wdt r(\th,\xi)= \nabla R(\th) + \xi$, then we only need the finite $\wdt p$-moments of $\xi_k$ for $\wdt p>1$.
\item Assumption \ref{H3} requires the stochastic process $\{\xi_n\}$ to be
exogenous, and bounded stationary mixing. Thus
for each $\al$,  $\{\wdt r(\al,\xi_k)\}$ is also a mixing sequence. A mixing process is one in which remote past and distant future are asymptotically independent. It covers a wide range of random processes such as i.i.d. sequences, martingale difference sequences,  moving average sequences driving by a martingale difference sequence, and functions of stationary Markov processes with a finite state space \cite{Bil99}, etc. The case of $\{w_k\}$ and $\{\xi_k\}$
being dependent can be handled, but for us $\{w_k\}$ is the added  perturbation to get the desired Brownian motion so independence is sufficient.

\item By exogenous in \ref{H3}, we mean that
\begin{align*}
& P(\xi_{n+1}\in A_1,\ldots, \xi_{n+k}\in
A_k|\al_0,\xi_j, x_j; \
j\le n)
\\
& =P(\xi_{n+1}\in A_1,\ldots,\xi_{n+k}\in A_k|\al_0, x_j,
\xi_j,\al _{j+1}; \
  j\le n),
\end{align*}
for all Borel sets $A_i$, $i\le k$, and for all $k$ and $n$.

\item In view of the mixing condition \ref{H3} on $\{\xi_k\}$,
for each
$b<\infty$ and $T<\infty$,
$\{\wdt r(\al,\xi_j); |
\al|\le b, j\stepa \le T\}$ and $\{\wdt r_\al(\al,\xi_j); |\al|\le b,
j \stepa\le T\}$ are
uniformly integrable.

\item
Again, using the mixing condition,
for each $\al$, as $n\to \infty$,
\begin{equation}
\label{ave}{1\over {n}}\sum^{m+n-1}_{j=m} \E_m \wdt r(\al,\xi_j) \to R_\al(\al)
\hbox{ in probability.}
\end{equation}

\item For a Borel set $A$,  we have
$P(\th_k\in A| {\cal F}_n)=\int_{\th\in A} \belief_k(\th|{\cal F}_n) d\th$.
If $\{\th_n\}$ is itself a stationary $\phi$-mixing sequence with a
continuous density, and if $\E|\th_n|^2<\infty$,
then by virtue of a well-known mixing inequality, some $\wdt c_0>0$,
\cite[Corollary 2.4 in Chapter~7]{EK86},
$$\E \bigl\{ | \int  \th \belief_k(\th|{\cal F}_n)d\th -\int \th \belief( \th) d\th |\bigr\} \le
\wdt c_0
\ph^{1/2}_\th(k-n)\E^{1/2}|\th_k|^2
\to 0 \hbox{   as  }k-n\to \infty,$$
where $\ph_\th(\cdot)$ denotes the
mixing measure.

\item Condition \ref{H4} is
  concerned with the properties of $K(\cdot)$.
 It assumes  that the kernel is  nonnegative, symmetric, bounded (similar to a probability density function),
 and square integrable. \ref{H4} is satisfied by a large class of kernels.
 For example, commonly used symmetric kernels with compact supports
 satisfy this condition (e.g., truncated Gaussian kernels). Moreover, it is also verifiable for kernels with {\em unbounded} support.
 A crucial point is  that the tails of $K(\cdot)$ are small (asymptotically negligible).
  For simplicity, we use \ref{H4} as a nicely packaged version.
   In fact,  \ref{H4} is a sufficient condition
for a much larger class of
kernels satisfying
\begin{subequations} \label{ker-1}
  \begin{align}
 &\int K(u) du=1, \ \int |u|^l K(u) du < \infty \ \hbox{ some } l,\\
&\int K^2 (u)du <\infty, \ \int |u|^2 K(u) du < \infty, \\
&\int (u^1)^{m_1} (u^2)^{m_2} \cdots (u^N)^{m_N} K(u) du =0\ \hbox{ if } \ l > 1, \label{eq:subc}\\
& \text{ where } 1 \le m_1+ m_2 + \cdots + m_N \le l-1 \label{eq:subd}
  \end{align}
\end{subequations}
Here $u^1,\dots,u^N$ denote the components of $u\in {\mathbb R}^N$. 
The parameter $l$ is a smoothness indicator of the kernel and the last line of \eqref{ker-1} is often used in  nonparametric estimation in statistics.
Such a condition stems from a large class of kernels used
in the so-called
 $l$th-order averaging operator; see \cite{Kat76}. Thus, \ref{H4} can be replaced by this more general setup. However,
 we use the current form of \ref{H4} because it is easily verifiable (e.g., by Gaussian kernel).
 
 Eq.\eqref{ker-1} can be written in   multi-index notation as follows. Let
 $m= (m_1,\dots, m_N)$ where each $m_i$ is a  nonnegative integer,
 $|m|= \sum^N_{i=1} m_i$, and $m!= \prod^N_{i=1} m_i!$. Thus, $u^m= (u^1)^{m_1} \cdots (u^N)^{m_N}$. Then  (\ref{eq:subc}), (\ref{eq:subd})   
 can be written in multi-index notation as
 $$\int u^m K(u)du=0  \ \hbox{ if } \ 1\le |m| \le l-1  \ \hbox{ and } \ l>1.$$

\end{compactitem}

\subsection{Main Result} \label{sec:proofmain}

As is well known
 \cite{KY03}, a classical fixed step size  stochastic gradient  algorithm
converges weakly  to a \textit{deterministic}  ordinary differential equation (ODE) limit; this is the basis of the so called ODE approach  for analyzing stochastic gradient algorithms.  In comparison, the discrete time IRL algorithm (\ref{eq:irl})  converges weakly to a \textit{stochastic} process limit $\al(\cdot)$.  We now formally state the weak convergence result  of the interpolated process $\{ \al^\stepa(\cdot)\}$ to the stochastic process limit $\al(\cdot)$ as $\stepa\rightarrow 0$.
Proving weak convergence
requires first that  the tightness of the sequence be verified and
then the limit be characterized via the so called
martingale problem formulation.  For a comprehensive treatment of the martingale problem of Stroock and Varadhan,  see \cite{EK86}.

\begin{theorem}\label{thm:weak-conv}
  Assume conditions {\rm \ref{H1}-\ref{H4}}. Then the interpolated process $\al^\stepa(\cdot)$ $($defined in \eqref{eq:interpolated}$)$ for IRL algorithm \eqref{eq:irl}  has the following properties:
  \begin{compactenum} \item
$\{ \al^\stepa(\cdot)\} $
is tight in $D^d[0,\infty)$.
\item Any weakly convergent subsequence
of $\{\al^\stepa(\cdot)\}$  has a
limit $\al(\cdot)$ that
satisfies
\beq \begin{split} d\al(t) &= \Big[ {\temperature \over 2} \belief^2(\al(t)) \Reward_\al(\al(t))+ \belief_\al(\al(t)) \belief(\al(t))\Big] dt+ \belief(\al(t) ) d\bm(t), \\
 \al(0) &=\al_0,
\end{split}
\label{eq:sde} \eeq
where $\bm(\cdot)$ is a standard Brownian motion with mean 0 and covariance being the identity matrix $I \in \reals^{\thdim\times \thdim}$,
provided~\eqref{eq:sde} has a
unique weak solution $($in a distributional sense$)$ for each initial condition.
\end{compactenum}
\end{theorem}

The proof is in \cite{KY21}.
For sufficient conditions leading to  unique weak solutions of  stochastic differential equation and uniqueness of martingale problem, see \cite[p. 182]{EK86} or
\cite{KS91}.

\section{Tracking Analysis of IRL in Non-Stationary Environment}
\label{sec:markov}

An important feature of the IRL algorithm (\ref{eq:irl}) is its constant step size $\stepa$  (as opposed to a decreasing step size).  This facilities  estimating (tracking) time evolving reward functions. This section analyzes the ability of IRL algorithm  to track  a time-varying reward  function.

Since we are  estimating  a time evolving  reward, we first  give a model for the evolution of the reward
$\Reward(\th)$ over time.
Below, the Markov chain $\{\mc_\dtime\}$ will be used as a \textit{ hyper-parameter} to model the evolution
of the time varying  reward, which we will denote as $\Reward(\th,\mc_\dtime)$.  By hyper-parameter we mean that the Markov chain model is not known or used  by the IRL  algorithm (\ref{eq:irl}). The Markov chain assumption is used  only in  our convergence analysis
to determine how well does the IRL  algorithm estimates (tracks) the  reward $\Reward(\th,\mc_\dtime)$  that jump changes
(evolves) according to an unknown Markov chain $\mc_\dtime$.

We assume that the RL agents perform gradient algorithm (\ref{eq:rl}) by
evaluating the sequence of gradients  $\{\nabla_\th \reward_\dtime(\th_\dtime, \mc_\dtime)\}$. Note that both   the RL and IRL do not know the sample path $\{\mc_\dtime\}$.
We will use similar   notation to  Sec.\ref{sec:weak}:
\begin{compactitem} \item
  Denote $ \nabla_\th\reward_\dtime(\th_\dtime, \mc_\dtime) $ as $\wdt \reward (\th_k,\xi_k,\mc_\dtime)$,
  \item We use $\belief_\al(\cdot)$ to denote $\nabla_\al \belief(\cdot)$.
\end{compactitem}

\subsection{Assumptions}
We focus on the following algorithm
\begin{equation}
  \label{eq:irl-1a}
\al_{k+1}=\al_k + {\stepa \over \kernelstep^\thdim} K\( {\th_k-\al_k \over \kernelstep} \) {\temperature \over 2} \wdt \reward(\th_k,\xi_k,x_k)\belief(\al_k)+ \stepa \belief_\al(\al_k) \belief(\al_k) +\sqrt \stepa \belief(\al_k) \noise_k,
\end{equation}
The  main assumptions are as follows.
\begin{enumerate}[label=(M{\arabic*})]
  \item  \label{M1} (Markovian hyper-parameter)   Let $\{\mc_\dtime, \dtime \geq 0\}$ be a Markov chain with finite state space  $\statespace=\{1,\dots,X\}$ and
  transition probability matrix $I +\mcstep Q$, where $\mcstep>0$ is a small parameter and $Q=(q_{ij})$ is an
  $\statedim \times \statedim$ irreducible generator (matrix) \cite[p.23]{YZ13} with $$q_{ij}\ge 0, \quad i\not =j, \qquad
  \sum_j q_{ij}=0 , \quad   i \in \statespace ,$$ also $\{x_k\}$ is independent of $\{\th_k\}$ and $\{w_k\}$.
\item \label{M2} Assumption \ref{H1} holds on   $ \wdt \reward(\cdot,\xi, i)$ for each fixed  state $i \in \statespace$.
    Also  \ref{H2}, \ref{H3}, \ref{H4} hold.
\end{enumerate}

\subsection{Main Result}
Recall that $\stepa$ is the step size of the IRL algorithm while $\mcstep$ reflects the rate at which the hyper-parameter Markov chain $\mc_\dtime$ evolves.   In the following tracking analysis of IRL algorithm (\ref{eq:irl}), we will consider three cases,
 $\stepa = O(\mcstep)$, $\stepa \ll \mcstep$, and $\stepa \gg \mcstep$.
 The three cases represent three different types of asymptotic behavior.
 If $\stepa \gg \mcstep$, the frequency of changes of the Markov chain is very slow. Thus, we are treating a case similar to a
 constant parameter, or
 we essentially deal with a ``single'' objective function.
 If $\stepa \ll \mcstep$, then the Markov chain jump changes frequently.
 So what we are optimizing is a function
 $\sum^X_{i=1} R(\al,i) \nu_i$,
 where $\nu_i$ is the stationary distribution  associated with the generator $Q$.
 If $\stepa =O(\mcstep)$, then the Markov chain changes in line with the optimization recursion. In this case, we obtain switching limit
 Langevin diffusion.

 In  Theorem \ref{thm:gld1} below, for brevity we use  $\stepa =\stepmc$ for $\stepa = O(\stepmc)$,
 $\mcstep = \stepa^{1+\wdt\Delta}$ for $\mcstep = o(\stepa)$ and   $\mcstep= \stepa ^{\wdt\Delta}$
 for $\stepa = o(\mcstep)$, respectively. These cover all three possible cases of the rate at which the hyper-parameter evolves compared to the dynamics of the Langevin IRL algorithm.

 \begin{theorem} \label{thm:gld1}
   Consider the
   IRL algorithm \eqref{eq:irl-1a}.
Under Assumptions {\rm\ref{M1}} and {\rm\ref{M2}}, assuming that
\eqref{eq:sde_track}, or \eqref{eq:sde_track2}, or \eqref{eq:sde_track3} has a unique solution in the sense in distribution. Then
   the following results hold.
    \begin{itemize}
    \item[{\rm 1.}]
Assume $\stepa =\stepmc$.
Then as $\stepa \downarrow 0$,
the interpolated process $(\eth^\stepa\cd,\mc^\stepa\cd)$  converges weakly to the switching diffusion $(\eth\cd, \mc\cd)$ satisfying
\beq  d\al(t) = \Big[ {\temperature \over 2} \belief^2(\al(t)) \Reward_\al(\al(t),\mc(t))+ \belief_\al(\al(t)) \belief(\al(t))\Big] dt+ \belief(\al(t) ) d\bm(t), \\
 \label{eq:sde_track} \eeq
where $\bm(\cdot)$ is a standard Brownian motion with mean 0 and covariance being the identity matrix $I \in \reals^{\thdim\times \thdim}$,
 and  $x\cd$ is a continuous-time Markov chain with generator $Q$.
    \item[{\rm 2.}]
    Suppose $\mcstep = \stepa^{1+\wdt\Delta}$ with $\wdt\Delta > 0$ and denote the initial distribution of $x^\mcstep(0)$ by $p_\iota$ $($independent of $\mcstep)$ for each $\iota \in {\cal X}$.
      Then  as $\stepa \downarrow 0$,
    the interpolated process $(\eth^\stepa\cd)$  converges weakly to the following
    diffusion process
    \beq d\al(t) = \Big[ {\temperature \over 2} \belief^2(\al(t)) \sum_{\iota \in \statespace} \Reward_\al(\al(t),\iota )\, p_\iota+ \belief_\al(\al(t)) \belief(\al(t))\Big] dt+ \belief(\al(t) ) d\bm(t), \\
 \label{eq:sde_track2} \eeq
\item[{\rm 3.}]  Suppose that $\mcstep= \stepa ^{\wdt\Delta}$ with $0<\wdt\Delta<1$ and denote the stationary distribution associated with the continuous-time Markov chain with generator $Q$ by $\nu =(\nu_1,\dots, \nu_{X})$. Then  as $\stepa \downarrow 0$,
    the interpolated process $(\eth^\mu\cd)$  converges weakly to the following
    diffusion process
     \beq  d\al(t) = \Big[ {\temperature \over 2} \belief^2(\al(t)) \sum_{\iota \in \statespace} \Reward_\al(\al(t),\iota )\, \nu_\iota+ \belief_\al(\al(t)) \belief(\al(t))\Big] dt+ \belief(\al(t) ) d\bm(t).
     \label{eq:sde_track3} \eeq
   \end{itemize}
 \end{theorem}

{\em Remark}. Theorem \ref{thm:gld1}  is proved in \cite{KY21} and uses ideas from \cite{YKI04,YIK09,NKY13}. The 
theorem characterizes the asymptotic behavior of the IRL  algorithm (\ref{eq:irl-1a})  with Markovian switching.  In accordance with the rates of variations of the adaptation rates (represented by the stepsize $\stepa$)
and the switching rate (represented by the stepsize $\mcstep$),
three cases are considered.  Case~1 indicates that when $\stepa$ is in line with $\stepmc$, the limit differential equation is a switching diffusion. Case~2 concentrates on the case that the switching is much slower than the stochastic approximation generated by the recursion. Thus, the limit Langevin equation is one in which the drift and diffusion coefficients are averaged out with respect to the initial distribution of the limit Markov chain. Roughly, it reveals that the ``jump change'' parameter $\mc(t)$ is more or less as a constant in the sense the coefficients are averages w.r.t. the initial distribution.
Case~3 is the one that the Markov chain is changing much faster than the stochastic approximation rate. As a result, the ``jump change'' behavior is replaced by an average with respect to the stationary distribution of the Markov chain.  Then we derive the associated limit Langevin equation. Again, the limit has no switching in it.

\section{Discussion}

This chapter  has presented and analyzed the convergence of passive Langevin dynamics algorithms for adaptive inverse reinforcement learning (IRL). Given noisy gradient estimates of a possibly time evolving reward function $\Reward$, the Langevin dynamics algorithm generates samples $\{\eth_k\}  $ from the Gibbs measure $\stat(\eth)   \propto  \exp \bigl(  \temperature \Reward(\eth ) \bigr)$; so the log of the empirical distribution of $\{\eth_k\}$ serves as a non-parametric estimator for  $\Reward(\eth)$.
The proposed algorithm is  a {\em passive} learning algorithm since the gradients are not evaluated at $\eth_\dtime$ by the inverse learner; instead the gradients are evaluated at the random points  $\th_\dtime$ chosen by the gradient (RL) algorithm. This passive framework  is natural in an IRL where the inverse learner
passively observes  forward learners.

Apart from the main IRL algorithm (\ref{eq:irl}), we presented
a two-time scale
IRL algorithm for variance reduction, an active IRL algorithm which deals with mis-specified gradients, and a non-reversible diffusion IRL algorithm with larger spectral gap and therefore faster convergence to the stationary distribution.
We presented three detailed numerical  examples: inverse Bayesian learning, a large dimensional IRL problem in logistic learning involving a real dataset, and IRL for a constrained Markov decision process.
Finally, we discussed the main weak convergence result for  the IRL algorithm (the datild proof uses martingale averaging methods) and also discussed  the tracking capabilities of the IRL algorithm when the utility function jump changes according to a slow (but unknown) Markov chain.

{\bf Extensions.} A detailed proof of the two-time scale variance reduction algorithm involves Bayesian asymptotics, namely, the Bernstein von Mises theorem; see  \cite{KY22}. It is important to note that the IRL algorithms proposed in this chapter are adaptive: given the estimates from an adaptive gradient algorithm, the IRL algorithm learns the utility function. In other words, we have a  gradient algorithm operating in series with a Langevin dynamics algorithm. In future work it is of interest to study the convergence properties of multiple such cascaded Langevin dynamics and gradient algorithms.

 \cite{KHH20}  shows that classical Langevin dynamics   yields more robust RL algorithms compared to classic stochastic gradient. In analogy to \cite{KHH20}, in future work it is worthwhile exploring how our passive Langevin dynamics framework can be viewed as a robust version of classical passive stochastic gradient algorithms.

This chapter discussed the weak convergence and tracking properties of passive Langevin dynamic algorithms. In future work it is of interest to analyze the asymptotic convergence rate and
  spectral gap of the diffusion process induced by the proposed algorithm. This will also facilitate  quantifying  how the convergence rate is affected when the step size $\step_n$ of
  each RL agent $n$ is chosen randomly (and unknown to the inverse learner).

Finally, it is worth extending the algorithms in this chapter to the case where the inverse learner observes the gradients from multiple utility functions, but the inverse learner does not know which gradient came from which utility. By using symmetric polynomial transformations (that form an algebraic ring), one can extend the methodology in \cite{Kri23} to estimate the individual utilities (subject to permutation).

{\bf Finite Sample Results}. \cite{SK25} presents a  finite sample analysis of the adaptive IRL algorithm presented in this chapter.



\printbibliography

\clearpage

\printindex
\clearpage
\thispagestyle{empty} 

\end{document}